\def\year{2020}\relax
\documentclass[letterpaper]{article} % DO NOT CHANGE THIS
\usepackage{aaai20}  % DO NOT CHANGE THIS
\usepackage{times}  % DO NOT CHANGE THIS
\usepackage{helvet} % DO NOT CHANGE THIS
\usepackage{courier}  % DO NOT CHANGE THIS
\usepackage[hyphens]{url}  % DO NOT CHANGE THIS
\usepackage{graphicx} % DO NOT CHANGE THIS
\usepackage{graphicx}  % DO NOT CHANGE THIS
\usepackage{microtype}
\usepackage{graphicx}
\graphicspath{{img/}{../img/}}
\usepackage{subfigure}
\usepackage{booktabs} % for professional tables
\usepackage{epstopdf}
\usepackage{subfiles}
\usepackage[english]{babel}
\usepackage{blindtext}
\usepackage{tabularx}
\usepackage{enumerate}
\usepackage{times}  %Required
\usepackage{courier}  %Required
\usepackage{graphicx}  %Required
\usepackage{color}

\usepackage{amsthm}
\usepackage{amsmath}
\usepackage{amssymb}
\usepackage{algorithm}
\usepackage{algorithmic}
\input{required.tex}

\begin{document}
	\title{Aggregated Gradient Langevin Dynamics}

	\author{Chao Zhang\textsuperscript{1,2}, Jiahao Xie\textsuperscript{1}, Zebang Shen\textsuperscript{1,3}, Peilin Zhao\textsuperscript{3},Tengfei Zhou\textsuperscript{1}, Hui Qian\textsuperscript{1}\thanks{Corresponding Author}\\
		\textsuperscript{1}College of Computer Science and Technology, Zhejiang University\\
		\textsuperscript{2}AI Lab, Tencent\\
		\textsuperscript{3}University of Pennsylvania\\
		zczju@zju.edu.cn,xiejh@zju.edu.cn, shenzebang@zju.edu.cn,zhoutengfei@zju.edu.cn, qianhui@zju.edu.cn, masonzhao@tencent.com
	}
	%\author{PaperID:7214}
	%%\authorrunning{Short form of author list} % if too long for running head
	%
	%\institute{Chao Zhang \at
	%			   School of Computer Science and Technology, Zhejiang University, China 310027\\
	%			   \email{zczju@zju.edu.cn}
	%%             \emph{Present address:} of F. Author  %  if needed
	%          Jiahao Xie \at
	%           School of Computer Science and Technology, Zhejiang University, China 310027\\
	%          \email{xiejh@zju.edu.cn}
	%          Zebang Shen \at
	%           School of Computer Science and Technology, Zhejiang University, China 310027\\
	%          \email{shenzebang@zju.edu.cn}           %  \\
	%          %             \emph{Present address:} of F. Author  %  if needed
	%          \and
	%          Nenggan Zheng \at
	%           School of Computer Science and Technology, Zhejiang University, China 310027\\
	%          \email{zng@zju.edu.cn}           %  \\
	%          %             \emph{Present address:} of F. Author  %  if needed
	%          \and
	%           $^*$Hui Qian, the corresponding author \at
	%          School of Computer Science and Technology, Zhejiang University, China 310027\\
	%          Tel.: +86-0571-87951916\\
	%           \email{qianhui@zju.edu.cn}           %  \\
	%          %             \emph{Present address:} of F. Author  %  if needed
	%}


	\maketitle

	\begin{abstract}
In this paper, we explore a general Aggregated Gradient Langevin Dynamics framework (AGLD) for the Markov Chain Monte Carlo (MCMC) sampling.
We investigate the nonasymptotic convergence of AGLD with a \emph{unified} analysis for  different data accessing (e.g. \emph{random access}, \emph{cyclic access} and \emph{random reshuffle}) and snapshot updating strategies, under convex and nonconvex settings respectively. 
It is the first time that bounds for I/O friendly strategies such as \emph{cyclic access} and \emph{random reshuffle}  have been established in the MCMC literature.
%Similar to Stochastic Gradient Langevin Dynamics (SGLD), AGLD only need to compute part of the component gradient in each iteration, while the mixture time bound is similar to that of Langevin dynamics Monte Carlo method (LMC), which uses full gradient.
The theoretic results also indicate that methods in AGLD possess the merits of both the low per-iteration computational complexity and the short mixture time.
Empirical studies demonstrate that  our framework allows to derive novel schemes to generate high-quality samples for large-scale Bayesian posterior learning tasks.
%	\keywords{MCMC \and Langevin Dynamics \and cyclic access}
	\end{abstract}
	%--------------------------------------------------------
	% Introduction
	%--------------------------------------------------------
	\section{Introduction}
	We focus on the Langevin dynamics based Markov Chain Monte Carlo (MCMC) methods for sampling the parameter vector $\theta\in\RBB^d$ from a target posterior distribution
	\begin{equation}\label{eq.problem}
	p^* \triangleq p(\theta|\{z_i\}_{i=1}^{N}) \propto p(\theta)\prod_{i=1}^{N} p(z_i|\theta),
	\end{equation}
	where $p(\theta)$ is some prior of $\theta$, $z_i$'s are the data points observed, and $p(z_i|\theta)$ is the likelihood function.
	The Langevin dynamics Monte Carlo method (LMC) adopts the \emph{gradient} of log-posterior in an iterative manner to drive the distribution of samples to the target distribution efficiently
	\citep{roberts2002langevin,roberts1996exponential,mattingly2002ergodicity}.
	To reduce the computational complexity for large-scale posterior learning tasks, the Stochastic Gradient Langevin Dynamics method (SGLD), which replaces the expensive full gradient with the stochastic gradient, has been proposed {\citep{welling2011bayesian}}.
	While such scheme enjoys a significantly reduced per-iteration cost, the mixture time, i.e., the total number of iterations required to achieve the correction from an out-of-equilibrium configuration to the target posterior distribution, is increased, due to the extra variance introduced by the approximation \citep{dalalyan2017user,dalalyan2017theoretical}.

	In recent years, efforts are made to design variance-control strategies to circumvent this slow convergence issue in the SGLD
	In particular, borrowing ideas from variance reduction methods in the optimization literature {\citep{johnson2013accelerating,defazio2014saga,lei2017less}}, the variance-reduced SGLD variants exploit the high correlations between consecutive iterates to construct \emph{unbiased aggregated gradient} approximations with less variance, which leads to better mixture time guarantees \citep{dubey2016variance,zou2018subsampled}.
	%Alternatively, a parallel line of research borrows ideas from variance reduction methods in the optimization literature {\citep{johnson2013accelerating,defazio2014saga,lei2017less}}, and use the historical information to construct the \emph{aggregated gradient} approximations with less variance {}, which achieve better performance both practically and theoretically.
	Among these methods, SAGA-LD and SVRG-LD {\citep{dubey2016variance}} are proved to be the most effective ones when high-quality samples are required {\citep{chatterji2018theory,zou2019sampling}}.
	%\crr{ \{Add most effective evidence here\} ....}
	While the nonasymptotic convergence guarantees for SVRG-LD and SAGA-LD have been established, 
	it is difficult to seamlessly extend these analyses to cover other Langevin dynamics based MCMC methods with different efficient gradient approximations.
	%they can still be further improved.
	\begin{itemize}
%		\item First of all, assumptions used in these works are too restrictive from a practical point of view.
%		For example,  \citep{dubey2016variance}  obtain mixing time guarantees for SVRG-LD and SAGA-LD
%		under the assumption that {the functional solution of a particular Poisson equation is bounded up to 3rd-order derivatives by some function, which is hard to verify}; \citep{chatterji2018theory} construct delicate Lypunov functions to improve the guarantees for SAGA-LD and SVRG-LD assuming the log-posterior has Lipschitz-continuous Hessian, which is strictly stronger than the conventional smoothness assumption.
		\item
		%The different targets of optimization and MCMC make the transition of variance reduced estimators used in the former literature to the latter
		First of all, different delicate Lyapunov functions are designed for SVRG-LD and SAGA-LD to prove the nonasymptotic convergence to the stationary distribution.
		Due to the different targets of optimization and MCMC, the mixture-time analysis is not a simple transition of the convergence rate analysis in optimization.
		The lack of a unified perspective of these variance-reduced SGLD algorithms makes it difficult to effectively explore other variance-reduced estimators used in optimization (e.g., HSAG \cite{reddi2015variance}) for Langevin dynamics based MCMC sampling.
		In particular, customized Lyapunov functions need to be designed if new variance-reduced estimators are adopted.
		\item Second, existing theoretical analysis relies heavily on the randomness of the data accessing strategy to construct an unbiased estimator of the true gradient. In practice, the random access strategy entails heavy I/O cost when the dataset is too large to fit into memory, thereby renders existing incremental Langevin dynamics based MCMC algorithms heavily impractical for sampling tasks in the big data scenario.
		% While existing theoretical analysis can not be extended to algorithms with other incremental strategies, such as \emph{cyclic access} and \emph{random reshuffle}, these data accessing patterns are known to be I/O friendly~\citep{xie2018towards}.
		While other data accessing strategies such as \emph{cyclic access} and \emph{random reshuffle} are known to be I/O friendly~\citep{xie2018towards}, existing analysis can not be directly extended to algorithms with these strategies.
	\end{itemize}

	\noindent{\bf Contributions} 
	Motivated by such imperatives, we propose a general framework named Aggregated Gradient Langevin Dynamics (AGLD),
	which %abstracts the procedure of SVRG-LD and SAGA-LD 
	maintains a historical snapshot set of the gradient to construct more accurate gradient approximations than that used in SGLD.
	AGLD possesses a three-step structure: \emph{Data-Accessing}, \emph{Sample-Searching}, and \emph{Snapshot-Updating}.
	Different  Data-Accessing (e.g. \emph{random access}, \emph{cyclic access} and \emph{random reshuffle}) and Snapshot-Updating strategies can be utilized in this framework.
	By appropriately implementing these two steps, we can obtain several practical gradient approximations, including those used in existing methods like SVRG-LD and SAGA-LD.
	%and recover the estimator used in existing methods like SVRG-LD and SAGA-LD as special cases in our framework, and .
	%By appropriately implementing these two steps, we recover existing methods like SVRG-LD and SAGA-LD as special cases in our framework.
	%We also note that the random reshuffle and cyclic access strategies can be used as inputs of the Data-Accessing step.
	% Typical inputs
	% By appropriately implementing these two steps, we recover existing methods like SVRG-LD and SAGA-LD and derive new algorithms such as SVRG-LD and SAGA-LD variants with random reshuffle or cyclic access.
	Under mild assumptions, a unified mixture-time analysis of AGLD is established, which holds as long as each component of the snapshot set is updated at least once in a fixed duration. We list our main contributions as follows.
	\begin{itemize}
		\item
		We first analyze the mixture time of AGLD under the assumptions that the negative log-posterior $f(x)$ is smooth and strongly convex and then extend the analysis to the general convex case.
		% Similar guarantee can be extended to the general convex case with a different convergence rate.
		We also provide theoretical analysis for nonconvex $f(x)$.
		These results indicate that AGLD has similar mixture time bounds as LMC under similar assumptions, while the per-iteration computation is much less than that of LMC.
		Moreover, the analysis provides a unified bound for a wide class of algorithms with no need to further design dedicated Lyapunov functions for different Data-Accessing and {Snapshot-Updating} strategies.
		\item
		For the first time, mixture time guarantee for {cyclic access} and {random reshuffle} Data-Accessing strategies is provided in the Langevin dynamics based MCMC literature.
		This fills the gap of practical use and theoretical analyses,
		since cyclic access is I/O friendly and often used as a practical substitute for random access when the dataset is too large to fit into memory.
		%Thus, it is theoretically guaranteed to customize the {\bf Data-Accessing} step of SVRG-LD and SAGA-LD with strategy other than \emph{random access}.
		\item
		We develop a novel Snapshot-Updating strategy, named Time-based Mixture Updating (TMU), which
		enjoys the advantages of both the Snapshot-Updating strategies used in SVRG-LD and SAGA-LD: %blends a periodically full update with per-iteration partial update.
		it always updates components in the snapshot set to newly computed ones as in SAGA-LD and also periodically updates the whole snapshot set to rule out the out-of-date ones as in SVRG-LD.
		Plugging TMU into AGLD, we derive novel algorithms to generate high-quality samples for Bayesian learning tasks.
	\end{itemize}
	Simulated and real-world experiments are conducted to validate our analysis. Numerical results on simulation and Bayesian posterior learning tasks demonstrate the advantages of proposed variants over the state-of-the-art.\\

	\noindent{\bf Notation.}
	We use $[N]$ to denote $\{1,\ldots,N\}$, use ${\bf 1}_d$ to denote the $d$-dimensional vector with all entries being $1$, and use $\IB_{d \times d}$ to denote the $d$-dimensional identity matrix.
	For $a,b \in \RBB^+$, we use $a = \OM(b)$ to denote $a \le C b$ for some $C > 0$, and use $a = \tilde \OM(b)$ to hide some logarithmic terms of $b$.
	For the brevity of notation, we denote $f(\theta) =\sum_{i=1}^N f_i(\theta)$, where each $f_i(\theta) = -\log p(\theta|z_i)-\log p(\theta)/N$, for $i \in [N]$.
%	The target distribution is denoted as $p^* \propto \exp(f(x))$.

	%To distinguish the continuous time and discrete time cases, we denote $\theta(t)$ as the continuous time dynamics and $x^{(k)}$ as its time discretization.

	\section{Preliminaries}
	\subsection{Wasserstein Distance and Mixture Time}
	We use the 2-Wasserstein  ($\WM_2$) distance to evaluate the effectiveness of our methods.
	Specifically, the $\WM_2$ distance between two probability measures $\rho$ and $\nu$ is defined as
	\begin{equation*}
	\WM_2^2(\rho,\nu) = \inf_{\pi \in \Gamma (\rho,\nu)} \{\int  \|x-y\|^2_2 \dB \pi(x,y) \}.
	\end{equation*}
	Here, $(x,y)$ are random variables with distribution density $\pi$ and $\Gamma(\rho, \nu)$ denotes the collection of joint distributions where the first part of the coordinates has $\rho$ as the marginal distribution and the second part has marginal $\nu$.

	$\WM_2$ distance is widely used in the dynamics based {MCMC} literature since it is a more suitable measurement of the closeness between two distributions than metrics like the total variation and the Kullback-Leibler divergence {\citep{zou2018stochastic,dalalyan2017further}}.
	In this paper, we say $K$ is the \emph{$\epsilon$-mixture time} of a Monte Carlo sampling procedure if for every $k\geq K$,
	the distribution $p^{(k)}$ of the sample generated in the $k$-th iteration is $\epsilon$-close to the target distribution $p^*$ in 2-Wasserstein distance, i.e. $\WM_2(p^{(k)}, p^*)\leq\epsilon$.
	\subsection{Stochastic Langevin Dynamics}
	By using the discretization of certain dynamics, dynamics based MCMC methods allow us to efficiently sample from the target distribution.
	A large portion of such works are based on the Langevin Dynamics  {\citep{parisi1981correlation}}
	\begin{equation}\label{LD}
	\dB \theta(t) = -\nabla_{\theta} f(\theta(t)) \dB t + \sqrt{2} \dB B(t),
	\end{equation}
	where 	$\nabla f$ is called the drift term, $B(t)$ is a $d$-dimensional Brownian Motion and $\theta(t) \in \RBB^d$ is the state variable. %{\citep{karatzas2012brownian}}.
	%According to the Fokker-Planck equation\citep{risken1996fokker}, the stationary distribution of this dynamic is the target posterior distribution $p^*$.

	The classic Langevin dynamics Monte Carlo method (LMC) generates samples $\{x^{(k)}\}$ in the following manner:
	\begin{equation}\label{gld}
	x^{(k+1)} = x^{(k)}  - \eta  \nabla f(x^{(k)})+ \sqrt{2\eta} \xi^{(k)},
	\end{equation}
	where $x^{(k)}$ is the time discretization of the continuous time dynamics $\theta(t)$, $\eta$ is the stepsize and $\xi^{(k)} \sim \NB(0,\IB_{d \times d})$ is the $d$-dimensional Gaussian variable.
	Due to the randomness of $\xi^{(k)}$'s, $x^{(k)}$ is a random variable and we denote its distribution as $p^{(k)}$.
	$p^{(k)}$ is shown to converge weakly to the target distribution $p^*$ \citep{dalalyan2017further,raginsky2017non}.
	%Although this algorithm is simple and easy to implement, the full gradient computation makes LMC ill-suited for big data tasks.

	To alleviate the expensive full gradient computation in LMC, the Stochastic Gradient Langevin Dynamics (SGLD) replaces $\nabla f (x^{(k)})$ in \eqref{gld} by the stochastic approximation
	\begin{equation} \label{eqn_stochastic_approximation}
	g^{(k)} =\frac{N}{n}\sum_{i \in I_k}\nabla f_i(x^{(k)}),
	\end{equation}
	where $I_k$ is the set of $n$ indices independently and uniformly drawn from $[N]$ in iteration $k$.
	Although the gradient approximation \eqref{eqn_stochastic_approximation} is always an unbiased estimator of the full gradient, the non-diminishing variance results in the inefficiency of sample-space exploration and slows down the convergence to the target distribution.

	To overcome such difficulty, SVRG-LD and SAGA-LD \citep{dubey2016variance,chatterji2018theory,zou2019sampling} use the two different variance-reduced gradient estimators of $\nabla f(x)$, which utilize the component gradient information of the past samples.
%	\begin{align}
%	g^{(k)} =\sum_{i \in I_k}\frac{N}{n}(\nabla f_i(x^{(k)})-\alpha_{i}^{(k)})+ \sum_{i=1}^N \alpha^{(k)}_i,
%	\end{align}
%	where $\{\alpha^{(k)}_i\}$ is the snapshot set constituted by historical gradient information.
%	%
%	The difference between SVRG-LD and SAGA-LD lies in how we update the snapshot set: SVRG-LD uses the periodic total substitution strategy (Strategy \ref{alg:ipgu} in section ~\ref{sec:snapshotupdate}) while SAGA-LD uses the  per-iteration partial substitution strategy (Strategy \ref{alg:iagu} in section ~\ref{sec:snapshotupdate}).
%	Under the assumptions that $f$ is $M$-smooth and $\mu$-strongly convex, and has Lipschitz-continuous Hessian,
%	\citeauthor{chatterji2018theory} (\citeyear{chatterji2018theory}) proved that both methods need $\tilde \OM(N+(M/\mu)^{3/2}\sqrt{d}/{ \epsilon})$ component gradient evaluations ($\nabla f_i$'s) to ensure the sample distribution $p^{(k)}$ satisfy $\WM_2(p^{(k)}, p^*) \leq \epsilon$.
	%Later, \citep{zou2019sampling} shows that these methods needs $\tilde \OM(N+ \frac{N^{3/4}}{\epsilon^2}+\frac{N^{1/2}}{\epsilon^4})$.
	While possessing similar low per-iteration component gradient computation as in SGLD, the mixture time bound of SVRG-LD and SAGA-LD are shown to be similar to that of LMC under similar assumptions~\cite{chatterji2018theory,zou2019sampling}.
%
%	There is another research line which uses the mode of the log-posterior to construct
%	the control-variate estimate of the gradient
%	{\citep{baker2017control,bierkens2016zig,nagapetyan2017true,chatterji2018theory}}.
%	However, calculating the mode is intractable for large-scale problems and we exclude this type of methods from consideration.

	\section{Aggregated Gradient Langevin Dynamics}\label{section:AGLD}
	In this section, we present our general framework named Aggregated Gradient Langevin Dynamics (AGLD).
	Specifically, AGLD maintains a snapshot set consisting of component gradients evaluated in historical iterates.
	The information in the snapshot set is used in each iteration to construct a gradient approximation which helps to generate the next iterate.
	Note that iterates generated during the procedure are samples of random variables, whose distributions converge to the target distribution.
	At the end of each iteration, the entries in the snapshot set are updated according to some strategy.
	By customizing the steps in AGLD with different strategies, we can derive different algorithms.
	Concretely, AGLD is comprised of the following three steps, where the first and third steps can accept different strategies as inputs.
%	The three steps of AGLD are listed as follows.
	\begin{enumerate}[i]
		\item {\bf Data-Accessing:} select a subset of indices $S_k$ from $[N]$ according to the input strategy.
		\item {\bf Sample-Searching:} construct the aggregated gradient approximation $g^{(k)}$ using the data points indexed by $S_k$ and the historical snapshot set, then generate the next iterate (the new sample) by taking one step along the direction of $g^{(k)}$ with an injected Gaussian noise.
		Specifically, the $(k+1)$-th sample is obtained in the following manner
		\begin{equation}\label{para-update}
		x^{(k+1)} = x^{(k)} - \eta g^{(k)}+ \sqrt{2\eta}\xi^{(k)},
		\end{equation}
		where $\xi^{(k)}$ is a Gaussian noise, $\eta$ is the stepsize, and
		\begin{align}\label{U_appro}
		g^{(k)} =\sum_{i \in S_k}\frac{N}{n}(\nabla f_i(x^{(k)})-\alpha_{i}^{(k)})+ \sum_{i=1}^N \alpha^{(k)}_i.
		\end{align}

		\item {\bf Snapshot-Updating:} update historical snapshot set according to the input strategy.
	\end{enumerate}
	We summarize AGLD in Algorithm \ref{alg:iaglg}.
	While our mixture time analyses hold as long as the input {Data-Accessing} and {Snapshot-Updating} strategies meet Requirements \ref{DAreq} and \ref{assum5}, we describe in detail several typical qualified implementations of these two steps below.
	\floatname{algorithm}{Algorithm}
	\begin{algorithm}[tb]
		\caption{Aggregated Gradient Langevin Dynamics}
		\label{alg:iaglg}
		\begin{algorithmic}[1]
			\REQUIRE initial iterate $x^{(0)}$, stepsize $\eta$, {\bf Data-Accessing} strategy, and {\bf Snapshot-Updating} strategy.
			\STATE {\bf Initialize} Snapshot set $\AM^{(0)}=\{\alpha_i^{(0)}\}_{i=1}^N$, where  $\alpha_i^{(0)}=\nabla f_i(x^{(0)})$.
			\FOR{$k=0$ {\bfseries to} $K-1$}
			\STATE $S_k$ = {\bf Data-Accessing}(k).
			\STATE {\bf Sample-Searching}: find $x^{(k+1)}$ according to (\ref{para-update}).
			\STATE $\AM^{(k+1)}=$ {\bf Snapshot-Updating}($\AM^{(k)}, x^{(k)}, k, S_k$).
			\ENDFOR
%					\vskip -0.2in
		\end{algorithmic}

	\end{algorithm}

	\subsection{The Data-Accessing Step}
	We make the following requirement on the {Data-Accessing} step to ensure the convergence of $\WM_2$ distance between the sample distribution $p^{(k)}$ and the target distribution $p^*$.
	\begin{requirement} \label{DAreq}
		In every iteration, each point in the dataset has been visited at least once in the past $C$ iterations,
		where $C$ is some fixed positive constant.
	\end{requirement}
	We note that Requirement \ref{DAreq} is general and covers three commonly used data accessing strategies: Random Access (RA),  Random Reshuffle (RR), and Cyclic Access (CA).
	\begin{enumerate}[]
		\item
		{\bf RA:\ } Select uniformly $n$ indices from $[N]$ with replacement;
		\item
		{\bf RR:\ } Select sequentially $n$ indices from $[N]$ with a permutation at the beginning of each data pass;
		\item
		{\bf CA:\ } Select $n$ indices from $[N]$ in a cyclic way.
	\end{enumerate}
	RA is widely used to construct unbiased gradient approximations in gradient-based Langevin dynamics methods, which is amenable to theoretical analysis. However, in big data scenarios when the dataset does not fit into the memory, RA is not memory-friendly, since it entails heavy data exchange between memory and disks.
	On the contrary, CA strategy promotes the spatial locality property significantly and therefore reduces the page fault rate when handling huge datasets using limited memory~\citep{xie2018towards}. RR can be considered as a trade-off between RA and CA. However, methods with either CA or RR are difficult to analyze in that the gradient approximation is commonly not an unbiased estimator of the true gradient~\citep{shamir2016without}.

	It can be verified that these strategies satisfy Requirement \ref{DAreq},
	For RR, in the $k$-th iteration, all the data points have been accessed in the past $2N/n$ iterations. For CA, all the data points are accessed in the past $N/n$ iterations. Note that, RA satisfies the Requirement \ref{DAreq} with $C = \OM(N\log N)$ w.h.p., according to the Coupon Collector Theorem
	\citep{dawkins1991siobhan}.
	%following lemma, the proof of which can be found in the appendix.
	%\begin{lemma}\label{lem:ra}
	%	Suppose that $Z$ is a dataset with $N$ different data points.
	%	It we uniformly draw $n$ samples from $Z$ in each iteration, then we have
	%	\begin{equation}
	%	P(T > \beta \frac{N \ln N}{n}) < N^{1-\beta},
	%	\end{equation}
	%	where $T$ denotes the iterations we need to collect the whole dataset $Z$.
	%\end{lemma}
	%
	%
	\subsection{The Snapshot-Updating Step}\label{sec:snapshotupdate}
	\floatname{algorithm}{Procedure}
	The {Snapshot-Updating} step maintains a snapshot set $\AM^{(k)}$ such that in the $k$-th iteration, $\AM^{(k)}$ contains $N$ records $\alpha_i^{(k)}$ for $\nabla f_i(y_i^{(k)})$ where $y_i^{(k)}$ is some historic iterate $y_i^{(k)} = x^{(j)}$ with $j\leq k$.
	Additionally, for our analyses to hold, the input strategy should satisfy the following requirement.
	\begin{requirement}\label{assum5}
		The gradient snapshot set $\AM^{(k)}$ should satisfy $\alpha_i^{(k)} \in \{\nabla f_i(x^{(j)})\}_{j= k-D}^{k}$, where $D$ is a fixed constant.
	\end{requirement}
	This requirement guarantees that $\alpha_i^{(k)}$'s are not far from the $\nabla f_i(x^{(k)})$'s and thus can be used to construct a proper approximation of $\nabla f(x^{(k)})$.
	The Snapshot-Updating step tries to strike a balance between the approximation accuracy and the computation cost.
	Specifically, in each iteration, updating a larger portion of the $N$ entries in the snapshot set would lead to a more accurate gradient approximation at the cost of a higher computation burden.
%	\crr{Specifically, the more recently the snapshot is updated, the more accurate the approximation is and more computation it costs.}
	In the following, we list three feasible Snapshot-Updating strategies considered in this paper: Per-iteration Partial Update (PPU), Periodically Total Update (PTU), and Time-based Mixture Update (TMU).

		\floatname{algorithm}{Strategy}
		\begin{algorithm}[t]
			\caption{PTU($\AM^{(k)},x^{(k)},x^{(k+1)},k,S_k $)}
			\label{alg:ipgu}
			\begin{algorithmic}
				\FOR{$i=1$ {\bfseries to} $N$}
				\STATE
				\resizebox{1.02\linewidth}{!}{$\alpha_{i}^{(k+1)} = \IBB_{\{\textrm{ mod }(k+1,D) =0\}}\nabla f_i(x^{(k+1)}) + \IBB_{\{\textrm{ mod } (k+1,D) \ne 0 \}} \alpha_i^{(k)}$}
				\ENDFOR
				\RETURN $\AM_{k+1}$
			\end{algorithmic}
		\end{algorithm}
		\begin{algorithm}[t]
			\caption{PPU($\AM^{(k)},\theta^{(k)},k,S_k $)} \label{alg:iagu}
			\begin{algorithmic}
				\FOR{$i=1$ {\bfseries to} $N$}
				\STATE $\alpha^{(k+1)}_{i} = \IBB_{\{i \in S_k\}}\nabla f_i(x^{(k)}) + \IBB_{\{i \notin i_k\}} \alpha^{(k)}_i$
				\ENDFOR
				\RETURN $\AM_{k+1}$
			\end{algorithmic}
		\end{algorithm}
		\begin{algorithm}[t]
			\caption{TMU($\AM^{(k)},\theta^{(k)},k,S_k $)}
			\label{alg:higu}
			\begin{algorithmic}
				\FOR{$i=1$ {\bfseries to} $N$}
				\IF{$\textrm {mod}(k+1,D) =0$}
				\STATE $\alpha_{i}^{(k+1)} = \nabla f_i(x^{(k+1)})$
				\ELSE
				\STATE $\alpha_{i}^{(k+1)} = \IBB_{\{i \in S_k\}}\nabla f_i(x^{(k)}) + \IBB_{\{i \notin S_k \}} \alpha_i^{(k)}$
				\ENDIF
				\ENDFOR
				\RETURN $\AM_{k+1}$
			\end{algorithmic}
		\end{algorithm}
		\begin{enumerate}[]
		\item
		\noindent{\bf PTU:\ } This strategy operates in an epoch-wise manner: at the beginning of each epoch all the entries in the snapshot set are updated to the current component gradient $\alpha_i^{(k)} = \nabla f_i(x^{(k)})$, and in the following $D\!-\!1$ iterations the snapshot set remains unchanged (see Strategy~\ref{alg:ipgu}).
		Such synchronous update to the snapshot set allows us to implement PTU in a memory efficient manner.
		In the $k$-th iteration, PTU only needs to store the iterate $\tilde{x}$ and its gradient $\nabla f(\tilde{x})$ where $\tilde{x} = x^{k - \mathrm{mod}(k, D)}$, as we can obtain the snapshot entry $\alpha_i^{(k)}$ via a simple evaluation of the corresponding component gradient at $\tilde{x}$ in the calculation of $g^{(k)}$.
		Therefore the PTU strategy is preferable when storage is limited.
		\item
		\noindent{\bf PPU:\ } This strategy substitutes $\alpha^{(k)}_{i}$ by $\nabla f_i(x^{(k)})$ for $i \in S_k$ in the $k$-th iteration (see Strategy \ref{alg:iagu}).
		This partial substitution strategy together with Requirement \ref{DAreq} can ensure the Requirement \ref{assum5}.
		The downside of PPU is the extra $\OM(d\cdot N)$ memory used to keep the snapshot set $\AM^{(k)}$.
		Fortunately, in many applications of interests, $\nabla f_i(x)$ is actually the product of a scalar and the data point $z_i$, which implies that only $\OM(N)$ extra storage is needed to store $N$ scalars.

		\item
		\noindent{\bf TMU:\ } This strategy updates the whole $\AM$ once every $D$ iterations and substitutes $\alpha^{(k)}_{i}$ by $\nabla f_i(x^{(k)})$ in the $k$-th iteration (see Strategy \ref{alg:higu}).
		TMU possesses the merits of both PPU and PTU: it updates components of gradient snapshot set in $S_k$ to newly computed one in each iteration as PPU, and also periodically updates the whole snapshot set as PTU in case that there exist indices unselected for a long time.
		Note that both PTU and TMU need one extra access to the whole dataset every $D$ iterations.
		Practically, we usually choose $D = c N$, which makes PTU and TMU have an extra $1/c$ averaged data point access in each iteration.
	\end{enumerate}

	\begin{remark}
		PPU is the Snapshot-Updating strategy used in SAGA-LD and PTU is the strategy used in SVRG-LD {\citep{dubey2016variance}}.
		To the best of our knowledge, TMU has never been proposed in the MCMC literature before.
		Note that the HSAG Snapshot-Updating strategy proposed by \citeauthor{reddi2015variance}~\shortcite{reddi2015variance} also satisfies our requirement, and we omit the discussion of it due to the limit of space.
	\end{remark}

	\subsection{Derived Algorithms} \label{section_derived_algorithm}
	By plugging the aforementioned Data-Accessing and Snapshot-Updating strategies into AGLD, we derive several practical algorithms.
	We name the algorithms by "Snapshot-updating - Data-Accessing", e.g. TMU-RA uses TMU as the Snapshot-Updating strategy and RA as the Data-Accessing strategy.
%	PTU-RA, PTU-RR, PTU-CA, PPU-RA, PPU-RR, PPU-CA, TMU-RA, TMU-RR and TMU-CA.
%	For example,
	Note that we recover SAGA-LD and SVRG-LD with PPU-RA and PTU-RA, respectively.
%	Note that all the algorithms are newly proposed except that PPU-RA and PTU-RA are actually SAGA-LD and SVRG-LD, respectively.
	In the following section, we provide unified analyses for all derived algorithms under different regularity conditions.
	We emphasize that, in the absence of the unbiasedness of the gradient approximation, our mixture time analyses are the first to cover algorithms with the I/O friendly cyclic data accessing scheme.
%	\begin{table}[t]
%		\caption{Practical Algorithms with different strategies}
%		\label{alg-table}
%		\vskip 0.15in
%		\begin{center}
%			\begin{small}
%				\begin{sc}
%					\begin{tabular}{lcc}
%						\toprule
%						Algorithm & \tabincell{c}{Data\\Accessing} & \tabincell{c}{Schedule\\Updating}\\
%						\midrule
%						\tabincell{l}{PPU-RA(SAGA-LD)\\\citep{dubey2016variance}}& RA & PPU\\
%						PPU-RR&RR&PPU\\
%						PPU-CA& CA & PPU\\
%						\tabincell{l}{PTU-RA(SVRG-LD)\\\citep{dubey2016variance}}& RA &  PTU\\
%						PTU-RR& RR & PTU\\
%						PTU-CA& CA & PTU\\
%						TMU-RA& RA & TMU\\
%						TMU-RR&RR & TMU\\
%						TMU-CA& CA & TMU\\
%						\bottomrule
%					\end{tabular}
%				\end{sc}
%			\end{small}
%		\end{center}
%		\vskip -0.1in
%	\end{table}
%
%

	\section{Theoretical Analysis}
	In this section, we provide the mixture time analysis for AGLD.
%	In particular, we say $K$ is the $\epsilon$-mixture time of a Monte Carlo simulation procedure if for $k\geq K$, the generated sample $\xB^{(k)}$ has its distribution $p^{(k)}$ $\epsilon$-close to the target distribution $p^*$ in $\WM_2$ distance, i.e. $\WM_2(p^{(k)}, p^*)\leq\epsilon$.
	The detailed proofs of the theorems are postponed to the Appendix due to the limit of space.
	%\footnote{{We upload our appendix in an anonymous website \url{https://drop.me/oWr3w3}.}}.

	%We first analyse the mixture time of AGLD under the strongly convexity and smoothness assumption of $f$.
	%We also show that if $f$ satisfies an extra Hessian Lipschitz assumption, the bound for AGLD with RA Data-Accessing strategy can be improved.
	%Besides, we also generalize the analysis to smooth and general convex $f$.

	%%state some extension of main result that might lead the

	\subsection{Analysis for AGLD with strongly convex $f(x)$}\label{sec:theory}
	%We measure the sampling accuracy of AGLD with the $\WM_2$ distance.
	We first investigate the $\WM_2$ distance between the sample distribution $p^{(k)}$ of the iterate $x^{(k)}$ and the target distribution $p^*$ under the smoothness and strong convexity assumptions.

	\begin{assumption}[Smoothness]\label{assum:smooth}
	Each individual $f_i$ is $\tilde M$-smooth.
		That is, $f_i$ is twice differentiable and there exists a constant $\tilde M > 0$ such that for all $x,y \in \RBB^d$
		\begin{equation}
		f_i(y) \le f_i(x)+ \langle \nabla f_i(x), y-x \rangle + \frac{\tilde M}{2} \|x-y\|^2_2.
		\end{equation}
		Accordingly, we can verify that the summation $f$ of $f_i's$ is $M$-smooth with $M = \tilde M N$.
	\end{assumption}

	\begin{assumption}[Strong Convexity]\label{assum:sconvex}
		The sum $f$ is $\mu$-strongly convex.
		That is, there exists a constant $\mu > 0$ such that for all $x,y \in \RBB^d$,
		\begin{equation}
		f(y) \ge f(x)+ \langle \nabla f(x), y-x \rangle + \frac{\mu}{2} \|x-y\|^2_2.
		\end{equation}
	\end{assumption}
	Note that these assumptions are satisfied by many Bayesian sampling models such as Bayesian ridge regression, Bayesian logistic regression and Bayesian Independent Component Analysis, and they are used in many existing analyses of Langevin dynamics based MCMC methods \citep{dalalyan2017theoretical,baker2017control,zou2018subsampled,chatterji2018theory}.

	\begin{theorem}\label{them:main1}
		%Assume that f is $M$-smooth and $\mu$-strongly convex, and Requirement \ref{assum5} is satisfied.
		Under Assumption \ref{assum:smooth}, \ref{assum:sconvex} and Requirement \ref{assum5},
		AGLD outputs sample $\xB^{(k)}$ with its distribution $p^{(k)}$ satisfying $\WM_2( p^{(k)}, p^*) \!\le\! \epsilon$ for any $k\!\ge\! K\! =\!\tilde  \OM(\epsilon^{-2})$ with $\eta \!= \!\OM(\epsilon^2)$.
%		where $p^{(K)}$ is the distribution of $x^{(K)}$.
	\end{theorem}
%	\crr{Describe that $K$ is your mixture time!}
%	\begin{remark}
%		Note that the biasness of the gradient approximation in CA and RR will make AGLD variants with CA/RR have a bigger constant in the $\OM(\cdot)$ than algorithms with RA, which indicates that algorithms with RA need fewer iterations to get the same accuracy in 2-Wasserstein distance.
%		However, when the dataset is not fitted to  the memory, the sequential data accessing nature of CA will result in less I/O cost than random data accessing, which makes AGLD variants with CA has better performance than the algorithms with RA in terms of the time efficiency.
%	\end{remark}

	\begin{remark}
		Under this assumption, the $\epsilon$-mixture time $K$ of AGLD has the same dependency on $\epsilon$ as that of LMC \citep{dalalyan2017theoretical}.
		Note that we hide the dependency of other regularity parameters such as $\mu$, $L$ and $N$ in the $\OM(\cdot)$ for simplicity.
		Actually, AGLD methods with CA/RR have a worse dependency on these parameters than algorithms with RA.
		However, when the dataset does not fit into the memory, the sequential data accessing nature of CA enjoys less I/O cost than random data accessing, which makes CA based AGLD methods have a better time efficiency than the RA based ones.
	\end{remark}

	%\noindent{\bf Tighter Bounds under Further Assumptions}\label{sec:improvedtheory}
	%\citep{chatterji2018theory} give the convergence analysis in 2-Wasserstein distance for SVRG-LD and SAGA-LD under additional Hessain lipschitz condition as follows.
	The bound of the mixture time for AGLD with RA can be improved under the Lipschitz-continuous Hessian condition.

	\begin{assumption}\label{assum:lch} [Lipschitz-continuous Hessian]
	There exists a constant $L > 0$ such that for all $x,y \in \RBB^d$
	\begin{equation*}
	\|\nabla^2 f(x) - \nabla^2 f(y)\| \le L \|x-y\|_2.
	\end{equation*}
\end{assumption}

	%Under this extra assumption, they showed that SVRG-LD and SAGA-LD need $\eta =  \OM(\epsilon)$ and $K = \OM(\frac{1}{\epsilon}\log \frac{1}{\epsilon})$ to get $\epsilon$-accuracy 2-Wasserstein distance.
	%Actually, with Hessian Lipschitz condition, we can also get the similar  result for algorithms in AGLD with RA.
	%With this extra assumption, we can establish the following theorem for AGLD with RA.
	\begin{theorem}\label{them:impro}
		%Assume that $f$ is $M$-smooth, $\mu$-strongly convex, and has {Lipschitz-continuous Hessian}, then
		Under Assumption \ref{assum:smooth}, \ref{assum:sconvex}, \ref{assum:lch} and Requirement \ref{assum5},
		AGLD methods with {RA}
		output sample $\xB^{(k)}$ with its distribution $p^{(k)}$ satisfying $\WM_2( p^{(k)}, p^*) \le \epsilon$ for any $k\!\ge\!K\! =\! \OM(\log ({1}/{\epsilon})/\epsilon)$ by setting $\eta = \OM(\epsilon)$.
%		where $p^{(K)}$ is the distribution of $x^{(K)}$.
	\end{theorem}
	%Note that, even with this extra assumption, the bound of mixture time and total component gradient estimations for AGLD with RR/CA are still the same as in Theorem ~\ref{them:main}~, which is inferior to AGLD with RA.

	%\citep{chatterji2018theory} give the convergence analysis for SVRG-LD and SAGA-LD under the same smoothness, strongly convexity and Hessian Lipschitz condition by designing two delicate Lyapunov function for SVRG-LD and SAGA-LD, respectively.
	%They also show that SVRG-LD and SAGA-LD achieve $\epsilon$-accuracy $2$-Wasserstein distance after $K = \OM(\frac{1}{\epsilon}\log \frac{1}{\epsilon})$ iterations  and $T_g = \OM(N+ \frac{1}{\epsilon}\log \frac{1}{\epsilon})$ component gradient estimations.

	%\begin{remark}\label{remark:denpendenceyonconstants}
	%	In Theorem ~\ref{them:main} and Theorem ~\ref{them:impro}~, we give a unified analysis for the AGLD algorithm and figure out the mixture time and total component gradient estimation we need w.r.t $\epsilon$ in order to obtain an $\epsilon$-accuracy 2-Wasserstein distance.
	%	The dependency on $N$, $M$, $\mu$ and $L$ are hidden into the $\OM(\cdot)$, since the bound are rather rough in order to make the analysis general enough and the tightest dependency usually needs to be specified individually for different Data-Accessing  and Snapshot-Updating strategies.
	%\end{remark}

	Additionally, when we adopt the random data accessing scheme, the mixture time of the newly proposed TMU-RA method can be written in a more concrete form, which is established in the following theorem.
%	Similar to \cite{chatterji2018theory,zou2018subsampled}, we treat $M$,$\mu$ and $L$ as constants of order $\OM(N)$ .
	%We figure out the dependence of the mixing time $k$ and component gradient complexity $T_g$(total $\nabla f_i(\cdot)$ computations) on the dimension $d$, the data size $N$, the condition number $\kappa= M/\mu$ and target accuracy $\epsilon$ under the smoothness, strongly convex and Hessian Lipschitz assumptions

	%Following the same setting as \citep{chatterji2018theory} that $M$, $\mu$ and $L$ are scaling linear with $N$ and $N \gg d$,  the mixture time $K$ and the number of gradient estimations $T_g$  for TMU-RA can be bounded as below.
	\begin{theorem}\label{them:tmura}
		%Assume that $f$ is $M$-smooth, $\mu$-strongly convex and has Lipschitz-continuous Hessian and
		Under Assumption \ref{assum:smooth}, \ref{assum:sconvex}, \ref{assum:lch} and denote $\kappa = {M}/{\mu}$.
		TMU-RA outputs sample $\xB^{(k)}$ with its distribution $p^{(k)}$ satisfying $\WM_2( p^{(k)}, p^*) \le \epsilon$ for any $k\geq K = \tilde \OM(\kappa^{3/2}\sqrt{d}/{(n \epsilon)})$ if we set $\eta < \epsilon n \sqrt{\mu}/ {(M \sqrt{d N})}$, $n \ge 9$, and $D =  N$.
%		Further, to achieve such state, the corresponding component gradient evaluations is $T_g = \tilde \OM(N+{\kappa^{3/2}\sqrt{d}}/{ \epsilon})$.
%		Note that we hide the logarithm terms in $\tilde \OM$.
	\end{theorem}
	\begin{remark}
		Note that the component gradient complexity to achieve $\WM_2(p^{(k)}, p^*)\leq\epsilon$ in TMU-RA is $T_g = \tilde \OM(N+{\kappa^{3/2}\sqrt{d}}/{ \epsilon})$, which is the same as those of SAGA-LD \citep{chatterji2018theory} and SVRG-LD \citep{zou2018subsampled}.
		Practically, in our experiments, TMU based variants always have a better empirically performance than the PPU based and PTU based counterparts as the entries in the snapshot set maintained by TMU is more up-to-date.
%		However, since TMU always updates components in the snapshot set to newly computed ones as PPU and also periodically update the whole snapshot set to rule out the out-of-date ones as PTU, it always has a better empirically performance, as observed in our experiments.
	\end{remark}

	%Note that the bound on the component gradient complexity $T_g = \tilde \OM(N+\frac{\kappa^{3/2}\sqrt{d}}{ \epsilon})$ are the same as that of SAGA-LD(proved in \citep{chatterji2018theory}) and SVRG-LD(shown in \citep{zou2018subsampled}).
	%However, since TMU always updates components in the snapshot to newly computed ones as PPU and also periodically update the whole snapshot set to rule out the out-of-date ones as PTU, it always has a better empirically performance, as shown in the experimental section.

	\subsection{Extension to general convex $f(x)$}
	%In the previous subsections, we provide the theoretical guarantee for AGLD under the strong convexity assumption.
	%The AGLD algorithm can also be used to sample w.r.t. general convex $f$ which is $L$-smooth.

	Following a similar idea from \citep{zou2018stochastic}, we can extend AGLD to drawing samples from densities with general convex $f(x)$.
	Firstly, we construct the following strongly convex approximation $\hat{f}(x)$ of $f(x)$,
	\[
	\hat f(x) = f(x) + {\lambda \|x\|^2}/{2}.
	\]
	%It can be verified that $\hat f$ is $\lambda$ strongly convex and $(L+\lambda)$ smooth.
	Then, we run AGLD to generate samples with $\hat{f}(x)$ until the sample distribution $p^{(K)}$ satisfies $\WM_2(p^{(K)}, \hat{p}^*) \le \epsilon/2$ where $\hat p^* \propto e^{-\hat f(x)}$ denotes stationary distribution of Langevin Dynamics with the drift term $\nabla \hat f$ (check \ref{LD} for definition).
	If we choose a proper $\lambda$ to make $\WM_2(\hat p^*, p^*) \le {\epsilon}/{2}$, then by the triangle inequality of the $\WM_2$ distance, we have $\WM_2( p^{(K)}, p^*) \le \WM_2( p^{(K)}, p^*)+\WM_2(\hat p^*, p^*) \le \epsilon$.
	Thus, we have the following theorem.
%	If we assume that the target has bounded forth order moment, we can choose $\lambda$ according to the following theorem.
	\begin{theorem}\label{them:general}
		Suppose the assumptions in Theorem \ref{them:main1} hold and further assume the target distribution $p^* \propto e^{-f}$ has bounded forth order moment, i.e. $\EBB_{p^*} [\|x\|_2^4]\le \hat U d^2$.
		If we choose $\lambda = 4\epsilon^2 / (\hat U d^2)$ and run the AGLD algorithm with $\hat f(x) = f(x) + {\lambda \|x\|^2}/{2}$, we have $\WM_2( p^{(k)}, p^*) \le \epsilon$ for any $k\!\ge\! K \!= \!\tilde \OM(\epsilon^{-8})$.
		If we further assume that $f$ has Lipschitz-continuous Hessian, then SVRG-LD, SAGA-LD, and TMU-RA can achieve $\WM_2( p^{(K)}, p^*)\! \le\! \epsilon$ in $K =\tilde  \OM(\epsilon^{-3})$ iterations.
	\end{theorem}

	\subsection{Theoretical results for nonconvex $f(x)$}
	In this subsection, we characterize the $\epsilon$-mixture time of AGLD for sampling from densities with nonconvex $f(x)$.
	The following assumption is necessary for our theory.
	\begin{assumption}\label{assump:dis}[Dissipative]
	There exists constants $a,b > 0$ such that for all $x\in \RBB^d$,
	the sum $f$ satisfies
	$$
	\langle \nabla f(x), x \rangle \ge b\|x\|_2^2 - a.
	$$
	\end{assumption}
	This assumption is typical for the ergodicity analysis of stochastic differential equations and diffusion approximations.
	It indicates that, starting from a position that is sufficiently far from the origin, the Lagevin dynamics (\ref{LD}) moves towards the origin on average.
	With this assumption, we establish the following theorem on the nonasymptotic convergence of AGLD for nonconvex $f(x)$.
	\begin{theorem}
		Under Assumption \ref{assum:smooth}, \ref{assump:dis}, and Requirement \ref{assum5}, AGLD
		outputs sample $\xB^{(k)}$ with distribution $p^{(k)}$ satisfying $\WM_2( p^{(k)}, p^*) \le \epsilon$ for any $k\!\ge\! K \!=\!\tilde \OM(\epsilon^{-4})$ with $\eta = \OM(\epsilon^4)$.
	\end{theorem}

	\begin{remark}
	This $\tilde{\OM}(\epsilon^{-4})$ result is similar to the bound for LMC sampling from nonconvex $f(x)$~\citep{raginsky2017non}.
	Note that, as pointed out by \citep{raginsky2017non}, vanilla SGLD fails to converge in this setting.
	\end{remark}
%\crr{Discuss how such result does not exist in the literature to emphasize your contributions.}
%	Note that,
	%By substituting the smoothness constant $(L+\lambda)$ and strongly convex constant $\lambda$ of $\hat f$ into the proof of Theorem ~\ref{them:main} ,Theorem ~\ref{them:impro} and Theorem \ref{them:tmura}, we can figure out the dependence of mixture time $K$ and total component gradient count $T_g$ on $\epsilon$ to get $\epsilon$-accuracy $2$-Wasserstein distance for general AGLD algorithm.
	%However, according to Remark \ref{remark:denpendenceyonconstants}, the bound would be rather rough and we do not list it here.

	%Note that we have figure out the dependency on the condition number $\kappa$ for TMU-RA in Theorem 3.
	%Combining Theorem \ref{them:tmura} and Lemma ~\ref{lemma:piandbarpi}, we have the following result.
	%\begin{theorem}
	%	Assuming that the target distribution $\pi \propto e^{-f}$ has bounded forth order moment and the $L$-smoothness and $L$ Hessian Lipschitz of $f$, running TMU-RA with $\hat f(x) = f(x) + \frac{\lambda \|x\|^2}{2}$ would guarantee that $\WM_2(p^{(k)},p^*)\le \epsilon$ after $\tilde \OM(N + \frac{N^{3/2}d^{7/2}}{\epsilon^3})$ component gradient evaluations.
	%\end{theorem}

	%Note that, since  the smoothness and strongly convex constants of $\bar f$ now depends on $\epsilon$, it is difficult to figure out the dependence of mixture time $K$ and total component gradient count $T_g$ on $\epsilon$ to get $\epsilon$-accuracy $2$-Wasserstein distance for general AGLD algorithm.

	%
	%
	%
	%
	%
	%
	%
	%
	%
	%
	%
	%
	%
	%
	%

		\begin{figure*}[!t]
		\centering
		\begin{subfigure}
			\centering
			\includegraphics[trim={0cm 0cm 0cm 0cm},clip,width=3cm,height =3cm]{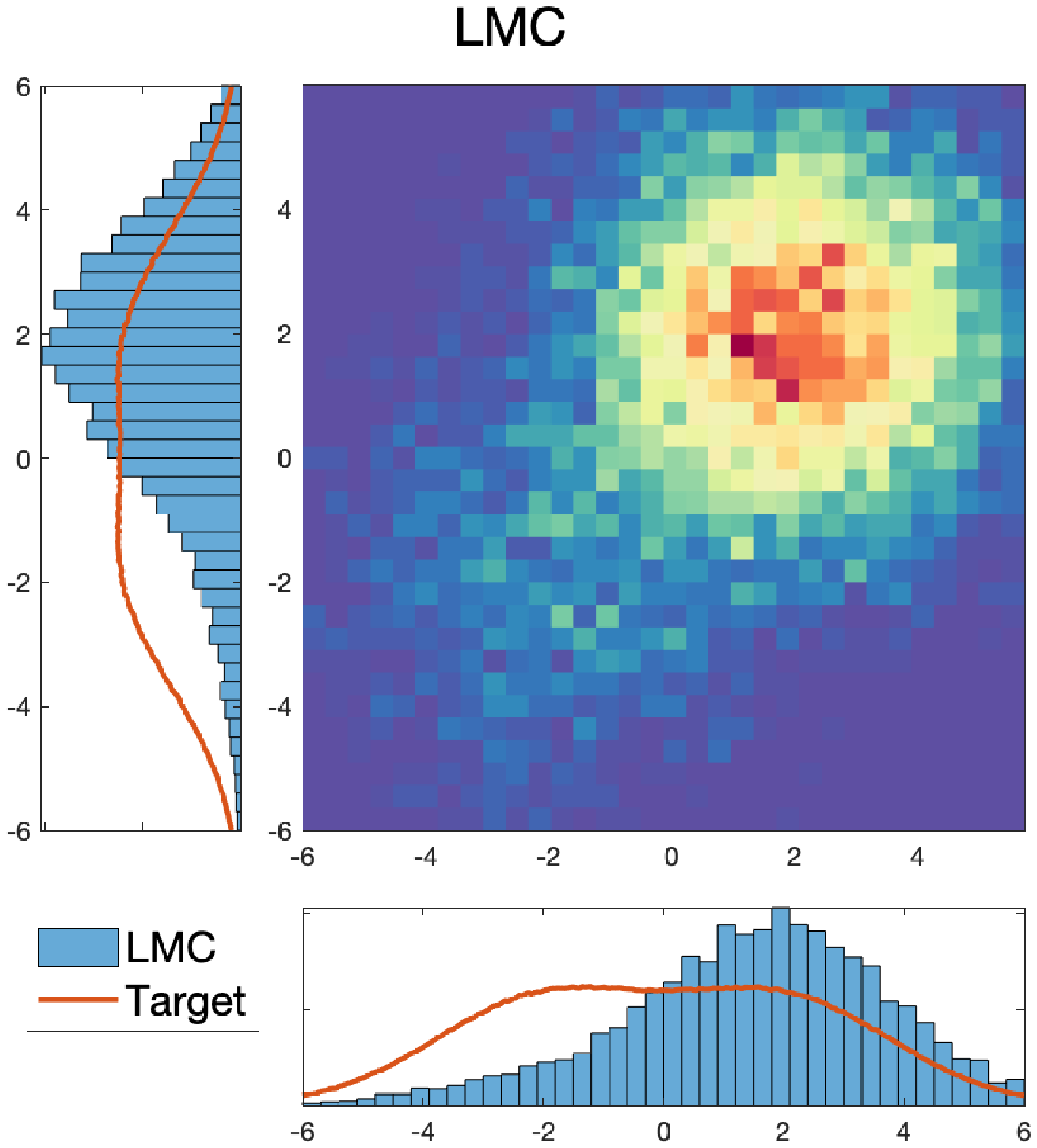}
		\end{subfigure}
		~
		\begin{subfigure}
			\centering
			\includegraphics[trim={0cm 0cm 0cm 0cm},clip,width=3cm,height =3cm]{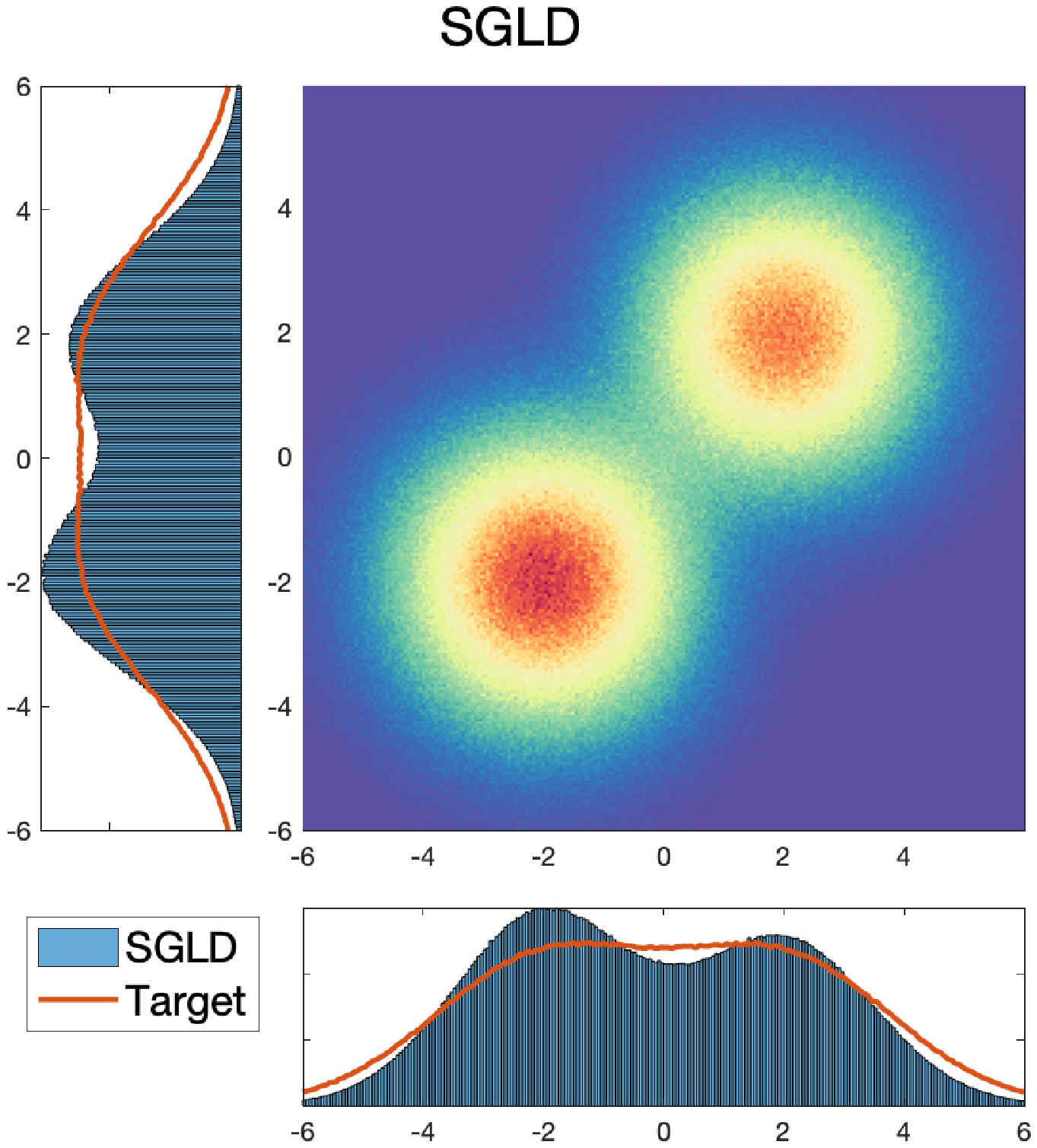}
		\end{subfigure}
		~
		\begin{subfigure}
			\centering
			\includegraphics[trim={0cm 0cm 0cm 0cm},clip,width=3cm,height =3cm]{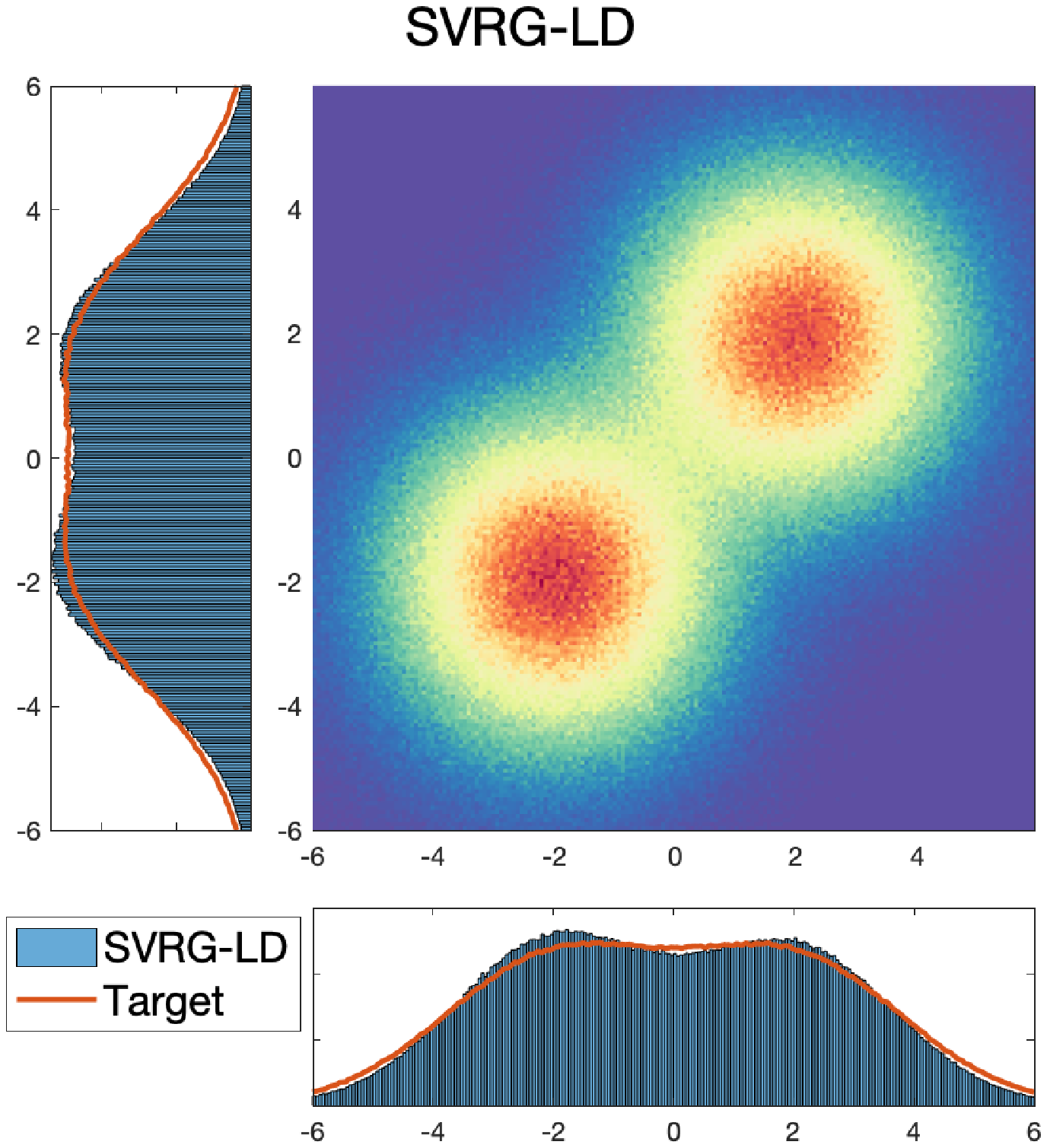}
		\end{subfigure}
		~
		\begin{subfigure}
			\centering
			\includegraphics[trim={0cm 0cm 0cm 0cm},clip,width=3cm,height =3cm]{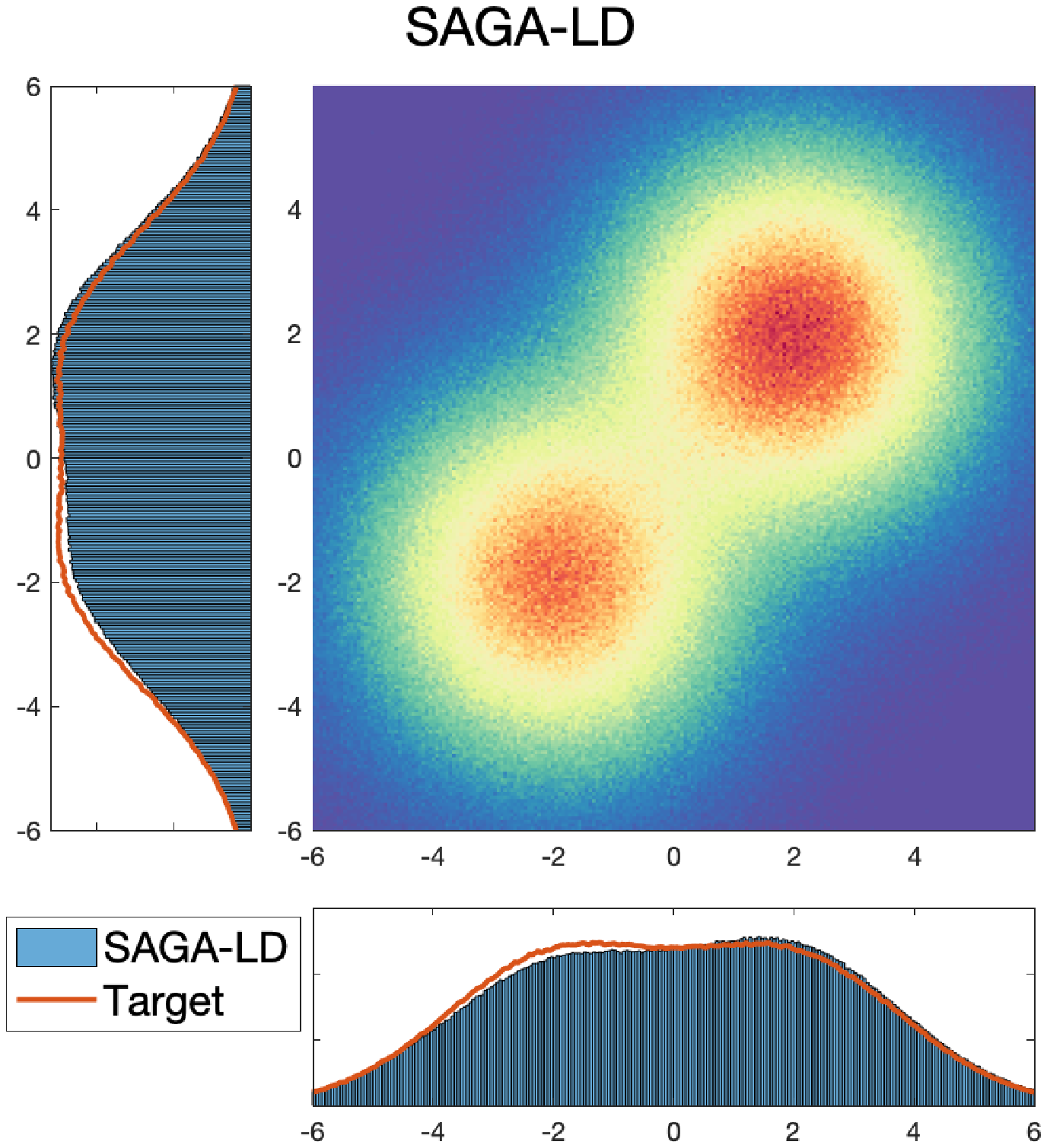}
		\end{subfigure}
		~
		\begin{subfigure}
			\centering
			\includegraphics[trim={0cm 0cm 0cm 0cm},clip,width=3cm,height =3cm]{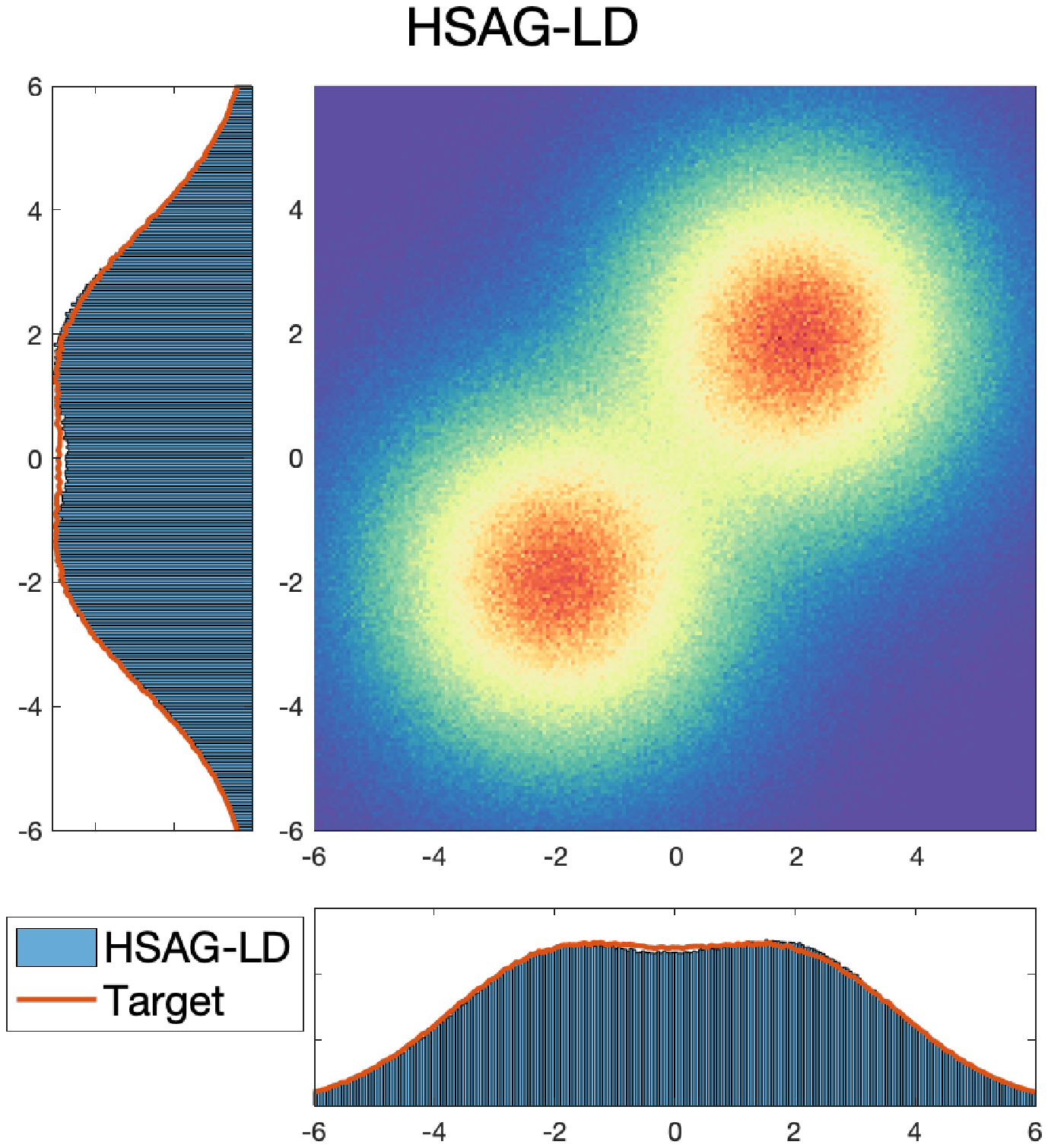}
		\end{subfigure}
		~
		\caption{Gaussian Mixture Model. The red line denotes the projection of the target distribution $p^*$. }
		%	The first three rows shows the results for difference methods with RA, RR and CA, respectively.In the last row, we compare the performance of TMU with different Data-Accessing strategies.}
		\label{fig:expGM}
		\vskip -0.25in
	\end{figure*}

	\section{Related Work}
	In this section, we briefly review the literature of Langevin dynamics based MCMC algorithms.

	By directly discretizing the Langevin dynamics~\eqref{LD}, \citeauthor{roberts1996exponential}~\shortcite{roberts1996exponential} proposed to use LMC~\eqref{gld} to generate samples of the target distribution.
	The first nonasymptotic analysis of LMC was established by \citeauthor{dalalyan2017theoretical}~\shortcite{dalalyan2017theoretical}, which analyzed the error of approximating the target distribution with strongly convex $f(x)$ in the total variational distance.
	This result was soon improved by \citeauthor{durmus2017nonasymptotic}~\shortcite{durmus2017nonasymptotic}.
	Later, \citeauthor{durmus2016high}~\shortcite{durmus2016high} and \citeauthor{cheng2018convergence}~\shortcite{cheng2018convergence}  established the convergence of LMC in the 2-Wasserstein distance and KL-divergence, respectively.
	%Recently, \citeauthor{dalalyan2017further} (~\citeyear{dalalyan2017further}~)
	%improved the existing results in terms of the 2-Wasserstein distance and provided further insights on the close relation between approximate sampling and optimization.
	While the former works focus on sampling from distribution with (strongly-)convex $f(x)$,
	\citeauthor{raginsky2017non} (\citeyear{raginsky2017non}) investigated the nonasymptotic convergence of LMC in the 2-Wasserstein distance when $f(x)$ is nonconvex.
	%
	% As for MALA,
	Another line of work incorporates LMC with Metropolis Hasting (MH) correction step~\cite{hastings1970monte}, which gives rise to Metropolis Adjusted Langevin Algorithm (MALA)~\cite{roberts1998optimal}.
	\citeauthor{eberle2014error}~\shortcite{eberle2014error} and \citeauthor{dwivedi2018log}~\shortcite{dwivedi2018log} proved the nonasymptotic convergence of MALA for sampling from distribution with general convex and strongly convex $f(x)$, respectively.
	Typically, MALA has better mixture bounds than LMC under the same assumption due to the extra correction step.
%	due to the extra MH step.
	However, the MH correction step needs extra full data access, and is not suitable for large-scale Bayesian learning tasks.

	With the increasing amount of data size in modern machine learning tasks,
	SGLD method\citep{welling2011bayesian}, which replaces the full gradient in LMC with a stochastic gradient~\cite{robbins1951stochastic}, has received much attention.
	\citeauthor{vollmer2015non}~\shortcite{vollmer2015non} analyzed the nonasymptotic bias and variance of SGLD using Poisson equations, and \citeauthor{dalalyan2017user}~\shortcite{dalalyan2017user} proved the convergence of SGLD in the 2-Wasserstein distance when the target distribution is strongly log-concave.
	%Later, \citeauthor{chen2015convergence} (\citeyear{chen2015convergence}) extended the stochastic gradient approximation technique to general stochastic MCMC method with higher-order integrator, and analyze the MSE of the average sample path.
	Despite the great success of SGLD, the large variance of stochastic gradients may lead to unavoidable bias
	%due to the lack of MH correction
	\citep{baker2017control,betancourt2015fundamental,brosse2018promises}.
	To overcome this, \citeauthor{teh2016consistency} \shortcite{teh2016consistency} proposed to decrease the step size to alleviate the bias and proved the asymptotic rate of SGLD in terms of Mean Square Error (MSE).
	\citeauthor{dang2019hamiltonian}~\shortcite{dang2019hamiltonian}
	 utilized an approximate MH correction step, which only uses part of the whole data set, to decrease the influence of variance.

	Another way to reduce the variance of stochastic gradients and save gradient computation is to apply variance-reduction techniques.
	\citeauthor{dubey2016variance}~\shortcite{dubey2016variance} used two different variance-reduced gradient estimators of $\nabla f(x)$, which utilize the component gradient information of the past samples, and devised SVRG-LD and SAGA-LD algorithms.
	They proved that these two algorithms improve the MSE upon SGLD.
	\citeauthor{chatterji2018theory} (\citeyear{chatterji2018theory}) and \citeauthor{zou2019sampling}~\shortcite{zou2019sampling} studied the nonasymptotic convergence of these methods in the 2-Wasserstein distance when sampling from densities with strongly convex and nonconvex $f(x)$, respectively.
	Their results show that SVRG-LD and SAGA-LD can achieve similar $\epsilon$-mixture time bound as LMC w.r.t. $\epsilon$, while the per-iteration computational cost is similar to that of SGLD.
	%	\begin{align}
	%	g^{(k)} =\sum_{i \in I_k}\frac{N}{n}(\nabla f_i(x^{(k)})-\alpha_{i}^{(k)})+ \sum_{i=1}^N \alpha^{(k)}_i,
	%	\end{align}
	%	where $\{\alpha^{(k)}_i\}$ is the snapshot set constituted by historical gradient information.
	%	%
	%	The difference between SVRG-LD and SAGA-LD lies in how we update the snapshot set: SVRG-LD uses the periodic total substitution strategy (Strategy \ref{alg:ipgu} in section ~\ref{sec:snapshotupdate}) while SAGA-LD uses the  per-iteration partial substitution strategy (Strategy \ref{alg:iagu} in section ~\ref{sec:snapshotupdate}).
	%	Under the assumptions that $f$ is $M$-smooth and $\mu$-strongly convex, and has Lipschitz-continuous Hessian,
	%	both methods need $\tilde \OM(N+(M/\mu)^{3/2}\sqrt{d}/{ \epsilon})$ component gradient evaluations ($\nabla f_i$'s) to ensure the sample distribution $p^{(k)}$ satisfy $\WM_2(p^{(k)}, p^*) \leq \epsilon$.
	%
	There is another research line which uses the mode of the log-posterior to construct control-variate estimates of full gradients
	{\citep{baker2017control,bierkens2016zig,nagapetyan2017true,chatterji2018theory,brosse2018promises}}.
	However, calculating the mode is intractable for large-scale problems, rendering these methods impractical for real-world Bayesian learning tasks.

	\section{Experiments}\label{section:exp}

	\begin{table}[!t]
	\caption{Statistics of datasets used in our experiments.}
	\label{tb:datasets}
	\vskip 0.15in
	\begin{center}
		\begin{small}
			\begin{sc}
				\begin{tabular}{lcc}
					\toprule
					Dataset & dimension & datasize\\
					\midrule
					%					noise & 5 & 1,503  \\
					%					parkinsons &19& 5,875 \\
					YearPredictionMSD &90& 515,345 \\
					SliceLoaction & 384&53500  \\
					criteo     & 999,999& 45,840,617 \\
					kdd12      & 54,686,45    & 149,639,105 \\
					\bottomrule
				\end{tabular}
			\end{sc}
		\end{small}
	\end{center}
	\vskip -0.3in
\end{table}

	We follow the experiment settings in the literature  {\citep{zou2018subsampled,dubey2016variance,chatterji2018theory,welling2011bayesian,zou2019sampling}}
	and conduct empirical studies on two simulated experiments (sampling from distribution with convex and nonconvex $f$, respectively) and two real-world applications (Bayesian Logistic Regression and Bayesian Ridge Regression).
	% to compare the performance of different Data-Accessing and Snapshot-Updating strategies.
	Nine instances of AGLD are considered, including SVRG-LD (PTU-RA), PTU-RR, PTU-CA, SAGA-LD (PPU-RA), PPU-RR, PPU-CA, TMU-RA, TMU-RR, and TMU-CA.
%	Note that all the algorithms except for SVRG-LD and SAGA-LD are newly derived ones from the AGLD framework.
	We also include LMC, SGLD, SVGR-LD+  {\citep{zou2018subsampled}}, SVRG-RR+ and SVRG-CA+\footnote{SVRG-RR+ is the random reshuffle variant of SVRG-LD+, and SVRG-CA+ is the cyclic access variant of SVRG-LD+.} as baselines.
%	Note that we do not include control-variates methods such as CV-ULD since experiments in \citep{chatterji2018theory} illustrate that their performances are worse than SVRG-LD and SAGA-LD.
	Due to the limit of space, we
	%only list part of our results here, and
	put the experiment sampling from distribution with convex $f$ into the Appendix.
	The statistics of datasets are listed in Table~\ref{tb:datasets}.

	\subsection{Sampling for Gaussian Mixture Distribution}
	\begin{figure*}[t]
		\centering
		\begin{subfigure}
			\centering
			\includegraphics[trim={0cm 0cm 0cm 0cm},clip,width=3.7cm,height =3.2cm]{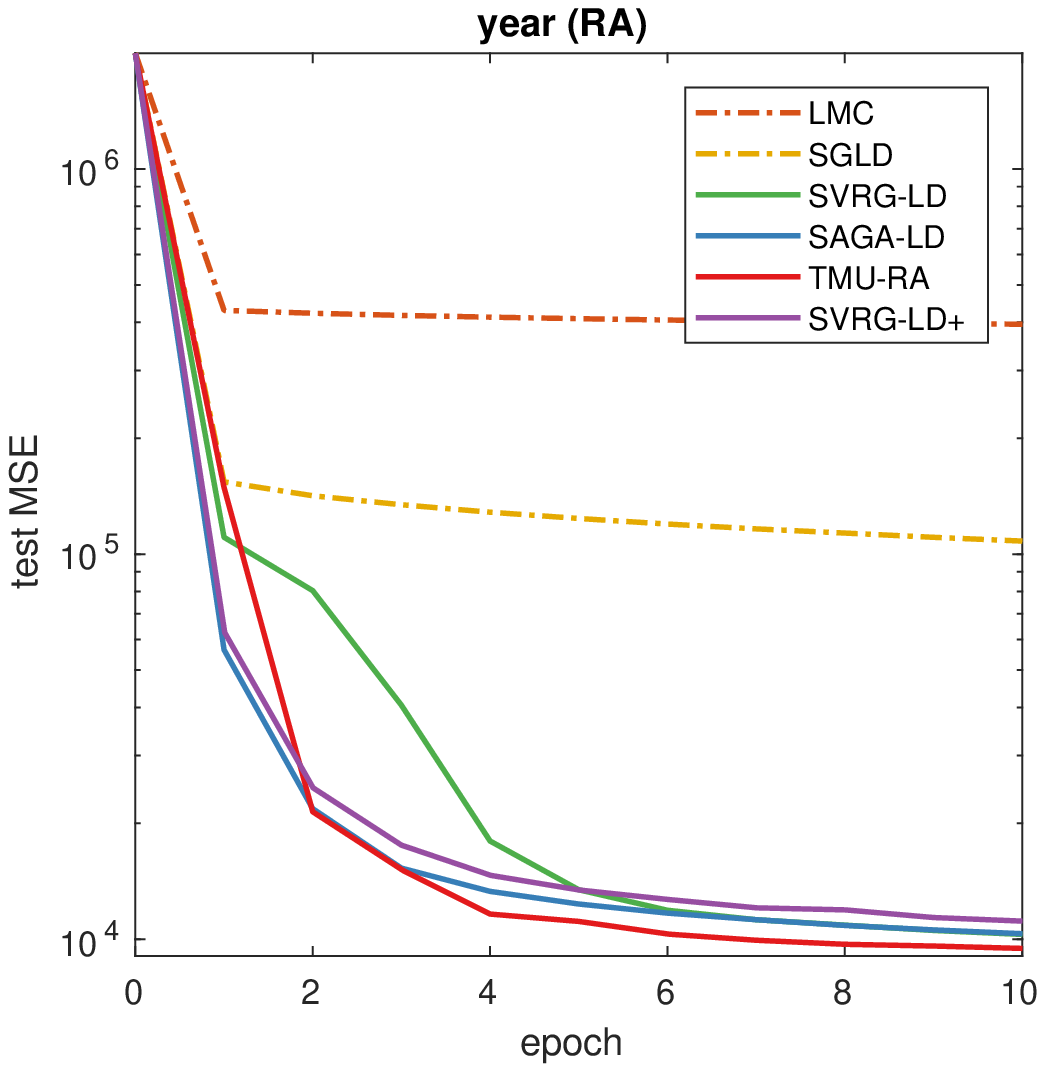}
		\end{subfigure}
		~
		%	\begin{subfigure}
		%		\centering
		%		\includegraphics[trim={0cm 0cm 0cm 0cm},clip,width=3.9cm]{img/RR/noise_epoch_RA}
		%	\end{subfigure}
		%	~
		%	\begin{subfigure}
		%		\centering
		%		\includegraphics[trim={0cm 0cm 0cm 0cm},clip,width=3.9cm]{img/RR/parkinsons_epoch_RA}
		%	\end{subfigure}
		%	~
		%		\begin{subfigure}
		%			\centering
		%			\includegraphics[trim={0cm 0cm 0cm 0cm},clip,width=3.7cm,height= 3.6cm]{img/RR/slice_location_epoch_RA}
		%		\end{subfigure}
		%		~
		\begin{subfigure}
			\centering
			\includegraphics[trim={0cm 0cm 0cm 0cm},clip,width=3.7cm,height= 3.2cm]{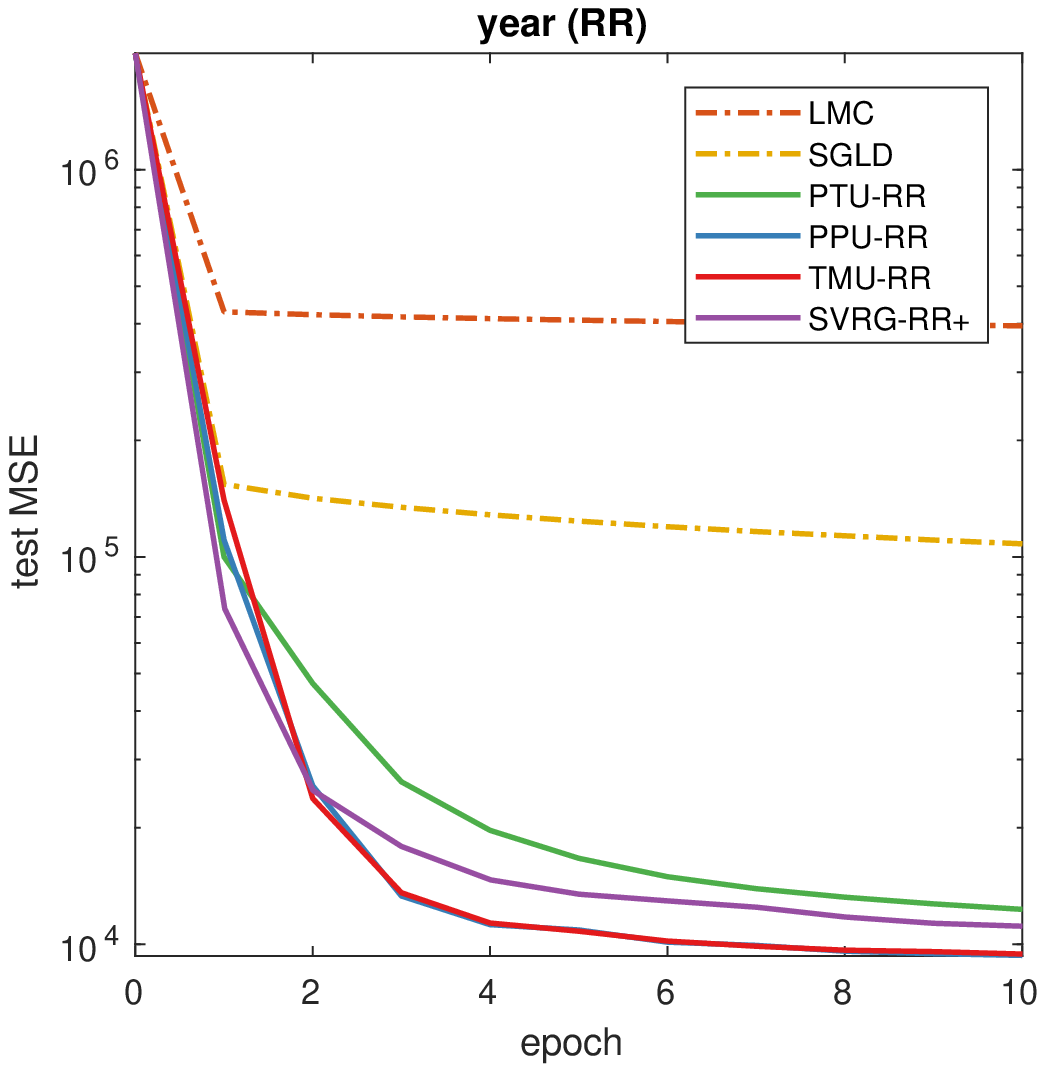}
		\end{subfigure}
		~
		%	\begin{subfigure}
		%		\centering
		%		\includegraphics[trim={0cm 0cm 0cm 0cm},clip,width=3.9cm]{img/RR/noise_epoch_RR}
		%	\end{subfigure}
		%	~
		%	\begin{subfigure}
		%		\centering
		%		\includegraphics[trim={0cm 0cm 0cm 0cm},clip,width=3.9cm]{img/RR/parkinsons_epoch_RR}
		%	\end{subfigure}
		%	~
		%		\begin{subfigure}
		%			\centering
		%			\includegraphics[trim={0cm 0cm 0cm 0cm},clip,width=3.7cm,height= 3.6cm]{img/RR/slice_location_epoch_RR}
		%		\end{subfigure}
		%		~
		\begin{subfigure}
			\centering
			\includegraphics[trim={0cm 0cm 0cm 0cm},clip,width=3.7cm,height= 3.2cm]{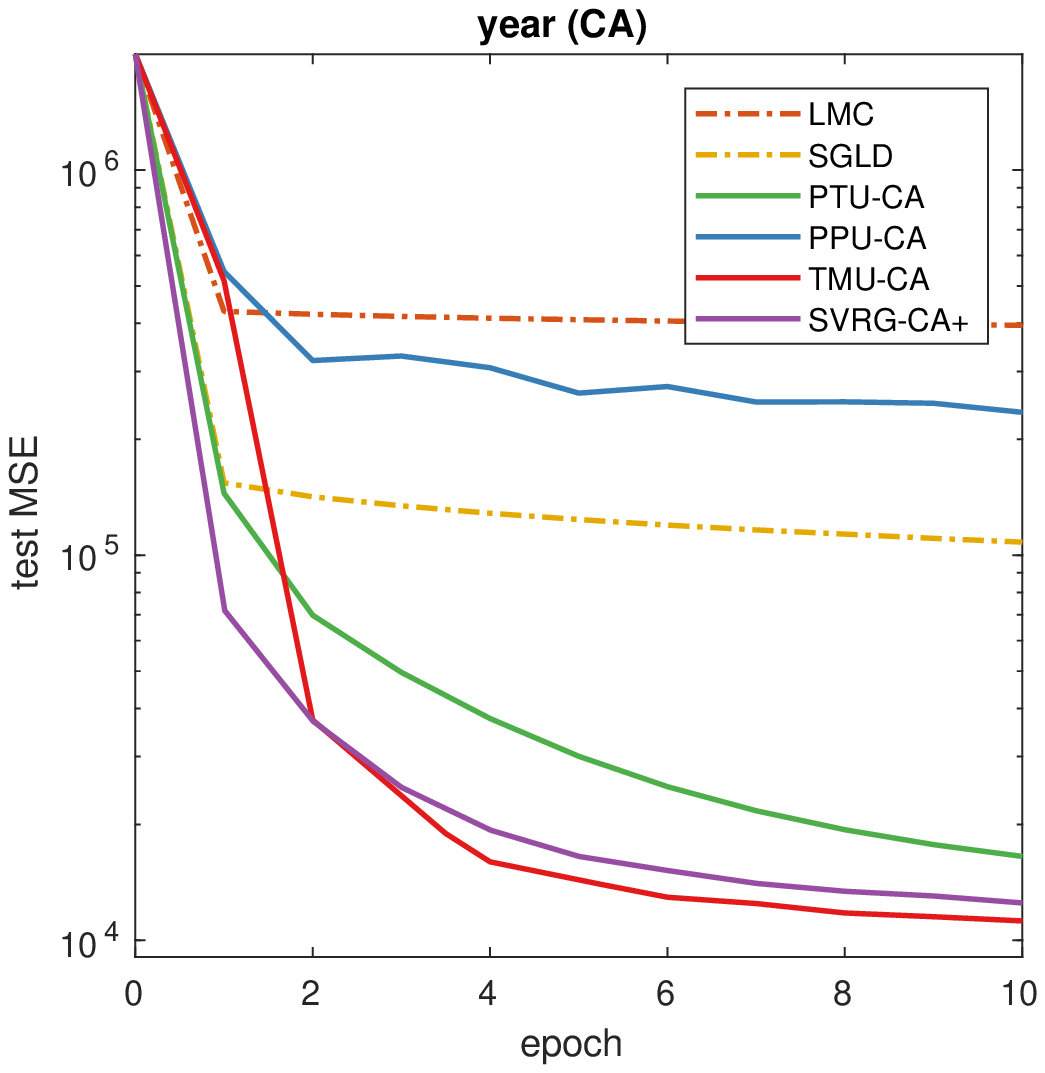}
		\end{subfigure}
		~
		%	\begin{subfigure}
		%		\centering
		%		\includegraphics[trim={0cm 0cm 0cm 0cm},clip,width=3.9cm]{img/RR/noise_epoch_CA}
		%	\end{subfigure}
		%	~
		%	\begin{subfigure}
		%		\centering
		%		\includegraphics[trim={0cm 0cm 0cm 0cm},clip,width=3.9cm]{img/RR/parkinsons_epoch_CA}
		%	\end{subfigure}
		%	~
		%		\begin{subfigure}
		%			\centering
		%			\includegraphics[trim={0cm 0cm 0cm 0cm},clip,width=3.7cm,height= 3.6cm]{img/RR/slice_location_epoch_CA}
		%		\end{subfigure}
		%		~
		\begin{subfigure}
			\centering
			\includegraphics[trim={0cm 0cm 0cm 0cm},clip,width=3.7cm,height= 3.2cm]{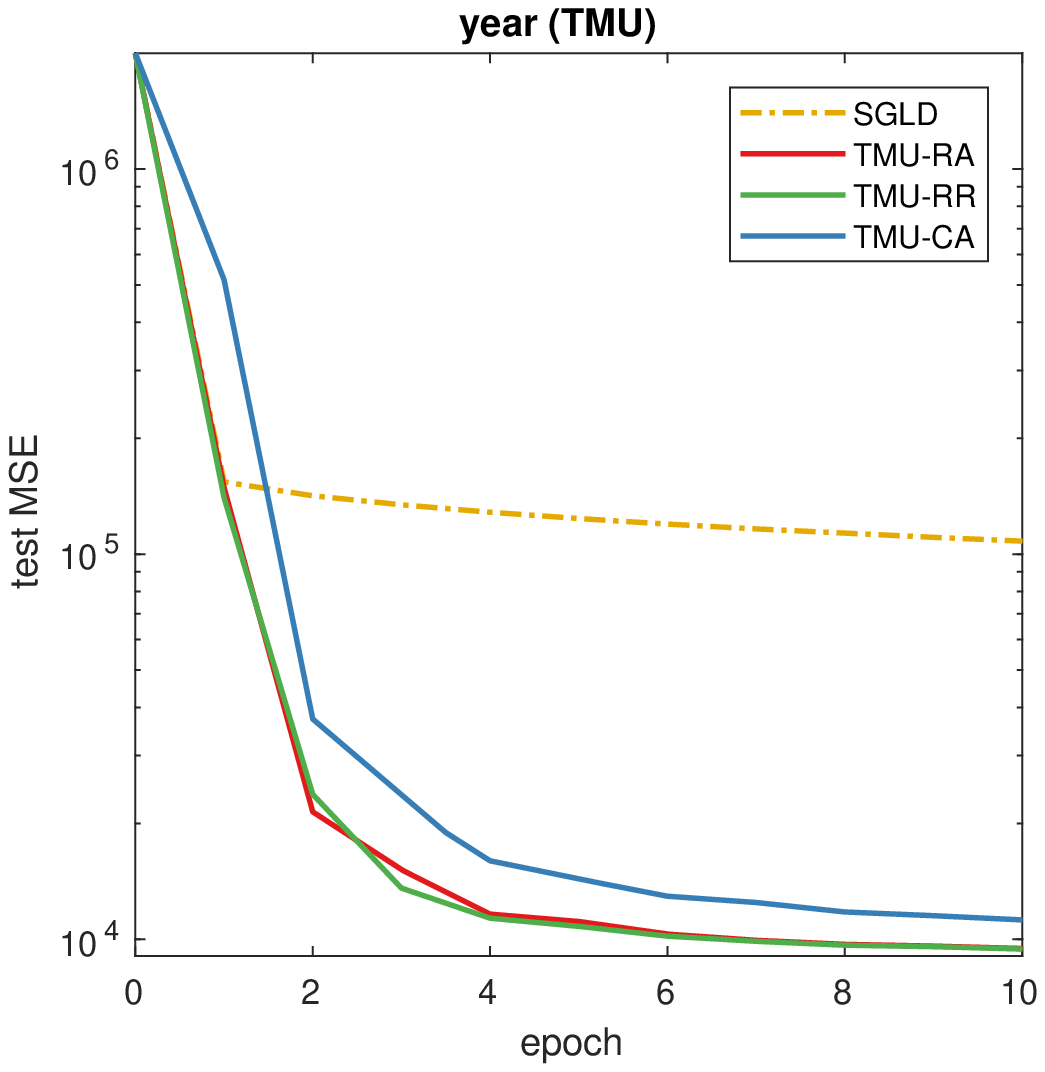}
		\end{subfigure}
		%	~
		%	\begin{subfigure}
		%		\centering
		%		\includegraphics[trim={0cm 0cm 0cm 0cm},clip,width=3.9cm]{img/RR/noise_epoch_TMU}
		%	\end{subfigure}
		%	~
		%	\begin{subfigure}
		%		\centering
		%		\includegraphics[trim={0cm 0cm 0cm 0cm},clip,width=3.9cm]{img/RR/parkinsons_epoch_TMU}
		%	\end{subfigure}
		%		~
		%		\begin{subfigure}
		%			\centering
		%			\includegraphics[trim={0cm 0cm 0cm 0cm},clip,width=3.7cm,height= 3.6cm]{img/RR/slice_location_epoch_TMU}
		%		\end{subfigure}
		\caption{Bayesian Ridge Regression.}
		%	The first three rows shows the results for difference methods with RA, RR and CA, respectively.In the last row, we compare the performance of TMU with different Data-Accessing strategies.}
		\label{fig:exp2}
		\vskip -0.2in
	\end{figure*}

	\begin{figure}[t]
		% \vskip 0.2in
		\centering
		\begin{subfigure}
			\centering
			\includegraphics[trim={0cm 0cm 0cm 0cm},clip,width=3.7cm,height= 3cm]{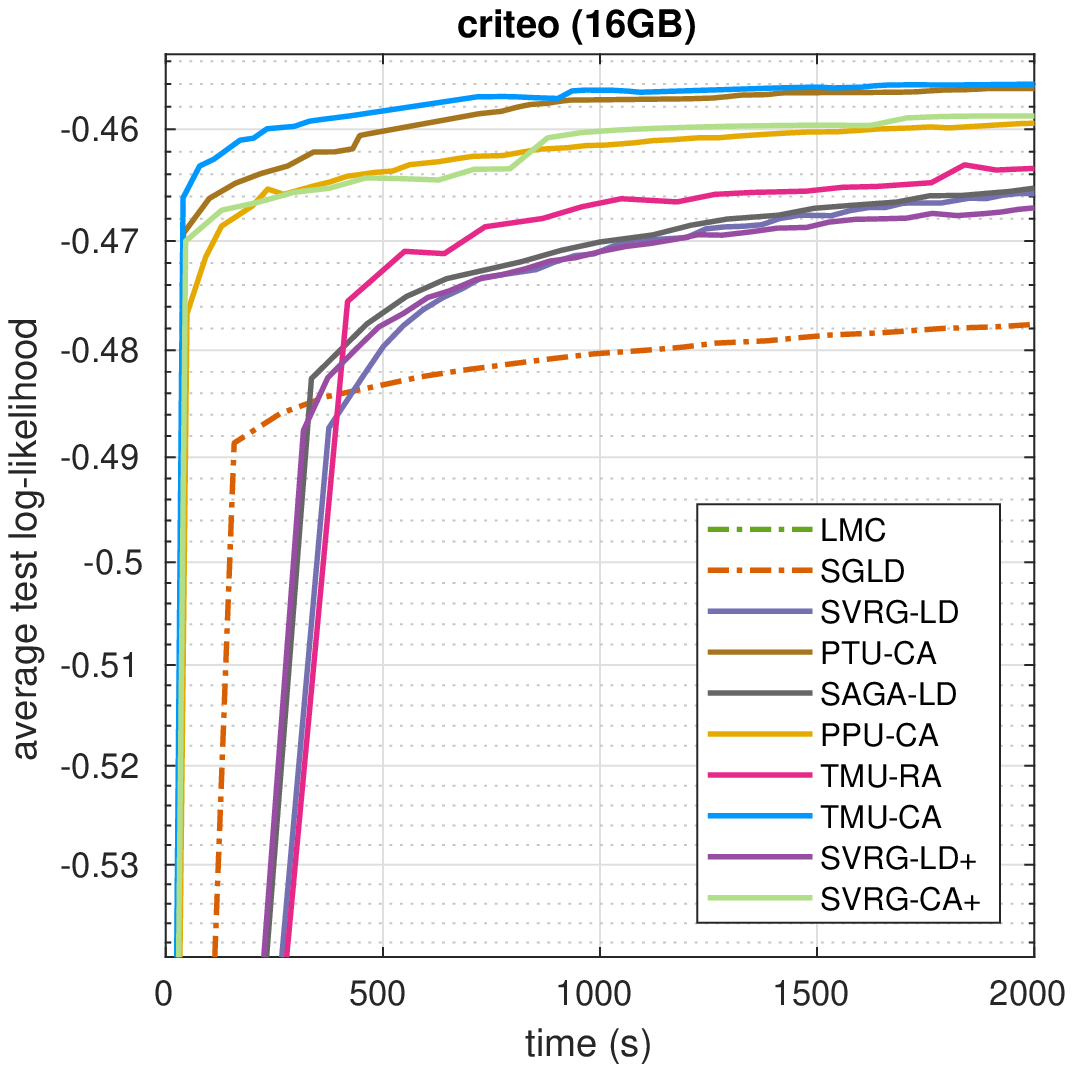}
		\end{subfigure}
		~
		\begin{subfigure}
			\centering
			\includegraphics[trim={0cm 0 0cm 0cm},clip,width=3.7cm,height= 3cm]{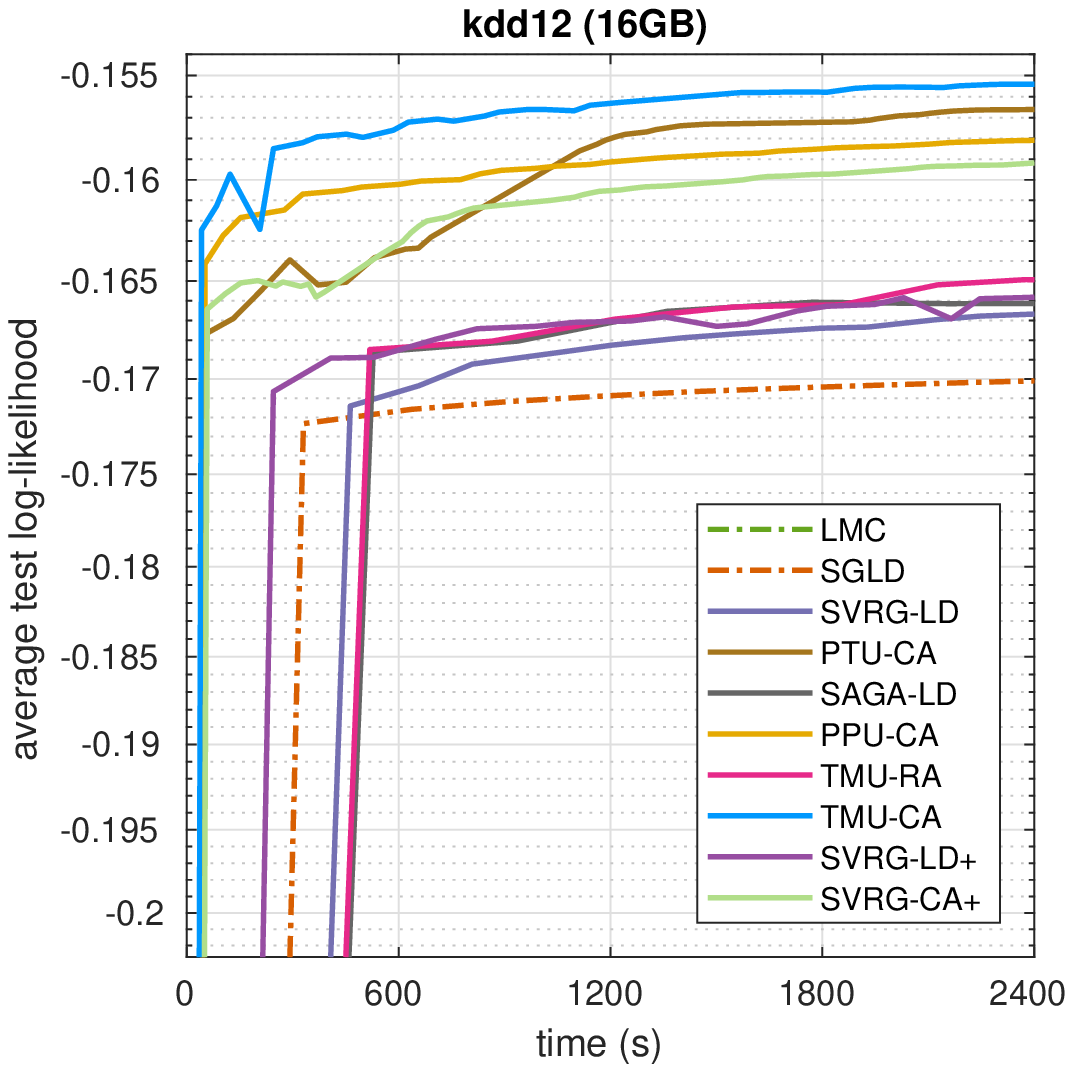}
		\end{subfigure}
		~
		\begin{subfigure}
			\centering
			\includegraphics[trim={0cm 0 0cm 0cm},clip,width=3.7cm,height= 3cm]{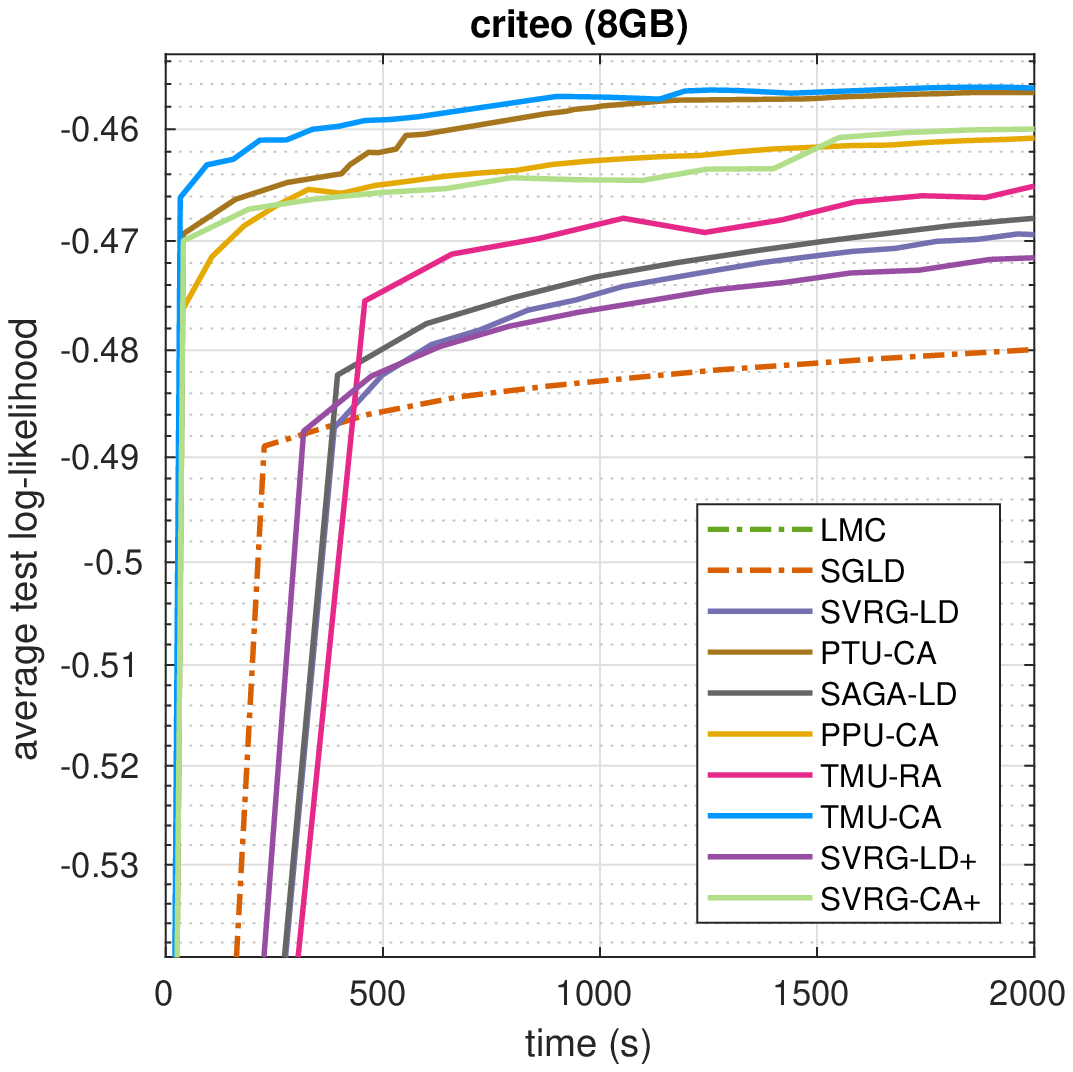}
		\end{subfigure}
		~
		\begin{subfigure}
			\centering
			\includegraphics[trim={0cm 0 0cm 0cm},clip,width=3.7cm,height= 3cm]{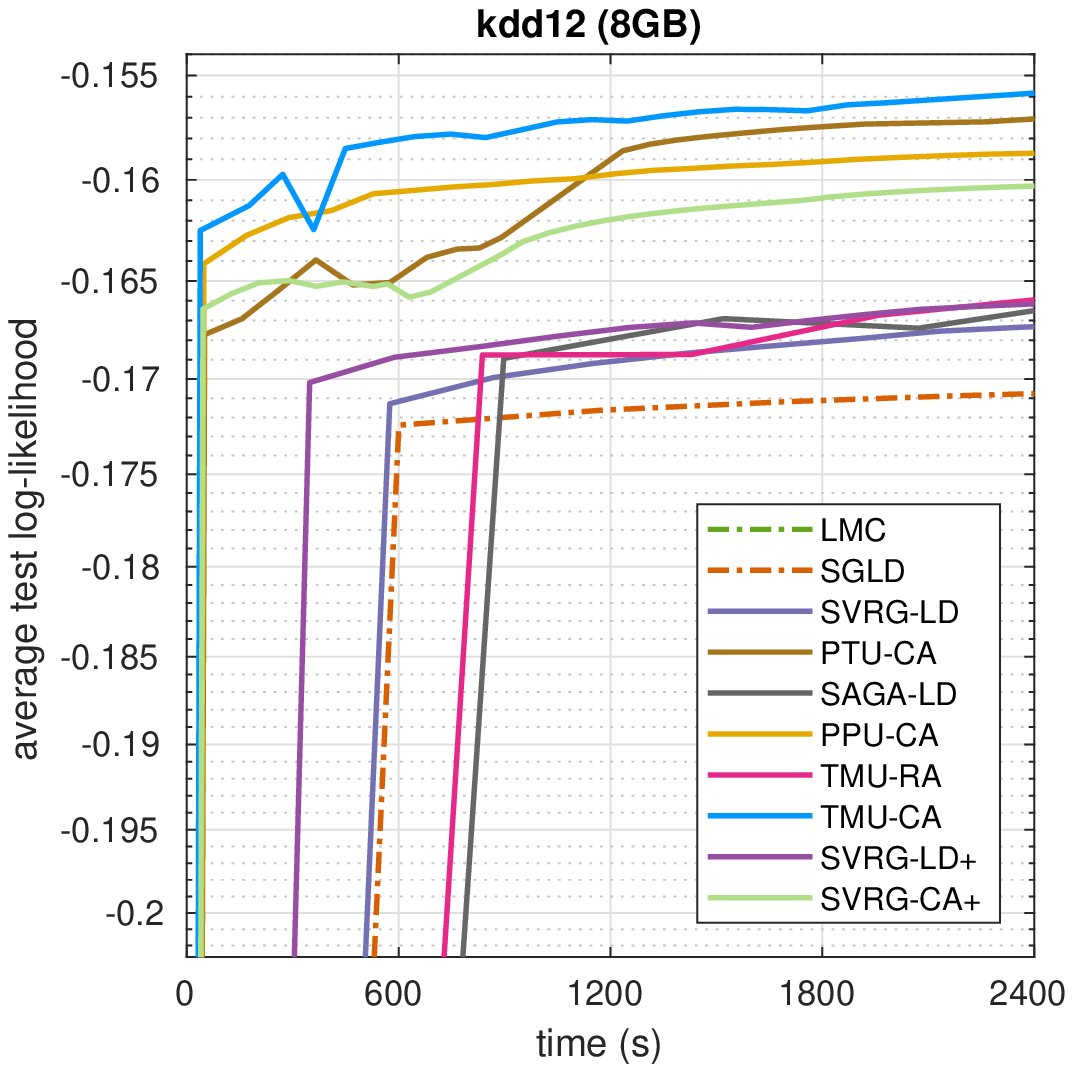}
		\end{subfigure}
		\caption{Bayesian Logistic Regression.}
		%	The first row shows the results of 16GB memory and the second row shows the results of 8GB memory.}
		\label{fig:exp3}
		\vskip -0.25in
	\end{figure}

	In this simulated experiment, we consider sampling from distribution
	$p^* \propto \exp(-f (x)) = \exp( - \sum_{i=1}^N f_i (x)/N) $, where each
	component $\exp(-f_i(x))$ is defined as
	$\exp(-f_i(x)) = e^{-\|x-a_i\|^2_2/2} + e^{-\|x+a_i\|2/2}, a_i \in \RBB^d$.
	It can be verified that $\exp(-f_i(x))$ is proportional to the PDF of a Gaussian mixture distribution.
	According to \cite{dalalyan2017theoretical}, when the parameter $a_i$ is chosen such that $\|a_i\|^2 \ge 1$, $f_i(x)$ is nonconvex.
	We set the sample size $N = 500$ and dimension $d = 10$, and randomly generate parameters $a_i \sim  \NB (\mu, \Sigma)$ with $\mu = (2,\cdots,2)^T$ and $\Sigma = \IB_{d\times d}$.

	In this experiment, we fix the Data-Accessing strategy to RA in AGLD and compare the performance of LMC, SGLD, SVRG-LD, SAGA-LD and TMU-RA algorithms.
	We run all algorithms for $2 \times 10^4$ data passes, and make use of the iterates in the last $10^4$ data passes to visualize distributions.

	In Figure~\ref{fig:expGM}, we report the 2D projection of the densities of random samples generated by each algorithm.
	%Here, we the ground truth by running MCMC with MH correction, which produces samples with the target distribution.
	It can be observed that all three AGLD methods, i.e., SVRG-LD, SAGA-LD and TMU-RA, can well approximate the target distribution in $2\times 10^4$ data passes, while the distributions generated by LMC and SGLD have obvious deviation from the true one.
	Moreover, the results show that the sample probability of TMU-RA approximates the target distribution best among the three AGLD methods.
	 % of TMU-RA

	\subsection{Bayesian Ridge Regression}

	Bayesian ridge regression aims to predict the response $y$ according to the covariate $x$, given the dataset $\bold{Z} = \{ x_i, y_i\}_{i=1}^{N}$.
	The response $y$ is modeled as a random variable sampled from a conditional Gaussian distribution $p(y|x, w) = \NB(w^T x, \lambda)$, where $w$ denotes the weight variable and has a Gaussian prior $p(w) = \NB(0, \lambda \IB_{d\times d})$.
	By the Bayesian rule, one can infer $w$ from the posterior $p(w | \bold{Z})$ and use it to make the prediction.
	Two publicly available benchmark datasets are used for evaluation: YearPredictionMSD and SliceLocation\footnote{\url{https://archive.ics.uci.edu/ml/index.php}}.
	All these datasets are small and can be loaded to the memory.

	In this task, we fix the Data-Accessing strategy to RA and compare the performance of different Snapshot-Updating strategies.
	To have a better understanding of the newly-proposed TMU Snapshot-Updating strategy, we also investigate the performance of TMU type methods with different Data-Accessing strategies.
	%Due to the limit of space, we list the results with random access here and defer the ones with random reshuffle and cyclic access to the Appendix.

	By randomly partitioning the dataset into training ($4/5$) and testing ($1/5$) sets, we report the test Mean Square Error (MSE) of the compared methods on  YearPredictionMSD in Fig.~\ref{fig:exp2}.
	The results for SliceLocation are similar to that of YearPredictionMSD,
	and are postponed to the Appendix due to the limit of space.
	We use the number of effective passes (epoch) of the dataset as the x-axis, which is proportional to the CPU time.
	From the first three columns of the figure, we can see that (\rNum{1}) TMU-type methods have the best performance among all the methods with the same Data-Accessing strategy,
	(\rNum{2}) SVRG+ and PPU type methods constantly outperform LMC, SGLD, and PTU type methods.
	These results validate the advantage of TMU strategy over PPU and PTU.
	The last column of Figure~\ref{fig:exp2} shows that TMU-RA outperforms TMU-CA/TMU-RR, when the dataset is fitted to the memory.
	These results imply that the TMU-RA is the best choice if we have enough memory.

	\subsection{Bayesian Logistic Regression}
	Bayesian Logistic Regression is a robust binary classification task.
	Let $\bold{Z} = \{ x_i, y_i\}_{i=1}^{N}$ be a dataset with $y_i \in \{-1, 1\}$ denoting the sample label and $x_i\in \RBB^d$ denoting the sample covariate vector.
	The conditional distribution of label $y$ is modeled by $p(y|x, w) = \phi(y_i w^T x_i )$, where $\phi(\cdot)$ is the sigmoid function and the prior of $w$ is $p(w) = \NB(0, \lambda \IB_{d\times d})$.
	%By the Bayesian inference framework, one can predict the label of a new coming data point in a Monte Carlo way.

	We focus on the big data setting, where the physical memory is insufficient to load the entire dataset.
	Specifically, two large-scale datasets criteo (27.32GB) and kdd12 (26.76GB) are used \textbf{\footnote{\url{https://www.csie.ntu.edu.tw/~cjlin/libsvmtools/datasets}}} and we manually restrict the available physical memory to 16 GB and 8 GB for simulation.

	%Both datasets are randomly partitioned into training set (80\%) and testing set (20\%).

	We demonstrate that CA strategy is advantageous in such setting by comparing $6$ AGLD methods with either CA or RA in the experiment, namely, SVRG-LD, PTU-CA, SAGA-LD, PPU-CA, TMU-RA, and TMU-CA.
	We also include LMC, SGLD, SVRG-LD+, and SVRG-CA+ as baseline.
	%LMC, SGLD, and SVRG-LD+  are included as baseline.
	Methods with the RR strategy have almost identical performance as their RA counterparts and are hence omitted.
	The average test log-likelihood versus execution time are reported in Fig.~\ref{fig:exp3}.
	The empirical results show that methods with CA outperform their RA counterparts.
	As the amount of physical memory gets smaller (from 16 GB to 8GB), the time efficiency of CA becomes more apparent.
	The results also show that TMU has better performance than other Snapshot-Updating strategies with the same Data-Accessing strategy.

	\section{Conclusion and Future Work}
	In this paper, we proposed a general framework called Aggregated Gradient Langevin Dynamics (AGLD) for Bayesian posterior sampling.
	% consisting of three main procedures: Data-Accessing, Sample-Searching, and Snapshot-Updating.
	%By coupling different Data-Access and Snapshot-Updating strategies, we derive different methods.
	%A unified analysis for AGLD is provided without the need to design different Lyapunov functions for different methods with different Data-Accessing strategies and Snapshot-Updating strategies individually.
	A unified analysis for AGLD is provided without the need to design different Lyapunov functions for different methods individually.
	In particular, we establish the first theoretical guarantees for cyclic access and random reshuffle based methods.
	By introducing the new Snapshot-Updating strategy TMU, we derive some new methods under AGLD.
	Empirical results validate the efficiency and effectiveness of the proposed TMU in both simulated and real-world tasks.
	The theoretical analysis and empirical results indicate that TMU-RA would be the best choice if the memory is sufficient and TMU-CA would be used, otherwise.

	%our proposed methods over the state-of-art.

	%In the future, we may introduce more different Data-Accessing and Snapshot-Updating strategies into AGLD.
	%Besides, we may also extend AGLD to other diffusion based MCMC method such as the Hamilton Monte Carlo  {\citep{chen2014stochastic}} and the Richardson-Romberg Markov Chain Monte Carlo  {\citep{durmus2016stochastic}}.
	%\newpage
	%\small
	\bibliographystyle{named}
	\bibliography{igld}

\begin{thebibliography}{}

\bibitem[\protect\citeauthoryear{Baker \bgroup \em et al.\egroup
  }{2017}]{baker2017control}
Jack Baker, Paul Fearnhead, Emily~B Fox, and Christopher Nemeth.
\newblock Control variates for stochastic gradient mcmc.
\newblock {\em arXiv preprint arXiv:1706.05439}, 2017.

\bibitem[\protect\citeauthoryear{Betancourt}{2015}]{betancourt2015fundamental}
Michael Betancourt.
\newblock The fundamental incompatibility of scalable hamiltonian monte carlo
  and naive data subsampling.
\newblock In {\em International Conference on Machine Learning}, pages
  533--540, 2015.

\bibitem[\protect\citeauthoryear{Bierkens \bgroup \em et al.\egroup
  }{2016}]{bierkens2016zig}
Joris Bierkens, Paul Fearnhead, and Gareth Roberts.
\newblock The zig-zag process and super-efficient sampling for bayesian
  analysis of big data.
\newblock {\em arXiv preprint arXiv:1607.03188}, 2016.

\bibitem[\protect\citeauthoryear{Brosse \bgroup \em et al.\egroup
  }{2018}]{brosse2018promises}
Nicolas Brosse, Alain Durmus, and Eric Moulines.
\newblock The promises and pitfalls of stochastic gradient langevin dynamics.
\newblock In {\em Advances in Neural Information Processing Systems}, pages
  8268--8278, 2018.

\bibitem[\protect\citeauthoryear{Chatterji \bgroup \em et al.\egroup
  }{2018}]{chatterji2018theory}
Niladri~S Chatterji, Nicolas Flammarion, Yi-An Ma, Peter~L Bartlett, and
  Michael~I Jordan.
\newblock On the theory of variance reduction for stochastic gradient monte
  carlo.
\newblock {\em arXiv preprint arXiv:1802.05431}, 2018.

\bibitem[\protect\citeauthoryear{Cheng and
  Bartlett}{2018}]{cheng2018convergence}
Xiang Cheng and Peter~L Bartlett.
\newblock Convergence of langevin mcmc in kl-divergence.
\newblock {\em PMLR 83}, (83):186--211, 2018.

\bibitem[\protect\citeauthoryear{Dalalyan and
  Karagulyan}{2017}]{dalalyan2017user}
Arnak~S Dalalyan and Avetik~G Karagulyan.
\newblock User-friendly guarantees for the langevin monte carlo with inaccurate
  gradient.
\newblock {\em arXiv preprint arXiv:1710.00095}, 2017.

\bibitem[\protect\citeauthoryear{Dalalyan}{2017a}]{dalalyan2017further}
Arnak~S Dalalyan.
\newblock Further and stronger analogy between sampling and optimization:
  Langevin monte carlo and gradient descent.
\newblock {\em arXiv preprint arXiv:1704.04752}, 2017.

\bibitem[\protect\citeauthoryear{Dalalyan}{2017b}]{dalalyan2017theoretical}
Arnak~S Dalalyan.
\newblock Theoretical guarantees for approximate sampling from smooth and
  log-concave densities.
\newblock {\em Journal of the Royal Statistical Society: Series B (Statistical
  Methodology)}, 79(3):651--676, 2017.

\bibitem[\protect\citeauthoryear{Dang \bgroup \em et al.\egroup
  }{2019}]{dang2019hamiltonian}
Khue-Dung Dang, Matias Quiroz, Robert Kohn, Minh-Ngoc Tran, and Mattias
  Villani.
\newblock Hamiltonian monte carlo with energy conserving subsampling.
\newblock {\em Journal of machine learning research}, 20(100):1--31, 2019.

\bibitem[\protect\citeauthoryear{Dawkins}{1991}]{dawkins1991siobhan}
Brian Dawkins.
\newblock Siobhan's problem: the coupon collector revisited.
\newblock {\em The American Statistician}, 45(1):76--82, 1991.

\bibitem[\protect\citeauthoryear{Defazio \bgroup \em et al.\egroup
  }{2014}]{defazio2014saga}
Aaron Defazio, Francis Bach, and Simon Lacoste-Julien.
\newblock Saga: A fast incremental gradient method with support for
  non-strongly convex composite objectives.
\newblock In {\em Advances in neural information processing systems}, pages
  1646--1654, 2014.

\bibitem[\protect\citeauthoryear{Dubey \bgroup \em et al.\egroup
  }{2016}]{dubey2016variance}
Kumar~Avinava Dubey, Sashank~J Reddi, Sinead~A Williamson, Barnabas Poczos,
  Alexander~J Smola, and Eric~P Xing.
\newblock Variance reduction in stochastic gradient langevin dynamics.
\newblock In {\em Advances in Neural Information Processing Systems}, pages
  1154--1162, 2016.

\bibitem[\protect\citeauthoryear{Durmus and Moulines}{2016}]{durmus2016high}
Alain Durmus and Eric Moulines.
\newblock High-dimensional bayesian inference via the unadjusted langevin
  algorithm.
\newblock {\em arXiv preprint arXiv:1605.01559}, 2016.

\bibitem[\protect\citeauthoryear{Durmus \bgroup \em et al.\egroup
  }{2017}]{durmus2017nonasymptotic}
Alain Durmus, Eric Moulines, et~al.
\newblock Nonasymptotic convergence analysis for the unadjusted langevin
  algorithm.
\newblock {\em The Annals of Applied Probability}, 27(3):1551--1587, 2017.

\bibitem[\protect\citeauthoryear{Dwivedi \bgroup \em et al.\egroup
  }{2018}]{dwivedi2018log}
Raaz Dwivedi, Yuansi Chen, Martin~J Wainwright, and Bin Yu.
\newblock Log-concave sampling: Metropolis-hastings algorithms are fast!
\newblock {\em arXiv preprint arXiv:1801.02309}, 2018.

\bibitem[\protect\citeauthoryear{Eberle}{2014}]{eberle2014error}
Andreas Eberle.
\newblock Error bounds for metropolis--hastings algorithms applied to
  perturbations of gaussian measures in high dimensions.
\newblock {\em The Annals of Applied Probability}, 24(1):337--377, 2014.

\bibitem[\protect\citeauthoryear{Hastings}{1970}]{hastings1970monte}
W~Keith Hastings.
\newblock Monte carlo sampling methods using markov chains and their
  applications.
\newblock 1970.

\bibitem[\protect\citeauthoryear{Johnson and
  Zhang}{2013}]{johnson2013accelerating}
Rie Johnson and Tong Zhang.
\newblock Accelerating stochastic gradient descent using predictive variance
  reduction.
\newblock In {\em Advances in neural information processing systems}, pages
  315--323, 2013.

\bibitem[\protect\citeauthoryear{Lei and Jordan}{2017}]{lei2017less}
Lihua Lei and Michael Jordan.
\newblock Less than a single pass: Stochastically controlled stochastic
  gradient.
\newblock In {\em Artificial Intelligence and Statistics}, pages 148--156,
  2017.

\bibitem[\protect\citeauthoryear{Mattingly \bgroup \em et al.\egroup
  }{2002}]{mattingly2002ergodicity}
Jonathan~C Mattingly, Andrew~M Stuart, and Desmond~J Higham.
\newblock Ergodicity for sdes and approximations: locally lipschitz vector
  fields and degenerate noise.
\newblock {\em Stochastic processes and their applications}, 101(2):185--232,
  2002.

\bibitem[\protect\citeauthoryear{Nagapetyan \bgroup \em et al.\egroup
  }{2017}]{nagapetyan2017true}
Tigran Nagapetyan, Andrew~B Duncan, Leonard Hasenclever, Sebastian~J Vollmer,
  Lukasz Szpruch, and Konstantinos Zygalakis.
\newblock The true cost of stochastic gradient langevin dynamics.
\newblock {\em arXiv preprint arXiv:1706.02692}, 2017.

\bibitem[\protect\citeauthoryear{Parisi}{1981}]{parisi1981correlation}
G~Parisi.
\newblock Correlation functions and computer simulations.
\newblock {\em Nuclear Physics B}, 180(3):378--384, 1981.

\bibitem[\protect\citeauthoryear{Raginsky \bgroup \em et al.\egroup
  }{2017}]{raginsky2017non}
Maxim Raginsky, Alexander Rakhlin, and Matus Telgarsky.
\newblock Non-convex learning via stochastic gradient langevin dynamics: a
  nonasymptotic analysis.
\newblock {\em arXiv preprint arXiv:1702.03849}, 2017.

\bibitem[\protect\citeauthoryear{Reddi \bgroup \em et al.\egroup
  }{2015}]{reddi2015variance}
Sashank~J Reddi, Ahmed Hefny, Suvrit Sra, Barnabas Poczos, and Alexander~J
  Smola.
\newblock On variance reduction in stochastic gradient descent and its
  asynchronous variants.
\newblock In {\em Advances in Neural Information Processing Systems}, pages
  2647--2655, 2015.

\bibitem[\protect\citeauthoryear{Robbins and
  Monro}{1951}]{robbins1951stochastic}
Herbert Robbins and Sutton Monro.
\newblock A stochastic approximation method.
\newblock {\em The annals of mathematical statistics}, pages 400--407, 1951.

\bibitem[\protect\citeauthoryear{Roberts and
  Rosenthal}{1998}]{roberts1998optimal}
Gareth~O Roberts and Jeffrey~S Rosenthal.
\newblock Optimal scaling of discrete approximations to langevin diffusions.
\newblock {\em Journal of the Royal Statistical Society: Series B (Statistical
  Methodology)}, 60(1):255--268, 1998.

\bibitem[\protect\citeauthoryear{Roberts and
  Stramer}{2002}]{roberts2002langevin}
Gareth~O Roberts and Osnat Stramer.
\newblock Langevin diffusions and metropolis-hastings algorithms.
\newblock {\em Methodology and computing in applied probability},
  4(4):337--357, 2002.

\bibitem[\protect\citeauthoryear{Roberts \bgroup \em et al.\egroup
  }{1996}]{roberts1996exponential}
Gareth~O Roberts, Richard~L Tweedie, et~al.
\newblock Exponential convergence of langevin distributions and their discrete
  approximations.
\newblock {\em Bernoulli}, 2(4):341--363, 1996.

\bibitem[\protect\citeauthoryear{Shamir}{2016}]{shamir2016without}
Ohad Shamir.
\newblock Without-replacement sampling for stochastic gradient methods.
\newblock In {\em NeurIPS}, pages 46--54, 2016.

\bibitem[\protect\citeauthoryear{Teh \bgroup \em et al.\egroup
  }{2016}]{teh2016consistency}
Yee~Whye Teh, Alexandre~H Thiery, and Sebastian~J Vollmer.
\newblock Consistency and fluctuations for stochastic gradient langevin
  dynamics.
\newblock {\em Journal of Machine Learning Research}, 17:1--33, 2016.

\bibitem[\protect\citeauthoryear{Vollmer \bgroup \em et al.\egroup
  }{2015}]{vollmer2015non}
Sebastian~J Vollmer, Konstantinos~C Zygalakis, et~al.
\newblock (non-) asymptotic properties of stochastic gradient langevin
  dynamics.
\newblock {\em arXiv preprint arXiv:1501.00438}, 2015.

\bibitem[\protect\citeauthoryear{Welling and Teh}{2011}]{welling2011bayesian}
Max Welling and Yee~W Teh.
\newblock Bayesian learning via stochastic gradient langevin dynamics.
\newblock In {\em Proceedings of the 28th International Conference on Machine
  Learning (ICML-11)}, pages 681--688, 2011.

\bibitem[\protect\citeauthoryear{Xie \bgroup \em et al.\egroup
  }{2018}]{xie2018towards}
Jiahao Xie, Hui Qian, Zebang Shen, and Chao Zhang.
\newblock Towards memory-friendly deterministic incremental gradient method.
\newblock In {\em International Conference on Artificial Intelligence and
  Statistics}, pages 1147--1156, 2018.

\bibitem[\protect\citeauthoryear{Zou \bgroup \em et al.\egroup
  }{2018a}]{zou2018stochastic}
Difan Zou, Pan Xu, and Quanquan Gu.
\newblock Stochastic variance-reduced hamilton monte carlo methods.
\newblock {\em arXiv preprint arXiv:1802.04791}, 2018.

\bibitem[\protect\citeauthoryear{Zou \bgroup \em et al.\egroup
  }{2018b}]{zou2018subsampled}
Difan Zou, Pan Xu, and Quanquan Gu.
\newblock Subsampled stochastic variance-reduced gradient langevin dynamics.
\newblock {\em conference on uncertainty in artificial intelligence}, 2018.

\bibitem[\protect\citeauthoryear{Zou \bgroup \em et al.\egroup
  }{2019}]{zou2019sampling}
Difan Zou, Pan Xu, and Quanquan Gu.
\newblock Sampling from non-log-concave distributions via variance-reduced
  gradient langevin dynamics.
\newblock In {\em AISTATS}, pages 2936--2945, 2019.

\end{thebibliography}
	\onecolumn
	\section{Appendix}
	\subsection{Theoretical results under strongly convex and smooth assumption}
	
	We first list the requirement for AGLD and the assumptions on $f$.
	\begin{requirement}\label{assum5}
		For the gradient snapshot $\AM^{(k)}$, we have $\alpha_i^{(k)} \in \{\nabla f_i(x^{(j)})\}_{j= k-D+1}^{(k)}$, where $D$ is a fixed constant.
	\end{requirement}

	\begin{assumption}[Smoothness]
		Each individual $f_i$ is $\tilde M$ smooth.
		That is, $f_i$ is twice differentiable and there exists a constant $\tilde M > 0$ such that for all $x,y \in \RB^d$
		\begin{equation}\label{assum:smooth}
		f_i(y) \le f_i(x)+ \langle \nabla f_i(x), y-x \rangle + \frac{\tilde M}{2} \|x-y\|^2_2.
		\end{equation}
		Accordingly, we can verify that the summation $f$ of $f_i's$ is $M = \tilde M N$ smooth. 
	\end{assumption}
	
	\begin{assumption}[Strongly Convexity]
		The sum $f$ is $\mu$ strongly convex. 
		That is, there exists a constant $\mu > 0$ such that for all $x,y \in \RB^d$, 
		\begin{equation}\label{assum:sconvex}
		f(y) \ge f(x)+ \langle \nabla f(x), y-x \rangle + \frac{\mu}{2} \|x-y\|^2_2.
		\end{equation}
	\end{assumption}
	
	%\
	%\begin{lemma}
	%	Suppose one uniformly samples $n$ data points from dataset of size $N$ at one trial.
	%	Let $T$ denotes the  time to collect all the data points,
	%	then we have
	%	\begin{equation}
	%	P(T > \beta \frac{N \ln N}{n}) < N^{1-\beta}.
	%	\end{equation}
	%\end{lemma} 
	%\begin{proof}
	%	Let $Y_i^{(k)}$	denote the event that the $i$-th sample was not selected in the first $k$ trials.
	%	Then we have
	%	\begin{equation*}
	%	P \big (Y_i^{(k)}\big ) = (1-\frac{n}{N})^{(k)} \le e^{-(nr)/N}.
	%	\end{equation*}
	%	Thus we have $ P [Y_i^{(k)}] \le N^{-\beta}$ for $k = \beta N \log N /n$.
	%	Hence,
	%	\begin{eqnarray*}
	%		P \big (T > \beta \frac{N \ln N}{n} \big) &= &P \big( \bigcup Y_i^{(\beta N \log N /n)}\big )
	%		\\
	%		&\le& N  P [Y_i^{(\beta N \log N /n)}] \le N^{1-\beta}
	%	\end{eqnarray*}
	%\end{proof}

	According to  {\cite{dalalyan2017user}}, we have the following bound for the $\WM_2$ distance
	\begin{equation}
	\WM_2^2(p,q) \le \E\|x - y\|^2,
	\end{equation}
	for arbitrary $x \sim p, y \sim q$.
	Thus, in order to bound the $\WM_2$ distance of $p^{(k+1)}$ and $p^*$, we construct a auxiliary sequence $y^{(k)}$ with $y^{(k)} \sim p^*$ for all $k \ge 0$ and bound $\E\|x^{(k+1)} - y^{(k+1)}\|^2$.
	Specifically, we make $y^{(k)} = y(\eta k)$, where $y(t)$ is the following auxiliary dynamics
	\begin{equation}\label{eq:ydynamics}
	\dB y(t) = -\nabla f(y(t)) \dB t + \sqrt{2} \dB B(t)
	\end{equation} 
	with $y(0) \sim p^*$.
	According to the Fokker-Planck theorem  {\cite{risken1996fokker}}, $y(t) \sim p^*$ for all $t\ge 0$.
	It can be verified that $y^{(k+1)}$ satisfies the following iteration relation.
	\begin{equation}
	y^{(k+1)} = y^{(k)} - \int_{k \eta}^{(k+1)\eta} \nabla f(y(s))ds + \sqrt{2\eta} \xi^{(k)},
	\end{equation}
	where $y^{(0)} = y(0)$ and $\xi^{(k)}$  is the same Gaussian variable used in ($x^{(k+1)}$).
	%So we can establish a bound on $\WM_2(p^{(k)},p^*)$ by bounding $\E\|x^{(k)} - y^{(k)}\|^2$.	

	For the simplicity of notation, we denote $\Delta^{(k+1)} = y^{(k+1)} - x^{(k+1)}$ and decompose it as
	\begin{align}\label{delta}
	\Delta^{(k+1)} &= \Delta^{(k)} - V^{(k)} - \eta U^{(k)} + \eta \Phi^{(k)},
	\end{align}
	where 
	\begin{displaymath}
	\left\{\begin{array}{l}
	V^{(k)}=\int_{k\eta}^{(k+1)\eta}\nabla f(y(s)) - \nabla f(y^{(k)})ds,\\
	U^{(k)} = \nabla f(y^{(k)}) - \nabla f(x^{(k)}),\\
	\Phi^{(k)}  =g^{(k)} - \nabla f (x^{(k)}).\end{array}\right.
	\end{displaymath}
	Here, $\Phi^{(k)}$ is the difference between the gradient approximation $g^{(k)}$ and the full gradient  $\nabla f (x^{(k)})$.
	It can be further decomposed into an unbiased part $\Psi^{(k)}$ and the remainder $\Gamma^{(k)}$.
	If the Data-Accessing strategy is RA, then $\Psi^{(k)} = \Phi^{(k)}$ and $\Gamma^{(k)} =0$.
	For other strategies, we make 
	\begin{align}
	&\Psi^{(k)} = \sum_{i \in I_k}\frac{N}{n}(\nabla f_i(x^{(k)})-\alpha_{i}^{(k)})+ \sum_{i=1}^N \alpha^{(k)}_i- \nabla f (x^{(k)}),\label{eqn:psi}\\
	&\Gamma^{(k)} = \frac{N}{n}\sum_{i \in S_k}(\nabla f_i(x^{(k)})-\alpha_{i}^{(k)}) -\frac{N}{n} \sum_{i \in I_k}(\nabla f_i(x^{(k)})-\alpha_{i}^{(k)}),\label{eqn:gamma}
	\end{align}
	where $I_k$ is a set with $n$ indexes uniformly sampled from $\{1, \ldots, N\}$.
	It can be verified that $\E [\Psi^{(k)}|x^{(k)}] = 0$ in both setting.

	By bounding $\E \|\Delta^{(k)} - \eta U^{(k)}\|^2$, $\E \|V^{(k)}\|^2$, $\E \|\Psi^{(k)}\|$ and $\E \|\Gamma^{(k)}\|$ from above properly, we can establish 
	a per-iteration decreasing result of $\E\|\Delta^{(k)}\|$.
	
	\noindent{\bf Bound for $\E \|\Delta^{(k)} - \eta U^{(k)}\|^2$ and $\E \|V^{(k)}\|^2$:}
	By Lemma 1 and Lemma 3 in \cite{dalalyan2017further}, this two terms can be bounded from above as follows.
	%Here, we give a lemma from  {\cite{dalalyan2017user}} which bounds $\E\|V^{(k)}\|$ and $\E \|\Delta^{(k)} - \eta U^{(k)}\|^2$.
	
	\begin{lemma}[Lemma 1 \& 3 in \cite{dalalyan2017further}]\label{lemma:UV}
		Assuming that $f$ is $M$-smooth and $\mu$-strongly convex, and $\eta \le \frac{2}{M+\mu}$, we have
		\begin{align}
		&\E \|V^{(k)}\|^2 \le(\sqrt{\eta^4 M^3 d/3} + \sqrt{h^3M^2d })^2\le \frac{2}{3}\eta^4 M^3 d +2 h^3M^2d,\label{lemma:V}\\
		&\E \|\Delta^{(k)} - \eta U^{(k)}\|^2 \le (1- \eta \mu)^2 \E \|\Delta^{(k)}\|^2.  \label{lemma:U}
		\end{align}
		%where $M$ and $\mu$ is the smoothness and strongly convex constant of $f$, respectively.
	\end{lemma}

	\noindent{\bf Bound for $\E \|\Psi^{(k)}\|$ and $\E \|\Gamma^{(k)}\|$:}
	These two term can be bounded from above in the following lemma, the proof of which is postponed to the appendix.   
	\begin{lemma}\label{lemma:phi}
		Assuming that $f$ is $M$-smooth and $\mu$-strongly convex, and Requirement \ref{assum5} is satisfied, we have the following upper bound on $\E\|\Psi^{(k)}\|^2$ and $\E\|\Gamma^{(k)}\|^2$
		\begin{align*}
		&\E \|\Psi^{(k)}\|^2 \le \frac{N}{n} \sum_{i=1}^N \E \|\nabla f_i(x^{(k)}) - \alpha_i^{(k)}\|^2,\\
		&\E \|\Gamma^{(k)} \|^2 \le \frac{2N(N+n)}{n} \sum_{i=1}^N \E \|\nabla f_i(x^{(k)}) - \alpha_i^{(k)}\|^2,
		\end{align*}
		and
		\begin{align}\label{lemma:alpha}
		\sum_{i=1}^N \E \|\nabla f_i(x^{(k)}) - \alpha_i^{(k)}\|^2 \le  32\eta^2 D^2M^2\E\|\Delta^{(k:k-2D)}\|_{2,\infty}
		+4\eta Dd  +48\eta^3 M^2D^3d(\eta D M +1) + 8\eta^2 MND^2d,
		\end{align}
		where	
		$
		\Delta^{(k:k-2D)}:= [\Delta^{(k)}, \Delta^{(k-1)}, \cdots \Delta^{([k-2D]_+)}]$.
		If Data-Accessing strategy is {RA}, then $\E \|\Gamma^{(k)} \|^2 = 0$.	
	\end{lemma}
	\begin{proof}
		\begin{align*}
		\E \|\Psi^{(k)}\|^2 &= \E\| \sum_{i \in I_k}\frac{N}{n}(\nabla f_i(x^{(k)})-\alpha_{i}^{(k)})+ \sum_{i=1}^N \alpha^{(k)}_i- \nabla f (x^{(k)}\|^2\\
		& = \E\|\frac{1}{n}\sum_{i \in I_k}\big( N(\nabla f_i(x^{(k)})-\alpha_{i}^{(k)})- (\nabla f (x^{(k)})- \sum_{i=1}^N \alpha^{(k)}_i)\big)\|^2\\
		& = \frac{1}{n}\E\|N(\nabla f_i(x^{(k)})-\alpha_{i}^{(k)})- (\nabla f (x^{(k)})- \sum_{i=1}^N \alpha^{(k)}_i\|^2\\
		& \le \frac{N^2}{n}\E\|\nabla f_i(x^{(k)})-\alpha_{i}^{(k)}\|^2
		= \frac{N}{n}\sum_{i=1}^N\E\|\nabla f_i(x^{(k)})-\alpha_{i}^{(k)}\|^2.
		\end{align*}
		The third equality follows from the fact that $I_k$ are chosen uniformly and independently.
		The first inequality is due to the fact that $\E\|X - \E X\|^2 \le \E \|X\|^2$ for any random variable $X$.
		Here in the last equality,we use that $i$ is chosen uniformly from $\{1,\cdots,N\}$ and $i$ here is no longer a random variable.
		
		\begin{align*}
		\E\|E^{(k)}\|^2 &= \E \|\frac{N}{n}\sum_{i \in S_k}(\nabla f_i(x^{(k)})-\alpha_{i}^{(k)}) - \frac{N}{n}\sum_{i \in I_k}(\nabla f_i(x^{(k)})-\alpha_{i}^{(k)})\|^2
		\\
		&\le \frac{2N^2}{n^2}\E (\|\sum_{i \in S_k}(\nabla f_i(x^{(k)})-\alpha_{i}^{(k)})\|^2 + \|\sum_{i \in I_k}(\nabla f_i(x^{(k)})-\alpha_{i}^{(k)})\|^2)
		\\
		&\le \frac{2N^2}{n^2}\E (n\sum_{i \in S_k}\|\nabla f_i(x^{(k)})-\alpha_{i}^{(k)}\|^2 + n\sum_{i \in I_k}\|\nabla f_i(x^{(k)})-\alpha_{i}^{(k)}\|^2)
		\\
		&\le \frac{2N^2}{n^2}\E (n\sum_{i =1}^N\|\nabla f_i(x^{(k)})-\alpha_{i}^{(k)}\|^2 + \frac{n^2}{N}\sum_{i =1}^{N}\|\nabla f_i(x^{(k)})-\alpha_{i}^{(k)}\|^2)
		\\
		&= \frac{2N(N+n)}{n} \sum_{i=1}^N \E \|\nabla f_i(x^{(k)}) - \alpha_i^{(k)}\|^2.
		\end{align*}
		In the first two inequality, we use that $\|\sum_{i=1}^n a_i\|^2 \le n \sum_{i=1}^n \|a_i\|^2$.
		The third inequality follows from the fact that $\{S_k\}$ are subset of $\{1,\cdots,N\}$ in \textbf{CA} and \textbf{RS} and $I_k$ are chosen uniformly and independently from $\{1,\cdots,N\}$.
		When we use \textbf{RA}, $S_k$ just equals to $I_k$ and $\E\|E^{(k)}\|^0 = 0$.
		
		Suppose that in the $k$-th iteration, snapshot $\alpha_i^{(k)}$ are taken at $x^{(k_i)}$, where $k_i \in \{(k-1)\lor 0,(k-2)\lor 0,\cdots,(k-D)\lor 0\}$.
		By the $\tilde M$ smoothness of $f_i$, we have
		\begin{align*}
		\sum_{i=1}^N\E \|\nabla f_i(x^{(k)}) - \alpha_i^{(k)}\|^2 \le \sum_{i=1}^N {\tilde M}^2 \E \|x^{(k)} - x^{(k_i)}\|^2=\sum_{i=1}^N \frac{ M^2}{N^2} \E \|x^{(k)} - x^{(k_i)}\|^2.
		\end{align*}
		
		According to the update rule of $x^{(k)}$, we have 
		\begin{align*}
		\E \|x^{(k)} - x^{(k_i)}\|^2 &= \E\|-\eta \sum_{j = k_i}^{k-1} g^{(j)} + \sqrt{2}\sum_{j = k_i}^{k-1} \xi^{(j)} \|^2\le 2 D\eta^2\sum_{j = k-D}^{k-1} \E \| g^{(j)}\|^2 + 4 D d \eta,
		\end{align*}
		where the in equality follows from $\|a+b\|^2 \le 2(\|a\|^2+\|b\|^2)$, $\xi^{(j)}$ are independent Gaussian variables and $k_i \ge k-D$.
		
		By expanding $g^{(j)}$, we have 
		\begin{align*}
		\E \|g^{(j)}\|^2 & = \E \|\sum_{p \in S_j}\frac{N}{n}(\nabla f_p(x^{(j)})-\alpha_{p}^{(j)})+ \sum_{p=1}^N \alpha^{(j)}_p \|^2
		\\
		&\le   \underbrace{2 \E \|\sum_{p \in S_j}\frac{N}{n}(\nabla f_p(x^{(j)})-\alpha_{p}^{(j)})\|^2}_A+ \underbrace{2\E\|\sum_{p=1}^N \alpha^{(j)}_p \|^2}_B.
		\end{align*}
		For $A$, we have
		\begin{align*}
		A& \le 2n\sum_{p \in S_j}\frac{N^2}{n^2} \E \|(\nabla f_p(x^{(j)})-\nabla f_p(y^{(j)}))+ (\nabla 
		f_p(y^{(j)}) - \nabla f_p(y^{(j_p)}))+(\nabla f_p(y^{(j_p)})-\alpha_{p}^{(j)})\|^2
		\\
		&\le \frac{6N^2}{n} \sum_{p \in S_j} (\E \|\nabla f_p(x^{(j)})-\nabla f_p(y^{(j)})\|^2+ E\|\nabla 
		f_p(y^{(j)}) - \nabla f_p(y^{(j_p)})\|^2+\E\|\nabla f_p(y^{(j_p)})-\alpha_{p}^{(j)}\|^2)
		\\
		&\le \frac{6M^2}{n} \sum_{p \in S_j} (\E \|x^{(j)})-y^{(j)}\|^2+ E\|y^{(j)}) - y^{(j_p)}\|^2+\E\| y^{(j_p)})-x^{(j_p)}\|^2),
		\end{align*}
		where the last inequality follows from the smoothness of $f_p$.
		
		Bu further expanding $y^{(j)}$ and $y^{(j_p)}$, we have
		\begin{align}\label{yinter}
		&\E \|y^{(j)} - y^{(j_p)}\|^2 = \E \|\int_{j_p \eta}^{j \eta} \nabla f(y(s)) ds - \sqrt{2}\sum_{q = j_p}^{j} \xi^{(q)}\|^2\nonumber\\
		& \le 2(j - j_p)\eta\int_{j_p \eta}^{j \eta} \E \|\nabla f(y(s)) \|^2ds +4\eta Dd
		\le 2D\eta\cdot D\eta Md + 4 \eta Dd \le 2D^2 \eta^2 Md + 4\eta Dd,
		\end{align}
		where, the first inequality is due to the Jensen's inequality and the second inequality follows by Lemma 3 in \cite{dalalyan2017user} to bound $\E \|\nabla f(y(s)) \|^2ds \le Md $.
		
		Then we can  bound $A$ above by
		\begin{align*}
		A &\le \frac{6M^2}{n} \sum_{p \in S_j} (\E \|\Delta^j\|^2 + 2D^2 \eta^2 Md + 4\eta Dd + \E \|\Delta^{j_p}\|^2)
		\\
		&\le 6M^2  (\E \|\Delta^j\|^2 + 2D^2 \eta^2 Md + 4\eta Dd + \E \|\Delta\|^D_{j}).
		\end{align*}
		
		Now we can bound $B$ with similar technique
		\begin{align*}
		B &= 2\E\|\sum_{p=1}^N (\alpha^{(j)}_p - \nabla f_p(y^{(j_p)})) + \sum_{p=1}^N \nabla f_p(y^{(j_p)})\|^2
		\\
		&\le  4N \sum_{p=1}^N \E \|\nabla f_p(x^{(j_p)}) - \nabla f_p(y^{(j_p)})\|^2 + 4N \sum_{p=1}^N \E\|\nabla f_p(y^{(j_p)})\|^2
		\\
		&\le \frac{4M^2}{N} \sum_{p=1}^N \E\|\Delta^{j_p}\|^2 + 4 NMD
		\le 4M^2 \E \|\Delta\|_j^D + 4NDM.
		\end{align*}
		
		By substituting all these back, then we have
		\begin{align*}
		&\sum_{i=1}^N\E \|\nabla f_i(x^{(k)}) - \alpha_i^{(k)}\|^2 \le \sum_{i=1}^N \frac{ M^2}{N^2} \E \|x^{(k)} - x^{(k_i)}\|^2 \le  \frac{ M^2}{N^2}\sum_{i=1}^N (2 D\eta^2\sum_{j = k-D}^{k-1} \E \| g^{(j)}\|^2 + 4 D d \eta)	
		\\
		\le&  \frac{ M^2}{N^2}\sum_{i=1}^N \big(2 D\eta^2\sum_{j = k-D}^{k-1} \big(6M^2  (\E \|\Delta^j\|^2 + 2D^2 \eta^2 Md + 4\eta Dd + \E \|\Delta^{(j:j-D)}\|_{2,\infty}) + 4M^2 \E \|\Delta\|_j^D + 4NDM\big) + 4 D d \eta\big)
		\\
		\le& \frac{M^2}{N}(4 Dd \eta + 2D^2\eta^2(16\E\|\Delta^{(k:k-2D)}\|_{2,\infty} + 24M^2 D d \eta(\eta D M +1) + 4NM d)).
		\end{align*}
		Then we can conclude this lemma.
	\end{proof}

	Based on the above lemmas, we establish the following theorem for $\E\|\Delta^{(k)}\|$:

	\begin{proposition}\label{them:one-iteration}
		Assuming that $f$ is $M$-smooth and $\mu$-strongly convex, and Requirement \ref{assum5} is satisfied, if $\eta \le \min\{\frac{\mu \sqrt{n}}{8\sqrt{10}DMN}, \frac{2}{m+M}\} $, we have for all $k \ge 0$
		\begin{align*}
		\E\|\Delta^{(k+1)}\|^2 \le (1-\frac{\eta \mu}{2})\E\|\Delta^{(k:k-2D)}\|_{2,\infty} + C_1\eta^3 + C_2 \eta^2,
		\end{align*}
		where both $C_1$ and $C_2$ are constants that only depend on $M,N,D,\mu$.
	\end{proposition}

	\begin{proof}
		Since $E[\Psi^{(k)} | x^{(k)}] = 0$, we have
		\begin{align*}
		&\E \|\Delta^{(k+1)}\|^2 = \E \|\Delta^{(k)} -\eta U^{(k)} -V^{(k)} + \eta E^{(k)} \|^2 + \eta^2 \E \|\Psi^{(k)}\|^2
		\\
		&\le (1+\alpha)\E \|\Delta^{(k)} -\eta U^{(k)}\|^2 +(1+\frac{1}{\alpha})\E \|V^{(k)} + \eta E^{(k)} \|^2  + \eta^2 \E \|\Psi^{(k)}\|^2
		\\
		&\le (1+\alpha)\E \|\Delta^{(k)} -\eta U^{(k)}\|^2 +2(1+\frac{1}{\alpha})(\E \|V^{(k)}\|^2 + \eta^2 \E\|E^{(k)} \|^2)  + \eta^2 \E \|\Psi^{(k)}\|^2,
		\end{align*}
		where the first and the second inequalities are due to the Young's inequality.
		
		By substituting the bound in Lemma \ref{lemma:UV} and Lemma \ref{lemma:phi}, we can get a one step result for $\E \|\Delta^{(k+1)}\|^2$.
		\begin{align*}
		\E \|\Delta^{(k+1)}\|^2 &\le (1+\alpha)(1-\eta \mu)^2\E \|\Delta^{(k)}\|^2 + 2(1+\frac{1}{\alpha} )(\frac{\eta^4M^3d}{3}+h^3M^2d) + 
		\\
		&\quad  \frac{N\eta^2}{n}(4(1+\frac{1}{\alpha})(N+n)+1)\sum_{i=1}^N \E \|\nabla f_i(x^{(k)}) - \alpha_i^{(k)}\|^2  
		\\
		& \le (1+\alpha)(1-\eta \mu)^2\E \|\Delta^{(k)}\|^2 + 2(1+\frac{1}{\alpha} )(\frac{\eta^4M^3d}{3}+h^3M^2d) 
		\\
		&\quad +\frac{5N(N+n)\eta^2}{n}(1+\frac{1}{\alpha})\frac{M^2}{N}\big(4\eta Dd  + 32\eta^2 D^2M^2\E\|\Delta\|_k^{2D}
		+48\eta^3 M^2D^3d(\eta D M +1) + 8\eta^2 MND^2d\big) 
		\\
		&\le \big( (1+\alpha)(1-\eta \mu)^2 + \frac{160D^2M^4(N+n)\eta^4}{n}(1+\frac{1}{\alpha})\big)\E\|\Delta^{(k:k-2D)}\|_{2,\infty} + C,
		\end{align*}
		where $C = 2(1+\frac{1}{\alpha} )(\frac{\eta^4M^3d}{3}+\eta^3M^2d) + +\frac{5M^2(N+n)\eta^2}{n}(1+\frac{1}{\alpha})\big(4\eta Dd 
		+48\eta^3 M^2D^3d(\eta D M +1) + 8\eta^2 MND^2d\big) $
		
		By choosing $\alpha =\eta \mu < 1$ and $\eta \le \frac{\mu \sqrt{n}}{8\sqrt{10(N+n)}DM^2}$,
		we have $(1+\alpha)(1-\eta \mu)^2+\frac{160D^2M^4(N+n)\eta^4}{n}(1+\frac{1}{\alpha}) \le 1- \frac{\eta \mu}{2}$
		and 
		\begin{align*}
		C &\le 2(1+\frac{1}{\alpha} )(\frac{\eta^4M^3d}{3}+\eta^3M^2d) + +\frac{5M^2(N+n)\eta^2}{n}(1+\frac{1}{\alpha})\big(4\eta Dd 
		+48\eta^3 M^2D^3d(\eta D M +1) + 8\eta^2 MND^2d\big)
		\\
		&\le \eta^3\underbrace{\big(\frac{4M^3d}{\mu}+\frac{10m^2(N+n)}{n\mu}(\frac{3\mu^2 D^2nd}{40(N+n)M}+\frac{6D^2d\mu\sqrt{n}}{\sqrt{10(M+n)}}+8MND^2d)\big)}_{C_1} + \eta^2 \underbrace{(\frac{4M^2d}{\mu}+\frac{40M^2 (N+n)Dd}{n\mu})}_{C_2}.
		\end{align*}
		
		Then we can simplify the one iteration relation into 
		\begin{align*}
		\E\|\Delta^{(k+1)}\|^2 \le (1-\frac{\eta \mu}{2}) \E\|\Delta^{(k:k-2D)}\|_{2,\infty} + C_1\eta^3 + C_2 \eta^2. 
		\end{align*}
	\end{proof}

	Based on Proposition ~\ref{them:one-iteration}, we give the main result for AGLD, which characterizes the order of the step-size $\eta$ and the iteration number $K$(i.e. the mixture time) in order to guarantee $\WM_2(p^{(k)},p^*) \le \epsilon$.
	\begin{theorem}\label{them:main}
		Assume that f is $M$-smooth and $\mu$-strongly convex, and Requirement \ref{assum5} is satisfied.
		AGLD output sample $\xB^{(k)}$ with its distribution $p^{(k)}$ satisfies $\WM_2( p^{(k)}, p^*) \le \epsilon$ for any $k\geq K = \OM(\log ({1}/{\epsilon})/\epsilon^2)$ by setting $\eta = \OM(\epsilon^2)$.
		%		where $p^{(K)}$ is the distribution of $x^{(K)}$.
	\end{theorem}
	\begin{proof}
		Now we try to get a $\epsilon$-accuracy 2-Wasserstein distance approximation .
		In order to use Lemma \ref{one-iter}, we can assume that $\E\|\Delta^{(k)}\|^2 > \frac{\epsilon^2}{4}$(for otherwise, we already have get $\epsilon/2$-accuracy ) and $\frac{C_1\eta^3}{\eta \mu /2} \le \frac{\epsilon^2}{16}$ and $\frac{C_2 \eta^2}{\eta \mu/2} \le \frac{\epsilon^2}{16}$.
		Then by using lemma \ref{one-iter} and the fact that $|a|^2+|b|^2|+|c|^2 \le (|a|+|b|+|c|)^2$, the Wasserstein distance between $p^{(k)}$ and $p^*$  is bounded by
		\begin{displaymath}\label{W_C}
		W_2(p^{(k)}, p^*) \le exp(-\frac{\mu \eta \lceil k/(2D) \rceil}{4})W_0 + \frac{C_1 \eta}{\sqrt{\mu}} +\frac{C_2 \sqrt{\eta}}{\sqrt{\mu}}.
		\end{displaymath}
		Then by requiring that $exp(-\frac{\mu \eta \lceil k/(2D) \rceil}{4})W_0 \le \frac{\epsilon}{2}$  ,$\frac{C_1 \eta}{\sqrt{\mu}} \le \frac{\epsilon}{4}$, $\frac{C_2 \sqrt{\eta}}{\sqrt{\mu}}\le \frac{\epsilon}{4}$, we have $W_2(p^{(k)}, p^*) \le \epsilon$.
		That is $\eta =\OM(\epsilon^2)$ and $k = \OM(\frac{1}{\epsilon^2}\log \frac{1}{\epsilon})$
	\end{proof}
	
	\begin{lemma}\label{one-iter}
		Given a positive sequence $\{a_i\}_{i=0}^N$ and $\rho \in (0,1)$, if we have $\frac{C}{\rho} < a_i$ for all $i \in \{1,2,\cdots,N\}$ and $a_k \le (1-\rho) \max(a_{[k-1]+}, a_{[k-2]+},\cdots,a_{[k-D]+} )+C$, we can conclude
		\begin{align*}
		a_k \le (1-\rho)^{\lceil k/D \rceil} a_0+ \sum_{i=1}^{\lceil k/D \rceil}(1-\rho)^{i-1} C \le \exp(-\rho\lceil k/D \rceil)a_0 + \frac{C}{\rho}.
		\end{align*} 
	\end{lemma}
	\begin{proof}
		For all $i \in \{1,2,\cdots,D\}$, we have $a_i \le (1-\rho) a_0 + C < a_0$.
		
		Then $a_{D+1} \le (1-\rho) \max(a_{D}, a_{D-1},\cdots,a_{1}) +C \le (1-\rho)^2 a_0 + \sum_{i=1}^2(1-\rho)^{i-1}C < (1-\rho)  a_0 + C$.
		
		And $a_{D+2} \le (1-\rho) \max(a_{D+1}, a_{D},\cdots,a_{2}) +C \le (1-\rho)^2 a_0 + \sum_{i=1}^2(1-\rho)^{i-1}C < (1-\rho)  a_0 + C$.
		
		By repeating this argument, we can conclude $a_k \le (1-\rho)^{\lceil k/D \rceil} a_0+ \sum_{i=1}^{\lceil k/D \rceil}(1-\rho)^{i-1} C$ by induction.
		
		Since $1-x \le \exp^{-x}$ and $\sum_{i=1}^N(1-\rho)^{i-1}C \le \frac{C}{\rho}$,
		we conclude this lemma.
		
	\end{proof}

	\subsection{B.Improved results under additional smoothness assumptions}
	Under the Hessian Lipschitz-continuous condition, we can improve the convergence rate of AGLD with random access.
	
	\noindent {\bf Hessian Lipschitz}: There exists a constant $L > 0$ such that for all $x,y \in \RB^d$
	\begin{equation}
	\|\nabla^2 f(x) - \nabla^2 f(y)\| \le L \|x-y\|^2_2.
	\end{equation}
	
	We first give a technical lemma
	\begin{lemma}\label{lemma:VS}[\cite{dalalyan2017user}]
		Assuming the $M$-smoothness $\mu$-strongly convexity and $L$ Hessian Lipschitz smoothness of $f$, we have
		\begin{align}
		&\E \|S^{(k)}\|^2 \le \frac{\eta^3 M^2 d}{3},\\
		&\E \|V^{(k)} - S^{(k)}\|^2 \le \frac{\eta^4 (L^2d^2+M^3d)}{2},
		\end{align}
		where $S^{(k)} = \sqrt{2}\int_{k\eta}^{(k+1)\eta}\int_{k\eta}^s \nabla^2 f(y(r))dW(r)ds$.
	\end{lemma}

	%	\begin{them}\label{them:impro}
	%		Under \textbf{Condition f} and Hessian Lipschitz condition, we have for $k \ge 0$ and $\eta \le \min\{\frac{\mu \sqrt{n}}{8\sqrt{10}DMN}, \frac{2}{m+M}\} $, the Wasserstein distance between $p^{(k)}$ and $p^*$  is bounded by
	%		\begin{displaymath}
	%		W_2(p^{(k)}, p^*) \le exp(-\frac{\mu \eta \lceil k/D \rceil}{4}) + \frac{C_1 \eta}{\sqrt{\mu}} +\frac{C_2 \sqrt{\eta}}{\sqrt{\mu}}(1 - \IM_{\{\textbf{RA}\}}),
	%		\end{displaymath}
	%		where $C_1$ and $C_2$ are constant that only depends  on $L,M,N,D,\mu$ and is independent of $k$.
	%	\end{them} 
	\begin{theorem}\label{them:impro}
		Assume that $f$ is $M$-smooth, $\mu$-strongly convex, and has {Lipschitz-continuous Hessian}, then AGLD variants with {RA}
		output sample $\xB^{(k)}$ with its distribution $p^{(k)}$ satisfies $\WM_2( p^{(k)}, p^*) \le \epsilon$ for any $k\geq  = \OM(\log ({1}/{\epsilon})/\epsilon)$ by setting $\eta = \OM(\epsilon)$.
		%		where $p^{(K)}$ is the distribution of $x^{(K)}$.
	\end{theorem} 
	\begin{proof}
		The proof is similar to the proof in Theorem \ref{them:main}, but there are some key differences.
		First, we also give the one-iteration result here.
		Since $E[\Psi^{(k)} | x^{(k)}] = 0$, we have
		\begin{align*}
		&\E \|\Delta^{(k+1)}\|^2 = \E \|\Delta^{(k)} -\eta U^{(k)} -(V^{(k)}-S^{(k)}) - S_k + \eta E^{(k)} \|^2 + \eta^2 \E \|\Psi^{(k)}\|^2
		\\
		&\le (1+\alpha)\E \|\Delta^{(k)} -\eta U^{(k)} -S^{(k)}\|^2 +(1+\frac{1}{\alpha})\E \|V^{(k)}-S^{(k)} + \eta E^{(k)} \|^2  + \eta^2 \E \|\Psi^{(k)}\|^2
		\\
		&\le (1+\alpha)(\E \|\Delta^{(k)} -\eta U^{(k)}\|^2 + \E\|S^{(k)}\|^2)+ \eta^2 \E \|\Psi^{(k)}\|^2 +2(1+\frac{1}{\alpha}) (\E \|V^{(k)}-S^{(k)}\|^2 + \eta^2 \E\|E^{(k)} \|^2) ,
		\end{align*}
		where in the second inequality, we use the fact that $\E(S^{(k)}|\Delta^{(k)},u^{(k)}) = 0$.
		By substituting the bound in Lemma \ref{lemma:UV} , Lemma \ref{lemma:phi} and Lemma \ref{lemma:VS}, we can get a one step result for $\E \|\Delta^{(k+1)}\|^2$ in the same way as in Proposition \ref{them:one-iteration}.
		\begin{align*}
		\E\|\Delta^{(k+1)}\|^2 \le (1-\frac{\eta \mu}{2})\E \|\Delta\|_{k}^{2D} + C_1\eta^3 + C_2 \eta^2(1-\IM_{\{\textbf{RA}\}}). 
		\end{align*}
		Here, we can see that for \textbf{RA}, the $\eta^2$ term has now disappeared and that is the reason why we can get a better result.
		Then following similar argument as the proof of Theorem \ref{them:main}, it can be verified that AGLD with \textbf{RA} procedure can achieve $\epsilon$-accuracy after $k = \OM(\frac{1}{\epsilon}\log \frac{1}{\epsilon})$ iterations by setting $\eta = \OM(\epsilon)$.
		However, for \textbf{CA} and \text{RS}, we still need $\eta = \OM(\epsilon^2)$ and $k = \OM(\frac{1}{\epsilon^2}\log \frac{1}{\epsilon})$.
	\end{proof}
	
	\subsection{A deeper analysis for TMU-RA}
	Here, we probe into the TMU-RA algorithm to have a better understanding about the newly proposed Snapshot-Updating strategy.
	We first give the proof of the Theorem 3 in the main paper, and then extend the results to the general convex $f$.
	%\begin{theorem}\label{them:tmura}
	%	Assuming the $M$-smoothness, $\mu$-strongly convexity and the $L$ Hessian Lipschitz of $f$, if we set $\eta < \frac{\epsilon b \sqrt{\mu} }{M \sqrt{d N}}$, $n \ge 9$ and $D =  N$, then TMU-RA achieves $\epsilon$-accuracy $2$-Wasserstein distance after $K = \tilde \OM(\frac{\kappa^{3/2}\sqrt{d}}{n \epsilon})$ iterations and $T_g = \tilde \OM(N+\frac{\kappa^{3/2}\sqrt{d}}{ \epsilon})$  component gradient estimation, where $\kappa = \frac{M}{\mu}$ denotes the condition number and we hide the logarithm term in $\tilde \OM$. 
	%\end{theorem}
	\begin{theorem}\label{them:tmura}
		Assume $f$ is $M$-smooth, $\mu$-strongly convex and the has $L$-lipschitz Hessian and denote $\kappa = {M}/{\mu}$.
		TMU-RA output sample $\xB^{(k)}$ with its distribution $p^{(k)}$ satisfies $\WM_2( p^{(k)}, p^*) \le \epsilon$ for any $k\geq K = \tilde \OM(\kappa^{3/2}\sqrt{d}/{(n \epsilon)})$ if we set $\eta < \epsilon n \sqrt{\mu}/ {(M \sqrt{d N})}$, $n \ge 9$, and $D =  N$.
		Moreover, the total number of component gradient evaluations is $T_g = \tilde \OM(N+{\kappa^{3/2}\sqrt{d}}/{ \epsilon})$.	
		%		Further, to achieve such state, the corresponding component gradient evaluations is $T_g = \tilde \OM(N+{\kappa^{3/2}\sqrt{d}}/{ \epsilon})$.		 
		%		Note that we hide the logarithm terms in $\tilde \OM$.
	\end{theorem}
	\begin{proof}[Proof of Theorem ~\ref{them:tmura}]
		Follow similar procedure as "B.2 SAGA Proof" of \cite{chatterji2018theory}, we can establish the result for TMU-RA.
		The only difference is that snapshot are now totally updated every $D =N$ iterations, and we need to adjust the Step 5 and Step 9 in the original proof for SAGA-LD.
		
		In Step 5, we need to bound $\EBB \|h_{i}^{(k)} - \nabla f_i(y^{(k)})\|^2_2$, where $\{h_i^{(k)}\}$ denotes the snapshot for $\{y^{(k)}\}$ updated at the same time with $\{\alpha_i^{(k)}\}$. 
		Let $p = 1 -(1-1/N)^n$ denotes the probability that a index is chosen in RA and $N_k = \lfloor (k-1)/N \rfloor \cdot N$ denotes the nearest multiple of $N$ that is smaller than $k$, then for TMU-RA, we have
		
		\begin{align}\label{step5}
		&\quad \EBB \|h_{i}^{(k)} - \nabla f_i(y^{(k)})\|^2_2 \nonumber
		\\
		&= \sum_{j = N_k}^{k-1} 		\EBB[ \|h_{i}^{(k)} - \nabla f_i(y^{(k)})\|^2_2 | h_{i}^{(k)} = \nabla f_i(y^{(j)})]\cdot \PBB[h_{i}^{(k)} = \nabla f_i(y^{(j)})]	\nonumber
		\\
		&\le \tilde M^2\sum_{j = N_k}^{k-1} 		\EBB[ \|y^{(j)} - \nabla y^{(k)}\|^2_2)]\cdot \PBB[h_{i}^{(k)} = \nabla f_i(y^{(j)})]	\nonumber
		\\
		&\le \tilde M^2\sum_{j = N_k+1}^{k-1}\EBB[ \|y^{(j)} -  y^{(k)}\|^2_2 ]\cdot (1-p)^{k-1-j}p +  \EBB[ \|y^{(N_k)} - y^{(k)}\|^2_2 ]\nonumber \quad\cdot (1-p)^{k-1-N_k}]\nonumber	
		\\	
		& \le \tilde M^2\sum_{j = N_k+1}^{k-1}\big[2 \delta^2(k-j)^2 Md+ 4 \eta d (k-j)\|^2_2 \big]\cdot (1-p)^{k-1-j}p + \big[ 2 \delta^2(k-N_k)^2 Md+ 4 \eta d (k-N_k)\big]\cdot (1-p)^{k-1-N_k}]\nonumber	
		\\
		& \le \tilde M^2\sum_{j = N_k+1}^{k-1}\big[2 \eta^2(k-j)^2 Md+ 4 \eta d (k-j)\|^2_2 \big]\cdot (1-p)^{k-1-j}p + \big[ 2 \eta^2(k-N_k)^2 Md+ 4 \eta d (k-N_k)\big]\cdot (1-p)^{k-1-N_k}]	\nonumber
		\\
		& \le \tilde M^2\sum_{j = 1}^{k-N_k-1}\big[2 \eta^2j^2 Md+ 4 \eta d j\|^2_2 \big]\cdot (1-p)^{j-1} p + \big[ 2 \eta^2(k-N_k)^2 Md+ 4 \eta d (k-N_k)\big]\cdot (1-p)^{k-1-N_k}]	\nonumber
		\\
		& \le \tilde M^2\sum_{j = 1}^{\infty}\big[2 \eta^2j^2 Md+ 4 \eta d j\|^2_2 \big]\cdot (1-p)^j p \le \frac{2\eta^2 \tilde M^2 M d}{p^2} + \frac{4d \eta \tilde M^2}{p},
		\end{align}
		where the second inequality is due to the update rule of TMU, the third inequality follows from ~\ref{yinter}, and the rest are just basic calculation.
		Inequality \ref{step5} is just the same as Step 5 in the proof of SAGA-LD in \cite{chatterji2018theory}.
		
		In Step 9, the authors establish following iterative relation for the Lyapunov function $T_k = c\sum_i^N \|\alpha_i^{(k)} -h_i^{(k)} \|^2_2 + \|x^{(k)}-y^{(k)}\|^2_2$,
		\begin{align}\label{step9}
		\EBB [T_{k+1}] \le (1-\rho) T_k + 2 \eta^3 \Box + \frac{2\eta^3 \triangle}{\mu},	
		\end{align}
		where $\rho = \min\{\frac{n}{3N},\frac{\mu \eta}{2}\}$, $\Box = 2 M^2 d + \frac{72 N M^2}{n^2}(\frac{\eta M N}{n}+1)$ and $\triangle = \frac{1}{2}(L^2d^2 + M^3 d)$.
		For TMU-RA, we can just follow the same way to get \ref{step9} if $k \mod N \ne 0$. 
		If $k \mod N = 0$, since we update the whole snapshot set, the result in Step 7 of their analysis  no longer holds and we can not conclude~(\ref{step9}) following their analysis.
		
		Actually, if $k \mod N = 0$, then $g^{(k)} = \nabla f(x^{(k)})$.
		Follow \ref{delta}, we have 
		\begin{align}\label{deltaN}
		&\quad\EBB  \|y^{(k+1)}-x^{(k+1)}\|^2_2 \nonumber
		\\
		&= \EBB \|\underbrace{y^{(k)}-x^{(k)}}_{\Delta^{(k)}} - \eta(\underbrace{\nabla f(y^{(k)} - \nabla f(x^{(k)})}_{U^{(k)}}) - \underbrace{\sqrt{2} \int_{k\eta}^{(k+1)\eta}\int_{k \eta}^{s} \nabla^2 f(y(t))dB(t) ds}_{\Upsilon^{(k)}}\nonumber
		\\
		&\quad -\underbrace{\int_{k\eta}^{(k+1)\eta}\big\{\nabla f(y(s)) - \nabla f(y^{(k)}) -\sqrt{2} \int_{k \eta}^{s} \nabla^2 f(y(t))dB(t)\big\}}_{\bar V^{(k)}} ds\|^2_2\nonumber
		\\
		& \le (1+a) \EBB \|\Delta^{(k)} - \eta U^{(k)} - \Upsilon^{(k)}\|^2_2 + (1+\frac{1}{a})\EBB \|\bar V^{(k)}\|^2_2\nonumber
		\\
		& \le (1+a) \EBB \|\Delta^{(k)} - \eta U^{(k)}\|^2_2 + (1+a)\EBB \| \Upsilon^{(k)}\|^2_2 + (1+\frac{1}{a})\EBB \|\bar V^{(k)}\|^2_2,
		\end{align}
		where the first inequality is due to Cauchy-Schwartz inequality and the second inequality follows from the fact that $\EBB [\Upsilon^{(k)}|,x^{(k)},y^{(k)}] =0$ and $\Delta^{(k)} - \eta U^{(k)} \perp \Upsilon^{(k)}| x^{(k)},y^{(k)}$.
		
		By lemma \ref{lemma:UV} and lemma \ref{lemma:VS}, we have 
		\begin{displaymath}
		\left \{\begin{array}{l}
		\E \|\bar V^{(k)}\|^2 \le \frac{1}{2}\eta^4( M^3 d + L^2d^2),\\
		\E \|\Delta^{(k)} - \eta U^{(k)}\|^2 \le (1- \eta \mu)^2 \E \|\Delta^{(k)}\|^2,\\
		\E \|\Upsilon^{(k)}\|^2_2 \le \frac{2M^2\eta^3 d}{3}.
		\end{array}\right.
		\end{displaymath}
		By substituting the above inequality back to \ref{deltaN}, we have
		\begin{align*}
		\EBB  \|y^{(k+1)}-x^{(k+1)}\|^2_2 &\le (1+a)(1-\eta \mu)^2 \E \|\Delta^{(k)}\|^2 +\frac{2 (1+a)M^2\eta^3 d}{3} +  \frac{1}{2}\eta^4(1+\frac{1}{a})( M^3 d + L^2d^2).
		\end{align*}
		Since all the snapshot $\alpha_i^{(k)}$ are now updated at $x^{(k)}$, we have
		\[
		\sum_{i=1}^N \|\alpha_i^{(k)} -h_i^{(k)} \|^2_2 = \sum_{i=1}^N \|\nabla f_i(x^{(k)}) -\nabla f_i(y^{(k)}) \|^2_2 \le \frac{M^2 \|x^{(k)}-y^{(k)}\|^2_2}{N},
		\]
		where the inequality follows from the smoothness of $f_i$.
		
		Thus we have
		\begin{align*}
		& \quad \EBB [T_{k+1}] =\E \|x^{(k)}-y^{(k)}\|^2_2 + c\sum_i^N E \|\alpha_i^{(k)} -h_i^{(k)} \|^2_2 \\
		&\le (1+a)(1-\eta \mu)^2 \E \|\Delta^{(k)}\|^2 +\frac{2 (1+a)M^2\eta^3 d}{3}+ \frac{1}{2}\eta^4(1+\frac{1}{a})( M^3 d + L^2d^2)+ \frac{c M^2 }{N} \E \|\Delta^{(k)}\|^2.
		\end{align*}
		By choosing $a = \eta \mu \le 1/6$ and $c \le \frac{\eta \mu}{2}$, we can conclude
		\begin{align}
		\EBB [T_{k+1}] \le (1-\rho) T_k + 2\eta^3 \Box + \frac{2\eta^3 \triangle}{\mu}.	
		\end{align}
		Note that in Step 9 of \cite{chatterji2018theory}, they require $c \le  \frac{24 \eta^2 N^2}{n^2}$.
		If we choose $\eta \le \frac{\mu n^2}{48 N^2}$, then $\frac{24 \eta^2 N^2}{n^2} \le \frac{\eta \mu}{2}$.

		With all these in hand, we can follow the proof of SAGA-LD and obtain similar result for TMU-RA as Theorem 4.1. in \cite{chatterji2018theory}. 
		
		\begin{align}\label{Wtmu}
		\WM_2(p^{(k)},p^*) \le 5 \exp(- \frac{\mu\eta}{T}) \WM_2(p^{(0)},p^*) + \frac{2 \eta L d}{\mu} + \frac{2 \eta M^{3/2}\sqrt{d}}{\mu} + \frac{24\eta M \sqrt{d N}}{\sqrt{m}n}.
		\end{align}
		
		Then by making each part in the right hand of \ref{Wtmu} less than $\epsilon/4$, and treating $M$,$\mu$ and $L$ as constants of order $\OM(N)$ if they appear alone, we complete the proof of Theorem \ref{them:tmura} and conclude that the component gradient evaluations is $T_g = \tilde \OM(N+{\kappa^{3/2}\sqrt{d}}/{ \epsilon})$.

	\end{proof}
	
	\subsection{Extension to general convex $f$}
	By inequality (A.17) in supplementary in \citep{zou2018stochastic}, we have the following lemma.
	\begin{lemma}\label{lemma:piandbarpi}
		Assuming the target distribution $p^* \propto e^{-f}$ has bounded forth order moment, i.e. $\EBB_{p^*} [\|x\|_2^4]\le \hat U d^2$, and $\hat f = f(x) + {\lambda \|x\|^2}/{2}$, then for $\hat p^* \propto e^{-\hat f}$, we have
		$\WM_2(\hat p^*, p^*) \le {\sqrt{\lambda \hat U d^2}}/{2}$.
	\end{lemma}

	As we have figure out the dependency on the condition number $\kappa$ for TMU-RA in Theorem 3.Combining Theorem \ref{them:tmura} and Lemma \ref{lemma:piandbarpi}, we have the following result.
	\begin{theorem}\label{them:general}
		Suppose the assumptions in Theorem \ref{them:main} hold and further assume the target distribution $p^* \propto e^{-f}$ has bounded forth order moment, i.e. $\EBB_{p^*} [\|x\|_2^4]\le \hat U d^2$.
		If we choose $\lambda = 4\epsilon^2 / (\hat U d^2)$ and run the AGLD algorithm with $\hat f(x) = f(x) + {\lambda \|x\|^2}/{2}$, we have $\WM_2( p^{(k)}, p^*) \le \epsilon$ for any $k \ge K = \OM(\log ({1}/{\epsilon})/\epsilon^8)$.
		If we further assume that $f$ has Lipschitz-continuous Hessian, then SVRG-LD, SAGA-LD, and TMU-RA can achieve $\WM_2( p^{(K)}, p^*) \le \epsilon$ in $K = \OM(\log ({1}/{\epsilon})/\epsilon^3)$ iterations.
	\end{theorem}
	\begin{proof}
		According to the fact that $\hat f$ is $(L+\lambda)$ smooth and $\lambda$ strongly convex, and $\lambda = \OM(\epsilon^2/d^2)$,
		it can be figured out that both $C_1$ and $C_2$ in (\ref{W_C}) are $\OM(\epsilon^2)$.
		Then by requiring that $exp(-\frac{\lambda \eta \lceil k/(2D) \rceil}{4})W_0 \le \frac{\epsilon}{2}$  ,$\frac{C_1 \eta}{\sqrt{\lambda}} \le \frac{\epsilon}{4}$, $\frac{C_2 \sqrt{\eta}}{\sqrt{\mu}}\le \frac{\epsilon}{4}$, we have $W_2(p^{(K)}, p^*) \le \epsilon$ in $K = \OM(\log ({1}/{\epsilon})/\epsilon^8)$ iterations..
		
		Similarly, according to Theorem ~\ref{them:tmura}, we need $T_g = \tilde \OM(N+\frac{(L+\lambda)^{3/2}\sqrt{d}}{ \lambda^{3/2}\epsilon})$ to ensure $\WM_2(p^{(K)},\bar \pi) \le \epsilon/2$ for SVRG-LD, SAGA-LD, and TMU-RA.
		By Lemma~\ref{lemma:piandbarpi}, $\lambda = \OM(\epsilon^2 / d^2)$ indicates $\WM_2(\bar \pi, p^*)\le \epsilon/2$.
		Combining this together and treating the smoothness constant $(L+\lambda)$ as constant of order $\OM(N)$, we conclude this proof.		 
	\end{proof}
	
	\subsection{Theoretical results for nonconvex $f(x)$}
	In this section, we characterize the convergence rates of AGLD for sampling from non-log-concave distributions.
	We first lay out the assumptions that are necessary for our theory.
	\begin{assumption}\label{assump:dis}[Dissipative]
		There exists constants $a,b > 0$ such that the sum $f$ satisfies 
		$$
		\langle \nabla f(x), x \rangle \ge b\|x\|_2^2 -a,
		$$ 
		for all $x\in R^d$.
	\end{assumption}
	This assumption is typical for the ergodicity analysis of stochastic differential equations and diffusion approximations.
	It indicates that, starting from a position that is sufficiently far from the origin, the Lagevin dynamics (\ref{eq:ydynamics}) moves towards the origin on average.

	In order to analyze the long-term behavior of the error between the discrete time AGLD algorithm and the continuous Langevin dynamics, we follow \cite{raginsky2017non,zou2019sampling} and construct the following continuous time Markov process $\{x(t)\}_{t \ge 0} $ to describe the approximation sequence $x^{(k)}$'s:
	\begin{equation}
	d x(t) =  -h(x(t)) dt + \sqrt{2} d B(t),
	\end{equation}
	where $h(x(t)) = \sum_{k=0}^\infty g^{(k)} \IBB_{t \in [\eta k, \eta(k+1)]} $ and $g^{(k)}$ are the aggregated gradient approximation constructed in the $k$-th step of AGLD.
	By integrating $x(t)$ on interval $[\eta k, \eta(k+1))$, we have
	$$
	x(\eta (k+1)) = x(\eta k) - \eta g^{(k)} + \sqrt{2} \xi^{(k)},
	$$
	where $\xi^{(k)}$  is a standard Gaussian variable.
	This implies that the distribution of$\{x(0),\cdots,x(\eta k),\cdots\}$  is equivalent to  $\{x^{(0)},\cdots,x^{(k)},\cdots\}$, i.e., the iterates in AGLD.
	Note that $x(t)$ is not a time-homogeneous Markov chain since the drift term $h(x(t))$ also depends on some historical iterates.
	However, \cite{gyongy1986mimicking} showed that one can construct an alternative Markov chain
	which enjoys the same one-time marginal distribution as that of $x(t)$ and is formulated as follows,
	\begin{equation}
	d \tilde x(t) = -\tilde h(\tilde x(t)) dt + \sqrt{2} d B(t),
	\end{equation}
	where $\tilde h(\tilde x(t)) = \EBB[h(x(t))|x(t) = \tilde x(t)]$.
	We denote the distribution of $\tilde x(t)$ as $\PBB_t$, which is identical to that of $x(t)$.
	Recall the Langevin dynamics starting from $y(0) = x(0)$, i.e.,
	\begin{equation}\label{eqn:yt}
	d y(t) = -\nabla f(y(t)) dt + \sqrt{2} dB(t),
	\end{equation}
	and define the distribution of $y(t)$ as $\QBB_t$.
	According to the Girsanov formula, the Radon-Nykodim derivative of $\PBB_t$ with respect to $\QBB_t$ is 
	\begin{equation}
	\frac{d \PBB_t}{d \QBB_t} (\tilde x(s)) = \exp \{- \int_{0}^t (h(\tilde x(s)) - \nabla f(\tilde x(s)) )^T d B(s) -\frac{1}{4}\int_0^t \EBB\|h(\tilde x(s)) - \nabla f(\tilde x(s))\|^2_2ds\},
	\end{equation}
	which in turn indicates that the KL-divergence between $\PBB_t$ and $\QBB_t$ is
	\begin{equation}
	KL(\QBB_t||\PBB_t) = - \E [\log(\frac{d \PBB_t}{d \QBB_t} (\tilde x(s)) )] = \frac{1}{4}\int_0^t \EBB\|h(\tilde x(s)) - \nabla f(\tilde x(s))\|^2_2ds.
	\end{equation} 
	
	According to the following lemma, we can upper bound the $\WM_2$ distance $\WM_2(P(x^{(k)}),x(\eta k))$ with the KL-divergence $KL(\QBB_{\eta k}||\PBB_{\eta k})$.
	\begin{lemma}[\cite{bolley2005weighted}]\label{WtoKL}
		For any two probability measures $P$ and $Q$, if they have finite second moments, we have
		\begin{equation*}
		\WM_2(Q,P) \le \Lambda (\sqrt{KL(Q||P)}+\sqrt[4]{KL(Q||P)}),
		\end{equation*}
		where $\Lambda = 2 \inf_{\lambda >0 }\sqrt{1/\lambda(3/2 + \log \EBB_{x \sim P} [\exp(\lambda \|x\|^2_2)])}$
	\end{lemma}
	
	\begin{lemma}[Lemma 3.3 in \cite{raginsky2017non}]\label{logexp}
		Under Assumption \ref{assum:smooth} and \ref{assump:dis} , if $b \ge 2$,we have 
		$$
		\log \EBB[\exp(\|x(t)\|^2_2)] \le \|x(0)\|^2_2 + (2b+d)t.
		$$
	\end{lemma}
	Note that if $b$ is less than 2, we can divide $f$ by $b/2$ and consider the dynamic $d x(t) = 2\nabla f(x(t))/b + \sqrt{b} dB(t)$, whose stationary distribution is still the target distribution $p^* \propto exp(-f(x))$. It can be verified that the smoothness of $2f(x)/b$ is the same as $f(x)$ and the analysis we derive here is still suitable for  this dynamic.
	From now on, we assume that $b \ge 2$ holds for $f(x)$, and we do not make this transformation in order to keep the notation similar to that in the convex setting. 
	
	First, we establish some lemmas will be useful in the proof of the main results.
	\begin{lemma}[Lemma A.2 in \cite{zou2019sampling}]\label{lemma:xt}
		Under Assumption \ref{assum:smooth} and \ref{assump:dis}, the continuous-time Markov chain $y(t)$ generated by Langevin dynamics \ref{eqn:yt} converges exponentially to the stationary distribution $p^*$, i.e.,
		$$ \WM_2(P(y(t)),p^*) \le D_4 \exp(-t/D_5), $$
		where both $D_4$ and $D_5$ are in order of $\exp(\tilde \OM(d) ) $ if we use $a = \tilde \OM(b)$ to hide some logarithmic terms of $b$.
	\end{lemma}

	\begin{lemma}\label{lemma:fnorm}
		Under Assumption \ref{assum:smooth}, for all $x \in R^d $ and $i \in \{1,cdots, N\}$, we have 
		$$  \|\nabla f_i(x)\| \le \tilde M \|x\| + G \quad \textrm{and}  \quad \|\nabla f_i(x)\|^2 \le \tilde 2M^2 \|x\|^2 + 2G^2,$$
		where $G = \max_{i = 1,\cdots,N} \|f_i(0)\|$.
	\end{lemma}
	\begin{proof}
		According to the $\tilde M$-smoothness of $f_i$, we have
		$$
		\|f_i(x)\| = \|f_i(x) - f_i(0)+ f_i(0)\| \le \|f_i(x) - f_i(0)\|+\|f_i(0)\| \le \tilde M \|x\| + G.
		$$
		Follow the Cauchy-Schwartz inequality, we can conclude the second part of the lemma.
	\end{proof}

	\begin{lemma}\label{lemma:xbound}
		Under Assumption \ref{assum:smooth} and \ref{assump:dis}, for sufficiently small stepsize $\eta$, if the initial point $x^{(0)} = 0$, the expectation of the $l^2$ norm of iterates and aggregated gradient generated in AGLD is bounded by
		$$ \EBB \|x^{(k)}\|^2_2 \le 4(1+1/b)(a +G^2 + d) := D_B  \quad \textrm{and}  \quad  \EBB \|g^{(k)}\|^2_2 \le4N^2(2\tilde M^2 D_B + G^2),$$
		where $G = \max_{i = 1,\cdots,N} \|f_i(0)\|$.
	\end{lemma}
	\begin{proof}
		We prove the bound of $\EBB \|x^{(k)}\|$ by induction.
		
		When $k=1$, we have 
		$$
		\EBB \|x^{(1)}\|^2_2 = \EBB \|x^{(0)} - \eta \nabla f(x^{(0)}) + \sqrt{2\eta}\xi^{(0)}\|^2_2 = \eta^2 \EBB \|\nabla f(0)\|^2_2 + \EBB \|\sqrt{2\eta}\xi^{(0)}\|^2_2 \le \eta^2 G^2 + 2\eta d,
		$$
		where the second equality holds since $x^{(0)}=0$, $g^{(0)} = \nabla f(0)$,and $\xi^{(0)}$ is independent of $\nabla f(0)$.
		Thus, for sufficiently small $\eta$, it is easy to make the conclusion hold for $\EBB \|x^{(1)}\|^2_2$.
		
		Now assume that the result holds for all iterates from $1$ to $k$, then for the $(k+1)$-th iteration,we have
		\begin{equation}\label{eqn:xk}
		\EBB \|x^{(k)}\|^2_2 = \EBB \|x^{(k)} -\eta g^{(k)} \|^2_2+ \EBB \|\sqrt{2\eta} \xi^{(k)} \|^2_2\le  \EBB \|x^{(k)} -\eta g^{(k)} \|^2_2+ 2\eta d.
		\end{equation}
		For the first part of the last inequality, we have 
		\begin{align}
		&\EBB \|x^{(k)} -\eta g^{(k)} \|^2_2 = \EBB \|x^{(k)} - \eta \nabla f(x^{(k)}) + \eta \nabla f(x^{(k)}) -\eta g^{(k)} \|\nonumber
		\\
		=& \EBB \|x^{(k)} - \eta \nabla f(x^{(k)}) - \eta\Psi^{(k)} -\eta \Gamma^{(k)}\|^2_2 
		\nonumber
		\\
		=& \EBB \|x^{(k)} - \eta \nabla f(x^{(k)})-\eta \Gamma^{(k)}\|^2_2 + \eta^2 \EBB \| \Psi^{(k)} \|^2_2 + 2\eta^2 \EBB \langle \Psi^{(k)}, \Gamma^{(k)}\rangle\nonumber
		\\
		\le & (1+\alpha)\EBB \|x^{(k)} - \eta \nabla f(x^{(k)})\|^2_2+ (2+ \frac{1}{\alpha})\eta^2\EBB \| \Gamma^{(k)}\|^2_2 + 2\eta^2 \EBB \| \Psi^{(k)} \|^2_2
		\nonumber 
		\\
		\le &(1+\alpha)\EBB \|x^{(k)} - \eta \nabla f(x^{(k)})\|^2_2+ \frac{2\eta^2N\tilde M^2}{n}((N+n)(2+ \frac{1}{\alpha})+1)\sum_{i=1}^N \EBB \|x^{(k)} - x^{(k_i)}\|^2_2 
		,\label{eqn:xbound}
		\end{align}
		where the second equality follows from the definition of $\Psi^{(k)}$~(\ref{eqn:psi} ) and $\Gamma^{(k)}$~(\ref{eqn:gamma}), the third equality is due to that $\EBB [\Gamma^{(k)}] = 0 $ and the conditional independent of $\Gamma^{(k)}$ and $x^{(k)} - \eta \nabla f(x^{(k)})$,
		the first inequality follows from that $\EBB \|X + Y\|^2_2 \le (1+\alpha)\EBB \|X\|^2_2 + (1+\frac{1}{\alpha})\EBB\|Y\|^2_2$ and $2\EBB\langle X,Y\rangle \le \EBB \|X\|^2_2 + \EBB \|Y\|^2_2$ for $\forall \alpha >0$ and any random variable $X$ and $Y$,
		and the last inequality is due to Lemma~\ref{lemma:phi} and the $\tilde M$-smoothness of $f_i$.
		
		For the first term in~(\ref{eqn:xbound}), we have 
		\begin{align}\label{eqn:xfx}
		&\EBB \|x^{(k)} - \eta \nabla f(x^{(k)})\|^2_2 = \EBB \|x^{(k)}\|^2_2 + \eta^2 \EBB\|\sum_{i =1}^N \nabla f_i(x^{(k)})\|^2_2 -2 \eta \EBB \langle x^{(k)}, \nabla f(x^{(k)}) \rangle
		\nonumber
		\\
		\le &\EBB \|x^{(k)}\|^2_2 + 2\eta^2N( \tilde M^2 \EBB\|(x^{(k)}\|^2_2 + G^2)+2 \eta (a -b\EBB \| x^{(k)}\|^2_2)
		\nonumber
		\\
		= &(1 -2b\eta+ 2N\eta^2 \tilde M^2)\EBB \|x^{(k)}\|^2_2 + 2N\eta^2 G^2+2 \eta a,
		\end{align}
		where in the first inequality, we use $\|\sum_{i=1}^N a_i\|^2 \le N \sum_{i=1}^N \|a_i\|^2_2$, Lemma~\ref{lemma:fnorm} and the dissipative assumption of $f$.
		
		For the second term of~(\ref{eqn:xbound}), we have
		\begin{align}\label{eqn:xkxki}
		&\EBB \|x^{(k)} - x^{(k_i)}\|^2_2  = \EBB \|-\eta \sum_{j = k_i}^{k-1} g^{(j)} + \sqrt{2\eta}\xi^{(j)}\|^2_2\le 2 \eta^2 \EBB \|\sum_{j = k_i}^{k-1} g^{(j)} \|^2_2+2\EBB\|\sum_{j = k_i}^{k-1}  \sqrt{2\eta}\xi^{(j)}\|^2_2
		\nonumber \\
		\le &2D \eta^2  \sum_{j = k-D}^{k-1}\EBB \| g^{(j)} \|^2_2+  2\sum_{j = k-D}^{k-1}\EBB \|\sqrt{2\eta}\xi^{(j)}\|^2_2	\le 2D \eta^2  \sum_{j = k-D}^{k-1}\EBB \| g^{(j)} \|^2_2+  4Dd\eta,
		\end{align}
		where in the second inequality, we use $k_i \ge k-D$ and $\|\sum_{j=1}^D a_j\|^2 \le D \sum_{j=1}^D \|a_j\|^2_2$, and the independence of $\xi^{(j)}$'s.
		
		Moreover, we have 
		\begin{align}\label{eqn:gj}
		&\EBB \|g^{(j)}\|^2_2 = \E \|\sum_{p \in S_j}\frac{N}{n}(\nabla f_p(x^{(j)})-\alpha_{p}^{(j)})+ \sum_{p=1}^N \alpha^{(j)}_p \|^2
		\nonumber
		\\
		\le&   2 \E \|\sum_{p \in S_j}\frac{N}{n}(\nabla f_p(x^{(j)})-\alpha_{p}^{(j)})\|^2+ 2\E\|\sum_{p=1}^N \alpha^{(j)}_p \|^2\nonumber
		\\
		\le & \frac{2N^2}{n} \sum_{p \in S_j} \E \|\nabla f_p(x^{(j)})-\alpha_{p}^{(j)}\|^2+ 2N\sum_{p=1}^N \E\|\alpha^{(j)}_p \|^2\nonumber
		\\
		\le & \frac{2N^2\tilde M^2}{n} \sum_{p \in S_j} \E \|x^{(j)}-x^{(j_p)})\|^2+ 4N\sum_{p=1}^N (\tilde M^2\E\|x^{(j_p)}\|^2 + G^2)
		\nonumber
		\\
		\le& 8 N^2\tilde M^2 [\Delta^{(k)}_D]^+ + 4N^2 G^2
		\end{align}
		where we denote $[\Delta^{(k)}_D]^+ := \max\{\E \|x^{(j-D)}\|^2_2,\E \|x^{(j-D+1)}\|^2_2,\cdots,\E \|x^{(j)}\|^2_2\}$, and the third inequality follows from the definition of $\alpha_p^{(j)}$ and Lemma \ref{lemma:fnorm} .
		
		Combining (\ref{eqn:xk}) (\ref{eqn:xbound}), (\ref{eqn:xfx}), (\ref{eqn:xkxki}), and (\ref{eqn:gj}) and choosing $\alpha = b\eta /2 $, we have
		\begin{align}
		&\E\|x^{(k+1)}\|_2^2 \le \E \|x^{(k)} - \eta g^{(k)}\|^2_2 + 2\eta d
		\nonumber
		\\
		\le& (1+b\eta/2)((1 -2b\eta+ 2N\eta^2 \tilde M^2)\EBB \|x^{(k)}\|^2_2 + 2N\eta^2 G^2+2 \eta a) 
		\nonumber
		\\
		&+ \frac{2\eta^2N\tilde M^2}{n}((N+n)(2+ \frac{2}{b \eta})+1)N(2D^2 \eta^2 (8 N^2\tilde M^2 [\Delta^{(k)}_D]^+ + 4N^2G^2)+ 4 Dd\eta) + 2\eta d
		\nonumber\\
		\le &(1- 3b\eta/2 \underbrace{-b^2 \eta^2 + 2N\eta^2 \tilde M^2 + bN\eta^3\tilde M^2 + \frac{64N^4 \tilde M^4D^2 \eta^3}{n}(N+n)(2\eta + 2/b)}_{A}) [\Delta^{(k)}_D]^+ 
		\nonumber \\
		&+(1+b\eta /2)(2N \eta ^2 G^2 + 2\eta a) +  \underbrace{\frac{4N^2 \tilde M^2 \eta^3}{n}(N+n)(2\eta + 2/b))(8N^2G^2D^2\eta^2 + 4Dd\eta)}_{B} +2\eta d,
		\end{align} 
		By selecting small enough $\eta$, we can make $A \le b \eta /2$ ,$B \le 2\eta d$, $b \eta /2 \le 1$ and $ N\eta \le 1$ and thus have
		$$
		\E \|x^{(k+1)}\|_2^2 \le 4(1 - b \eta)(1+ 1/b)(a + G^2+ d) + 4 \eta (a +G^2 + d)\le 4(1+ 1/b)(a + G^2+ d),
		$$
		where we use the induction condition $ [\Delta^{(k)}_D]^+  \le 4(1+ 1/b)(a + G^2+ d)$.
		
		According to (\ref{eqn:gj}), we can establish 
		$$ \E \|g^{(k)})\|^2_2 \le 4N^2(2\tilde M^2 D_B + G^2) .$$
		
	\end{proof}
	
	Based on these lemmas, we now give our main theorem on the convergence of sample distribution.
	\begin{theorem}
		Under Assumption \ref{assum:smooth} and \ref{assump:dis}, AGLD
		output sample $\xB^{(k)}$ with its distribution $p^{(k)}$ satisfies $\WM_2( p^{(k)}, p^*) \le \epsilon$ for any $k\geq  = \OM(\log ({1}/{\epsilon})/\epsilon^4)$ by setting $\eta = \OM(\epsilon^4)$.
	\end{theorem}
	\begin{proof}
		Denote the distribution of $x^{(k)}$ and $x(\eta k)$ as $p^{(k)}$ and $Q_k$ respectively.
		By Lemma \ref{WtoKL}, we have
		$$
		\WM(Q_k,p^{(k)}) \le \Lambda(\sqrt{KL(Q_k||p^{(k)})}+\sqrt[4]{KL(Q_k||p^{(k)})}).
		$$
		By data-processing theorem in terms of KL-divergence, we have
		\begin{equation}\label{eqn:klqkpk}
		KL(Q_k||p^{(k)}) \le KL(\QBB_{\eta k}||\PBB_{\eta k}) = \frac{1}{4}\int_0^{k\eta} \EBB\|h(\tilde x(s)) - \nabla f(\tilde x(s))\|^2_2ds=\frac{1}{4}\int_0^t \EBB\|h(x(s)) - \nabla f(x(s))\|^2_2ds,
		\end{equation}
		where the last equality holds since $x(s)$ and $\tilde x(s)$ have the same one-time distribution.
		
		Since $h(x(s))$ is a step function and remains  constant when $s \in [v\eta,(v+1)\eta]$ for any $v$, we have
		\begin{align}\label{eqn:hxsfxs}
		&\int_0^t \EBB\|h(x(s)) - \nabla f(x(s))\|^2_2ds = \sum_{v =0}^{k-1} \int_{v \eta}^{(v+1)\eta} \EBB\|g^{(v)} - \nabla f(x(s))\|^2_2ds 
		\nonumber
		\\
		\le & 2\sum_{v =0}^{k-1} \int_{v \eta}^{(v+1)\eta} \EBB\|g^{(v)} - \nabla f(x^{(v)})\|^2_2ds +  2\sum_{v =0}^{k-1} \int_{v \eta}^{(v+1)\eta} \EBB\|\nabla f(x(v\eta))- \nabla f(x(s))\|^2_2ds 
		\end{align}
		where we use the Young's inequality and the fact that $x^{{v}} = x(v\eta)$ in the inequality.
		
		According to Lemma \ref{lemma:phi}, we have 
		\begin{align}\label{eqn:gvfx}
		&\EBB\|g^{(v)} - \nabla f(x^{(v)})\|^2_2 \le 2\EBB\|\Psi^{(v)}\|^2_2 + 2\EBB \|\Gamma^{(v)}\|^2_2 \le \frac{2N(2(N+n)+1)}{n}\sum_{i =1}^N \EBB \|\nabla f_i(x^{(v)}) - \alpha_i^{(v)}\|^2_2
		\nonumber\\
		\le &\frac{2N\tilde M^2(2(N+n)+1)}{n}\sum_{i =1}^N \EBB \|x^{(v)} - x^{(v_i)}\|^2_2 = \frac{2M^2(2(N+n)+1)}{nN}\sum_{i =1}^N \EBB \|x^{(v)} - x^{(v_i)}\|^2_2
		\nonumber\\
		= &\frac{2M^2(2(N+n)+1)}{nN}\sum_{i =1}^N \EBB \|\sum_{j= v_i}^v \eta g^{(j)} + \sum_{j= v_i}^v \sqrt{2\eta} \xi_j\|^2_2 
		\le \frac{2M^2(2(N+n)+1)}{n}\sum_{j= v-D}^v( 2D\eta^2 \EBB \| g^{(j)}\|^2 + 2\EBB \|\sqrt{2\eta} \xi_j\|^2_2)
		\nonumber\\
		\le & \frac{16M^2(2(N+n)+1)}{n}( \eta^2 N^2 D^2 ( 2\tilde M^2 D_B + G^2)   + D\eta d),
		\end{align}
		where we use the $\tilde M$ smoothness of $f_i$'s and $M = N\tilde M$ in the second line, Requirement \ref{assum5} and Jensen's inequality in the third line and Lemma \ref{lemma:xbound} in the last line. 
		
		For the second term of (\ref{eqn:hxsfxs}), we have 
		\begin{align}\label{eqn:fvsfs}
		&\sum_{v =0}^{k-1} \int_{v \eta}^{(v+1)\eta} \EBB\|\nabla f(x(v\eta))- \nabla f(x(s))\|^2_2ds
		\le
		\sum_{v =0}^{k-1} \int_{v \eta}^{(v+1)\eta} M^2\EBB\|x(v\eta)- x(s)\|^2_2ds
		\nonumber
		\\
		=&
		\sum_{v =0}^{k-1} \int_{v \eta}^{(v+1)\eta} M^2((s-v\eta)^2\EBB\|g^{(k)}\|^2_2 + 2(s-v\eta)d)ds
		\le \sum_{v =0}^{k-1} (\frac{M^2 \eta^3}{3}\EBB\|g^{(k)}\|^2_2 + 2M^2 \eta^2 d)
		\nonumber\\
		\le &\frac{4kN^2M^2 \eta^3(2\tilde M^2 D_B + G^2)}{3} + 2k M^2 \eta^2 d,
		\end{align}
		where the first inequality follows from the $M$-smoothness assumption of $f(x)$, the first equality follows from the definition of $x(s)$, and the last inequality is due to Lemma \ref{lemma:xbound}. 
		
		Combining (\ref{eqn:klqkpk}), (\ref{eqn:hxsfxs}), (\ref{eqn:gvfx}) and (\ref{eqn:fvsfs}),
		we have
		\begin{align}
		KL(Q_k || p^{(k)}) \le \frac{8kM^2(2(N+n)+1)\eta^2}{n}( \eta N^2 D^2 ( 2\tilde M^2 D_B + G^2)   + Dd) + \frac{4kN^2M^2 \eta^3(2\tilde M^2 D_B + G^2)}{3} + 2k M^2 \eta^2 d
		\end{align}
		Applying Lemma \ref{WtoKL}, Lemma \ref{logexp}, and Lemma \ref{lemma:xt} , and choosing $\lambda = 1$ and $x(0) = 0$ in Lemma \ref{WtoKL}, we obtain
		\begin{align*}
		&\WM_2(P(x^{(k)}),p^*) \le \WM_2(P(x^{(k)}),P(x(\eta k)))  + \WM_2(P(x(\eta k)),p^*) 
		\\
		\le& D_A[\frac{8kM^2(2(N+n)+1)\eta^2}{n}( \eta N^2 D^2 ( 2\tilde M^2 D_B + G^2)   + Dd) + \frac{4kN^2M^2 \eta^3(2\tilde M^2 D_B + G^2)}{3} + 2k M^2 \eta^2 d ]^{1/4} + D_4 e^{-k\eta/D_5},
		\end{align*}
		where we assume $\sqrt{KL(Q_k||p^{(k)})} \le 1$ since our target is to obtain high equality samples, and we denote $D_A = 4 \sqrt{3/2+(2b+d)\eta k}$.
		If we denote $T = k \eta$ and hide the constants, we have
		$$
		\WM_2(p^{(k)},p^*) \le \OM((T\eta + T\eta^2)^{1/4} + e^{-T/D_5}).
		$$
		By letting $T = \OM(D_5\log(\frac{1}{\epsilon}))$,$\eta = \tilde \OM(\epsilon^4)$ and $k  = T/\eta =\tilde  \OM(\frac{1}{\epsilon^4})$, we have $\WM_2(P(x^{(k)}),p^*) \le \epsilon$.

	\end{proof}

	\subsection{Extra experiments}
	In this section, we present the simulation experiment of sampling from distribution with convex $f(x)$ and the Bayesian ridge regression experiment on SliceLocation dataset.
	\subsubsection{Sampling from distribution with convex $f(x)$}
	In this simulated experiment, each component function $f_i(x)$ is convex and constructed in the following way:
	first a $d$-dimensional Gaussian variable $a_i$ is sampled from $\NM(2 {\bf 1}_d,4 I_{d \times d})$,
	then $f_i(x)$ is set to $f_i(x) = (x - a_i)^T \Sigma (x - a_i) / 2$,
	where $\Sigma$ is a positive definite symmetric matrix with maximum eigenvalue
	$M = 40$ and minimum eigenvalue $\mu = 1/2$.
	It can be verified that $f(x)$ is $M$ smooth, $\mu$ strongly convex and $0$ Hessian Lipschitz.
	The target density
	$p^*$ is a multivariate Gaussian distribution with mean $\bar{a} = \sum_{i=1}^n a_i$ and covariance matrix $\Sigma$.
	In order to compare the performance of different Data-Accessing and Snapshot-Updating strategies, we show the results for PPU/PTU/TMU with RA and TMU with RA/CA/RR.
	
	We report the $\WM_2$ distance between the distribution $p^{(k)}$ of each iterate and the target distribution $p^*$ for different algorithms with respect to the number of data passes( evaluation of $n$ $\nabla f_i$'s) in Figure \ref{fig:exp1}.
	In order to estimate $\WM_2(p^{(k)},p^*)$, we repeat all algorithms for $20,000$ times and obtain $20,000$ random samples for each algorithm in each iteration.
	From the left sub-figure of Figure ~\ref{fig:exp1}, we can see that TMU-RA outperforms SVRG-LD (PTU-RA) and SAGA-LD (PPU-RA).
	The right sub-figure of Figure ~\ref{fig:exp1} shows that RA Data-Accessing strategy outperforms CA/RR when using the same Snapshot-Updating strategy.
	
	\begin{figure}[!h]
		% \vskip 0.2in
		\centering
		\begin{subfigure}
			\centering
			\includegraphics[trim={0cm 0cm 0cm 0cm},clip,width=8cm,height= 5cm]{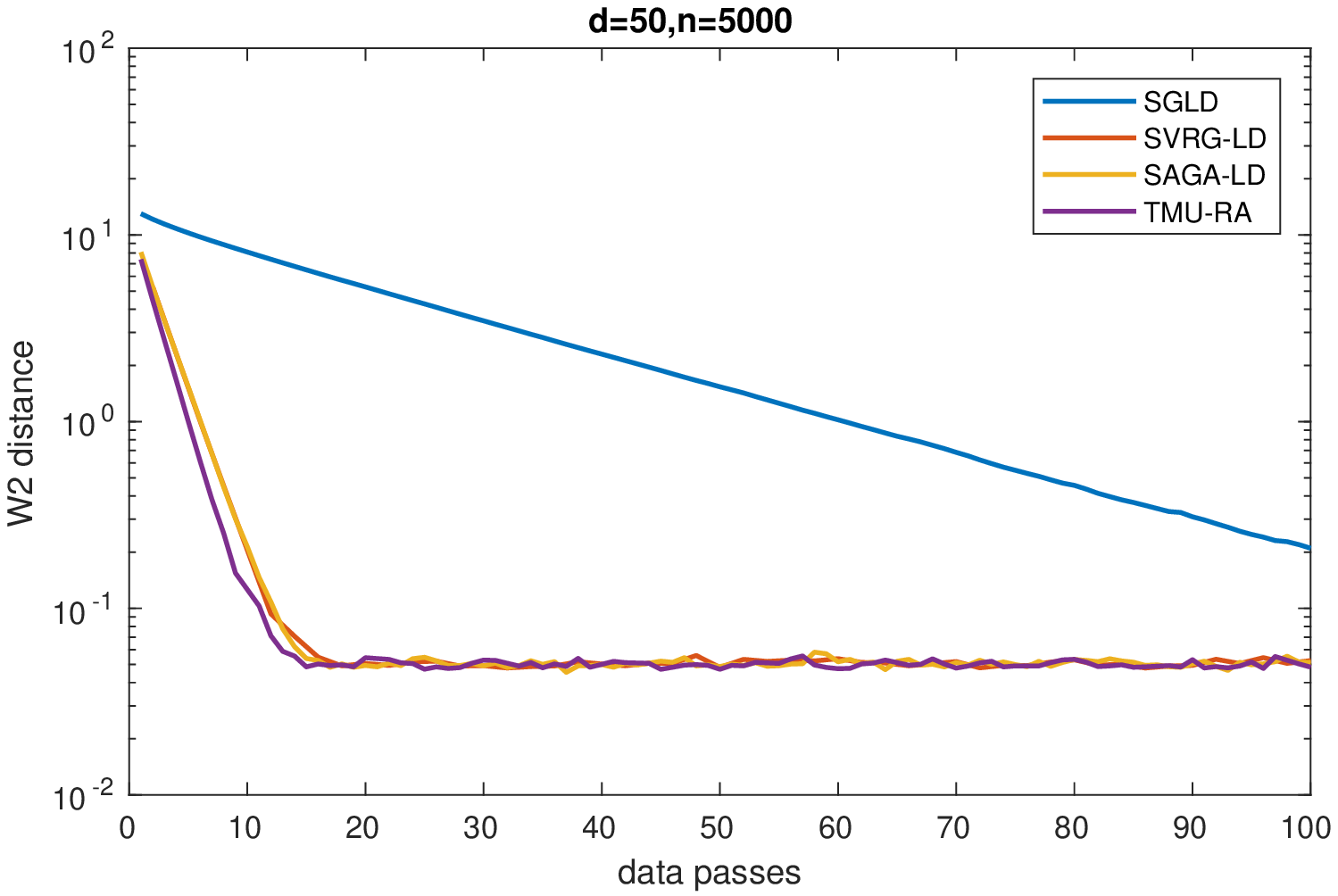}
		\end{subfigure}
		~
		\begin{subfigure}
			\centering
			\includegraphics[trim={0cm 0 0cm 0cm},clip,width=8cm,height= 5cm]{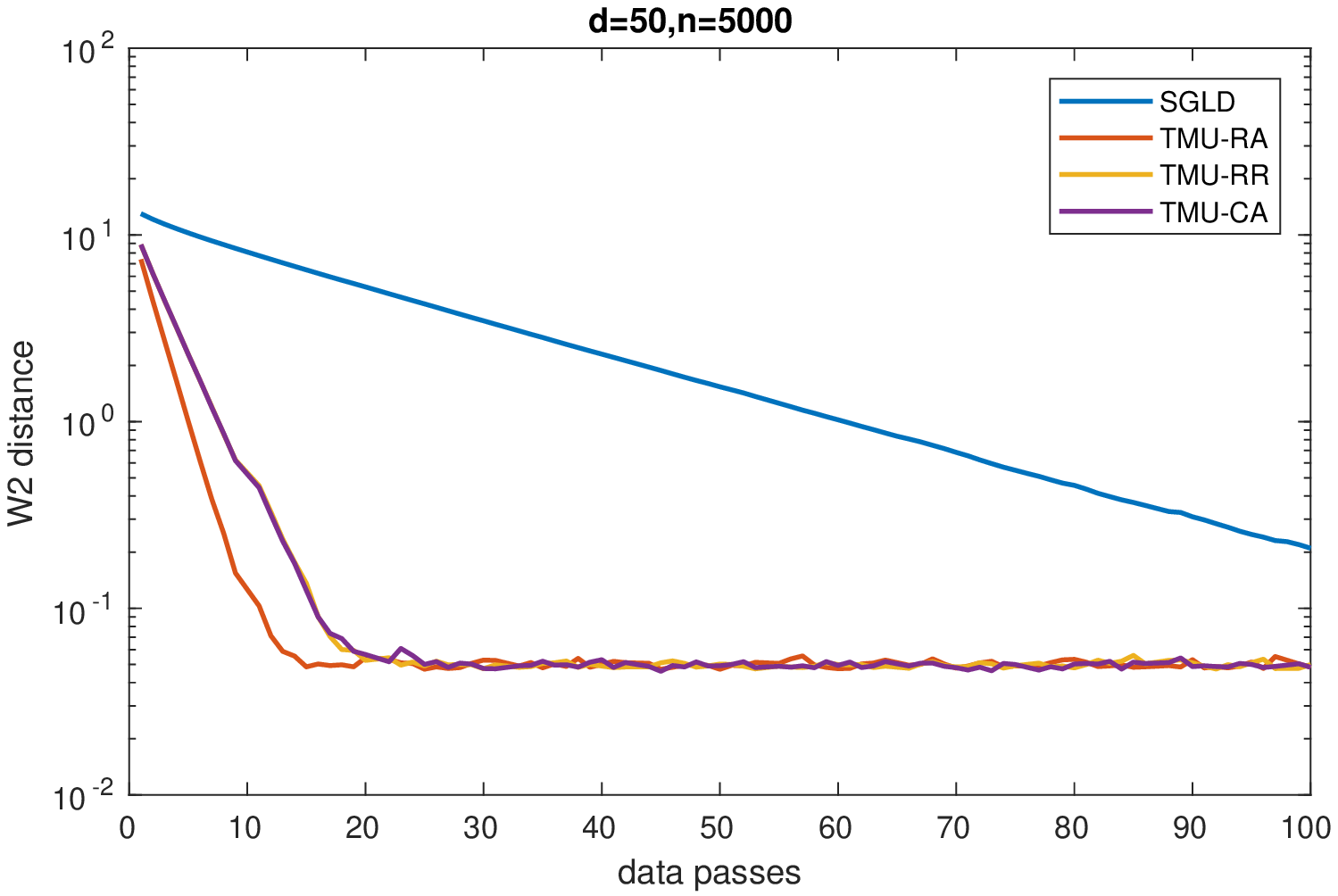}
		\end{subfigure}
		\caption{$\WM_2$ distance on simulated data.}
		%		The left sub-figure shows the result for different Snapshot-Updating strategies with RA; the right compares the result for TMU with different Data-Accessing strategies.}
		\label{fig:exp1}
		\vskip -0.1in
	\end{figure}
	
	\subsubsection{Bayesian ridge regression on SliceLocation dataset }
	Here, we present the results of Bayesian ridge regression on SliceLocation dataset.
	Similar results as that on YearPredictionMSD dataset can be observed, i.e. (\rNum{1}) TMU type methods have the best performance among all the methods with the same Data-Accessing strategy,
	(\rNum{2}) SVRG+ and PPU type methods constantly outperform LMC, SGLD, and PTU type methods.
	(\rNum{3})TMU-RA outperforms TMU-CA/TMU-RR, when the dataset is fitted to the memory.
	\begin{figure*}[t]
		\centering
		\begin{subfigure}
			\centering
			\includegraphics[trim={0cm 0cm 0cm 0cm},clip,width=3.7cm,height =3.2cm]{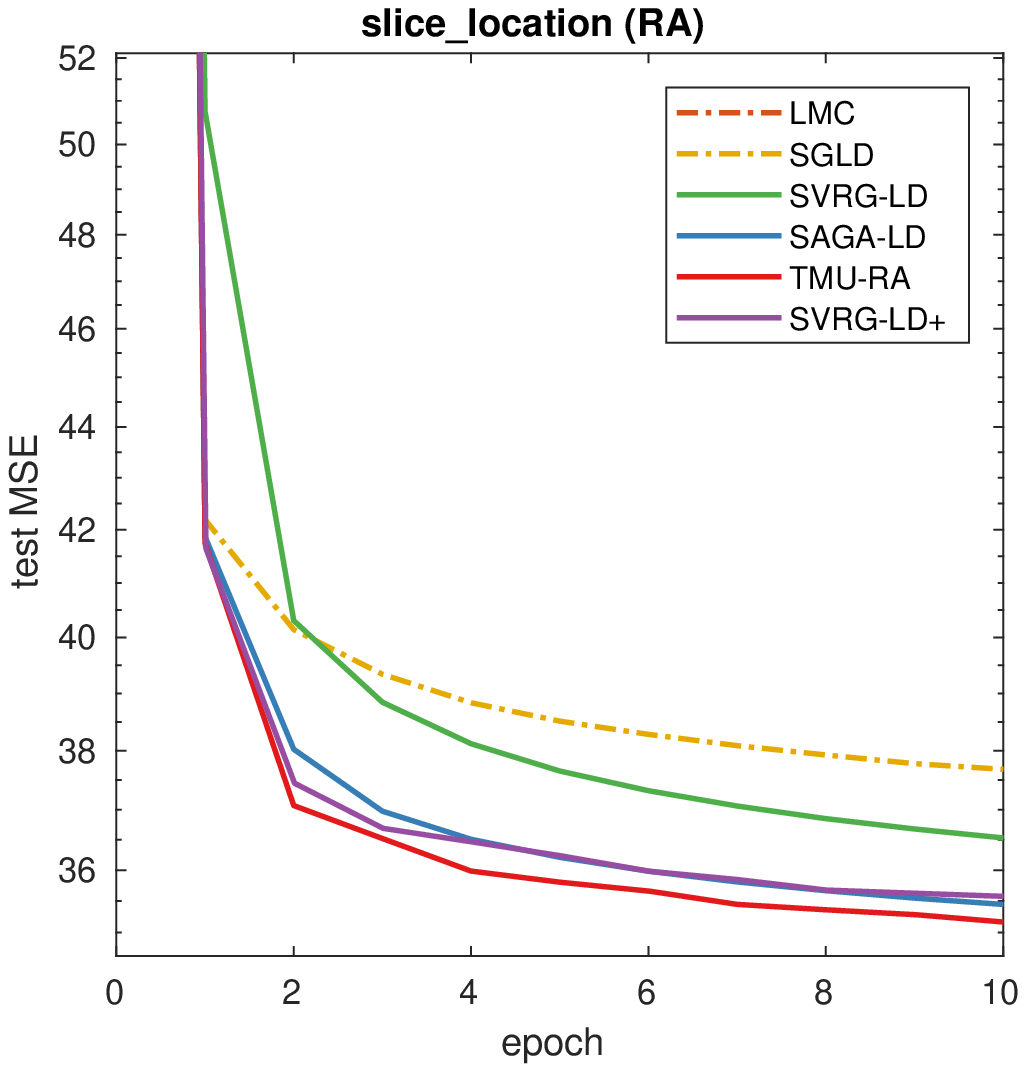}
		\end{subfigure}
		~
		%	\begin{subfigure}
		%		\centering
		%		\includegraphics[trim={0cm 0cm 0cm 0cm},clip,width=3.9cm]{img/RR/noise_epoch_RA}
		%	\end{subfigure}
		%	~
		%	\begin{subfigure}
		%		\centering
		%		\includegraphics[trim={0cm 0cm 0cm 0cm},clip,width=3.9cm]{img/RR/parkinsons_epoch_RA}
		%	\end{subfigure}
		%	~
		%		\begin{subfigure}
		%			\centering
		%			\includegraphics[trim={0cm 0cm 0cm 0cm},clip,width=3.7cm,height= 3.6cm]{img/RR/slice_location_epoch_RA}
		%		\end{subfigure}
		%		~
		\begin{subfigure}
			\centering
			\includegraphics[trim={0cm 0cm 0cm 0cm},clip,width=3.7cm,height= 3.2cm]{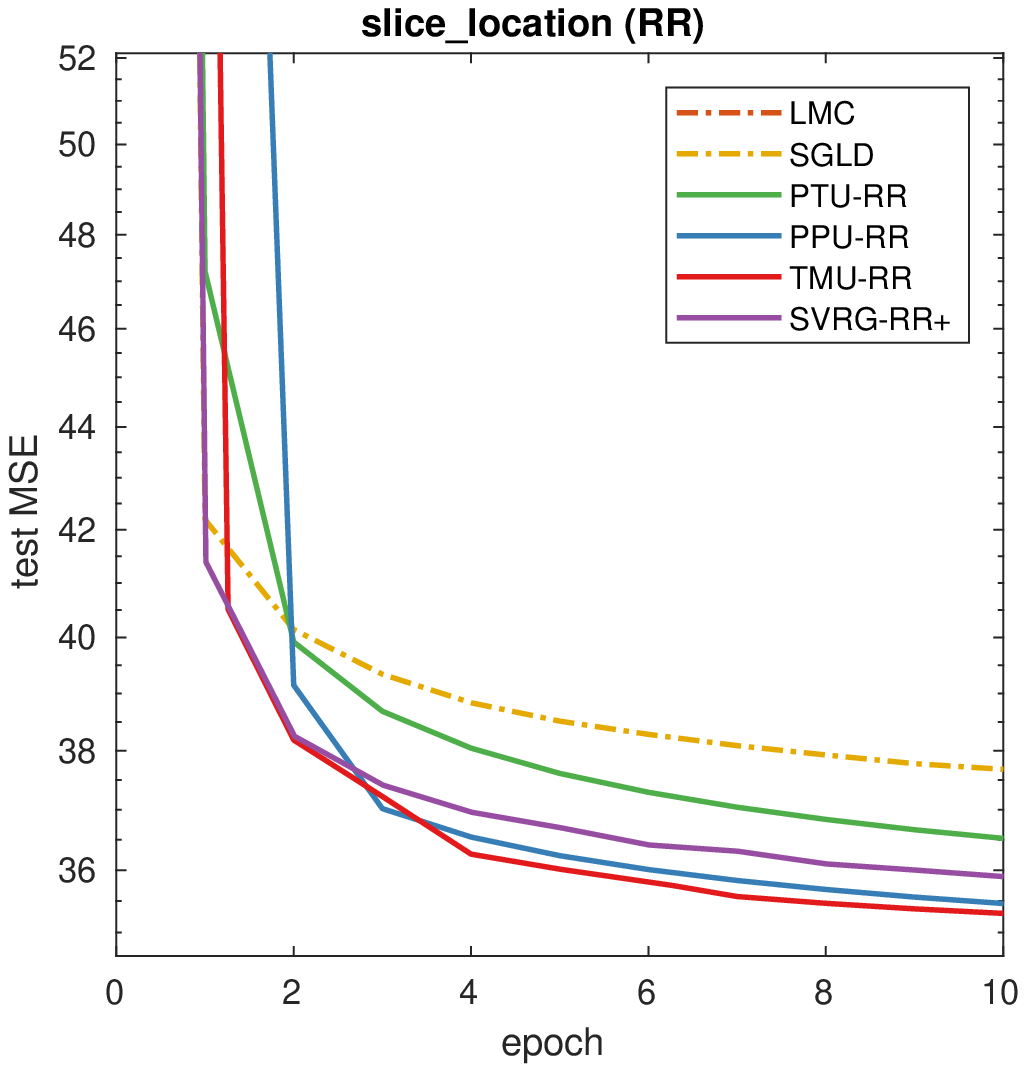}
		\end{subfigure}
		~
		%	\begin{subfigure}
		%		\centering
		%		\includegraphics[trim={0cm 0cm 0cm 0cm},clip,width=3.9cm]{img/RR/noise_epoch_RR}
		%	\end{subfigure}
		%	~
		%	\begin{subfigure}
		%		\centering
		%		\includegraphics[trim={0cm 0cm 0cm 0cm},clip,width=3.9cm]{img/RR/parkinsons_epoch_RR}
		%	\end{subfigure}
		%	~
		%		\begin{subfigure}
		%			\centering
		%			\includegraphics[trim={0cm 0cm 0cm 0cm},clip,width=3.7cm,height= 3.6cm]{img/RR/slice_location_epoch_RR}
		%		\end{subfigure}
		%		~
		\begin{subfigure}
			\centering
			\includegraphics[trim={0cm 0cm 0cm 0cm},clip,width=3.7cm,height= 3.2cm]{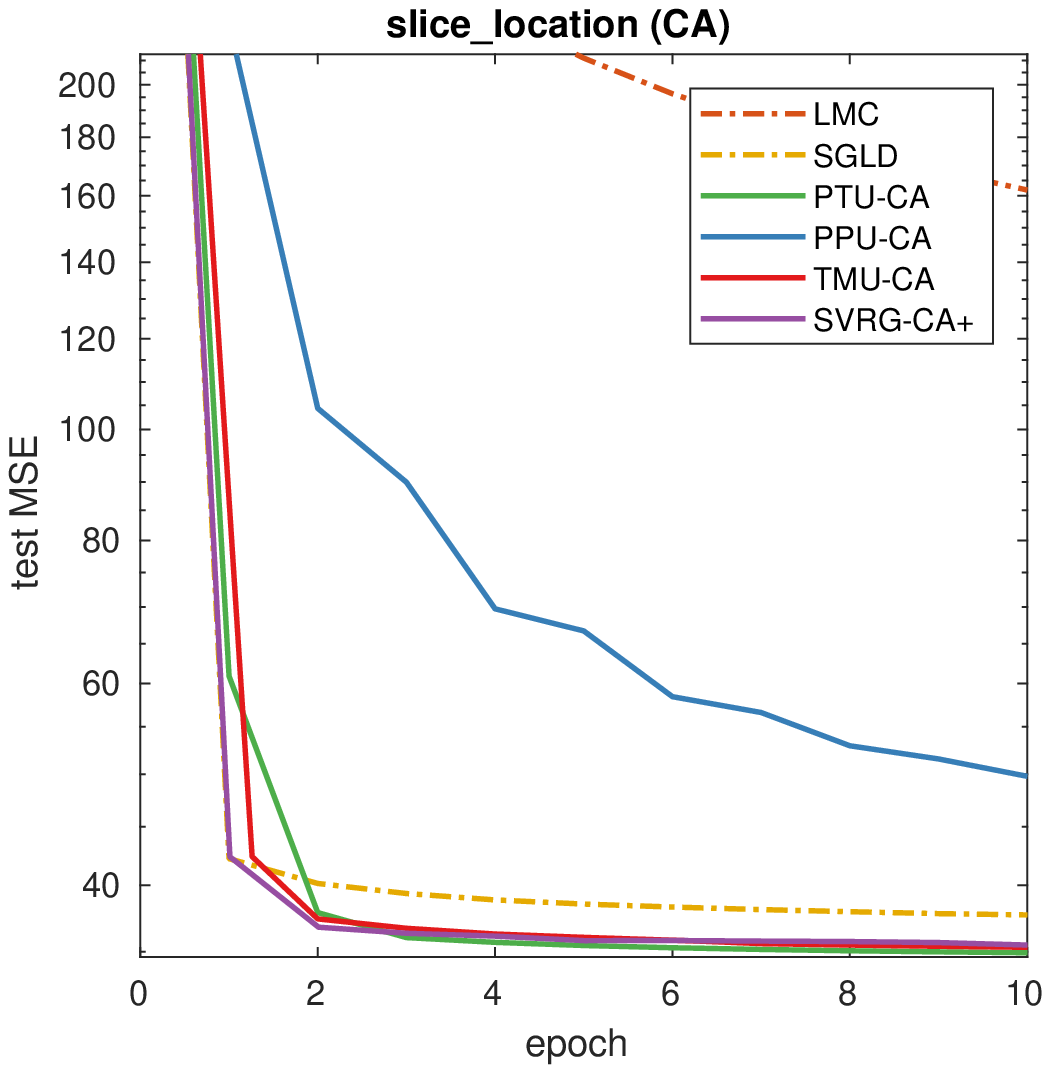}
		\end{subfigure}
		~
		%	\begin{subfigure}
		%		\centering
		%		\includegraphics[trim={0cm 0cm 0cm 0cm},clip,width=3.9cm]{img/RR/noise_epoch_CA}
		%	\end{subfigure}
		%	~
		%	\begin{subfigure}
		%		\centering
		%		\includegraphics[trim={0cm 0cm 0cm 0cm},clip,width=3.9cm]{img/RR/parkinsons_epoch_CA}
		%	\end{subfigure}
		%	~
		%		\begin{subfigure}
		%			\centering
		%			\includegraphics[trim={0cm 0cm 0cm 0cm},clip,width=3.7cm,height= 3.6cm]{img/RR/slice_location_epoch_CA}
		%		\end{subfigure}
		%		~
		\begin{subfigure}
			\centering
			\includegraphics[trim={0cm 0cm 0cm 0cm},clip,width=3.7cm,height= 3.2cm]{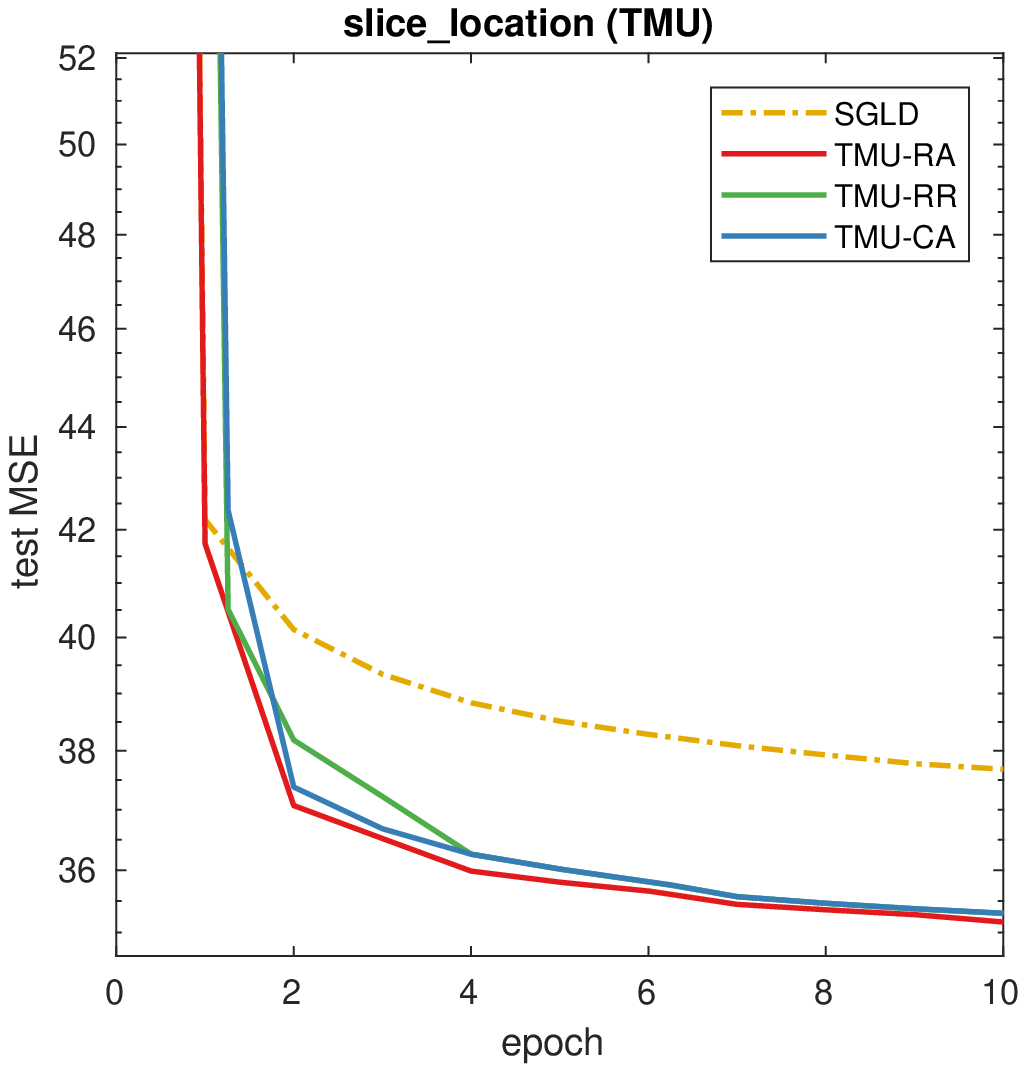}
		\end{subfigure}
		%	~
		%	\begin{subfigure}
		%		\centering
		%		\includegraphics[trim={0cm 0cm 0cm 0cm},clip,width=3.9cm]{img/RR/noise_epoch_TMU}
		%	\end{subfigure}
		%	~
		%	\begin{subfigure}
		%		\centering
		%		\includegraphics[trim={0cm 0cm 0cm 0cm},clip,width=3.9cm]{img/RR/parkinsons_epoch_TMU}
		%	\end{subfigure}
		%		~
		%		\begin{subfigure}
		%			\centering
		%			\includegraphics[trim={0cm 0cm 0cm 0cm},clip,width=3.7cm,height= 3.6cm]{img/RR/slice_location_epoch_TMU}
		%		\end{subfigure}
		\caption{Bayesian Ridge Regression on SliceLocation dataset.}
		%	The first three rows shows the results for difference methods with RA, RR and CA, respectively.In the last row, we compare the performance of TMU with different Data-Accessing strategies.}
		\label{fig:exp2}
		\vskip -0.2in
	\end{figure*}

\end{document}